\documentclass{article}

\usepackage{etoolbox}
\newcommand{\arxiv}[1]{\iftoggle{arxiv}{#1}{}}
\newtoggle{arxiv}
\global\toggletrue{arxiv}
\newtoggle{neurips}
\global\togglefalse{neurips}

\newcommand{\loose}{\looseness=-1}

	\usepackage{natbib}
	\bibpunct{(}{)}{;}{a}{,}{,}

\usepackage[utf8]{inputenc} %
\usepackage[T1]{fontenc}    %
\usepackage{lmodern}
\usepackage[letterpaper, left=1in, right=1in, top=1in, bottom=1in]{geometry}
\usepackage{hyperref}
\hypersetup{colorlinks,citecolor=blue,linkcolor = blue}       %
\usepackage{url}            %
\usepackage{booktabs}       %
\usepackage{amsfonts}       %
\usepackage{nicefrac}       %
\usepackage{microtype}      %
\usepackage[dvipsnames]{xcolor}         %
\usepackage{xspace}

\usepackage{makecell}
\usepackage{enumitem}
\usepackage{etoolbox}
\usepackage{comment}
\usepackage{wrapfig}
\usepackage{colortbl}
\usepackage{setspace}
\usepackage{transparent}
\usepackage{inconsolata}
\usepackage[scaled=.90]{helvet}
\usepackage{xspace}
\usepackage{verbatim}
\usepackage{mathtools}

\usepackage{algorithm}
\usepackage[noend]{algpseudocode}

\usepackage{tablefootnote}
\usepackage[disable]{todonotes}
\usepackage{crossreftools}
\newcommand{\crefzak}[1]{\crtcref{#1}}

\usepackage{pifont}
\newcommand{\cmark}{\ding{51}}%
\newcommand{\xmark}{\ding{55}}%

\usepackage{amsthm}
\usepackage{mathtools}
\usepackage{amsmath}
\usepackage{bm}
\usepackage{bbm}
\usepackage{amsfonts}
\usepackage{amssymb}
\usepackage{MnSymbol} %
\usepackage[nameinlink,capitalize]{cleveref}

\usepackage{thmtools}
\usepackage{thm-restate}
\declaretheorem[name=Theorem,parent=section]{theorem}
\declaretheorem[name=Lemma,parent=section]{lemma}
\declaretheorem[name=Corollary,parent=section]{corollary}

\declaretheorem[name=Assumption, parent=section]{assumption}

\declaretheorem[name=Remark,parent=section]{remark}
\declaretheorem[name=Proposition, parent=section]{proposition}

\makeatletter
\renewenvironment{proof}[1][Proof]%
{%
	\par\noindent{\bfseries\upshape {#1.}\ }%
}%
{\qed\newline}
\makeatother

\newtheorem*{theorem*}{Theorem}

\theoremstyle{plain}
\newtheorem{definition}[theorem]{Definition}

\usepackage{xpatch}
\xpatchcmd{\proof}{\itshape}{\normalfont\proofnameformat}{}{}
\newcommand{\proofnameformat}{\bfseries}

\newcommand{\pref}[1]{\cref{#1}}
\newcommand{\pfref}[1]{Proof of \pref{#1}}

\renewcommand{\eqref}[1]{\texorpdfstring{\hyperref[#1]{(\ref*{#1})}}{(\ref*{#1})}}
\crefformat{equation}{#2Eq.\,(#1)#3}
\Crefformat{equation}{#2Eq.\,(#1)#3}

\Crefformat{figure}{#2Figure~#1#3}
\Crefformat{assumption}{#2Assumption~#1#3}

\Crefname{assumption}{Assumption}{Assumptions}

\def\ddefloop#1{\ifx\ddefloop#1\else\ddef{#1}\expandafter\ddefloop\fi}
\def\ddef#1{\expandafter\def\csname bb#1\endcsname{\ensuremath{\mathbb{#1}}}}
\ddefloop ABCDEFGHIJKLMNOPQRSTUVWXYZ\ddefloop
\def\ddefloop#1{\ifx\ddefloop#1\else\ddef{#1}\expandafter\ddefloop\fi}
\def\ddef#1{\expandafter\def\csname b#1\endcsname{\ensuremath{\mathbf{#1}}}}
\ddefloop ABCDEFGHIJKLMNOPQRSTUVWXYZ\ddefloop
\def\ddef#1{\expandafter\def\csname sf#1\endcsname{\ensuremath{\mathsf{#1}}}}
\ddefloop ABCDEFGHIJKLMNOPQRSTUVWXYZ\ddefloop
\def\ddef#1{\expandafter\def\csname c#1\endcsname{\ensuremath{\mathcal{#1}}}}
\ddefloop ABCDEFGHIJKLMNOPQRSTUVWXYZ\ddefloop
\def\ddef#1{\expandafter\def\csname h#1\endcsname{\ensuremath{\widehat{#1}}}}
\ddefloop ABCDEFGHIJKLMNOPQRSTUVWXYZ\ddefloop
\def\ddef#1{\expandafter\def\csname hc#1\endcsname{\ensuremath{\widehat{\mathcal{#1}}}}}
\ddefloop ABCDEFGHIJKLMNOPQRSTUVWXYZ\ddefloop
\def\ddef#1{\expandafter\def\csname t#1\endcsname{\ensuremath{\widetilde{#1}}}}
\ddefloop ABCDEFGHIJKLMNOPQRSTUVWXYZ\ddefloop
\def\ddef#1{\expandafter\def\csname tc#1\endcsname{\ensuremath{\widetilde{\mathcal{#1}}}}}
\ddefloop ABCDEFGHIJKLMNOPQRSTUVWXYZ\ddefloop
\def\ddefloop#1{\ifx\ddefloop#1\else\ddef{#1}\expandafter\ddefloop\fi}
\def\ddef#1{\expandafter\def\csname scr#1\endcsname{\ensuremath{\mathscr{#1}}}}
\ddefloop ABCDEFGHIJKLMNOPQRSTUVWXYZ\ddefloop

\newcommand{\veps}{\varepsilon}
\newcommand{\tveps}{{\varepsilon}}

\renewcommand{\tC}{4}
\newcommand{\cfrak}{\mathfrak{c}}

\DeclareMathOperator*{\argmin}{arg\,min} %
\DeclareMathOperator*{\argmax}{arg\,max}

\def\ddef#1{\expandafter\def\csname b#1\endcsname{\ensuremath{\mb{#1}}}}
\ddefloop abcdeghijklmnopqrstuvwxyz\ddefloop

\newcommand{\ind}[1]{^{(#1)}}

\DeclarePairedDelimiter{\abs}{\lvert}{\rvert} %
\DeclarePairedDelimiter{\brk}{[}{]}
\DeclarePairedDelimiter{\crl}{\{}{\}}

\DeclarePairedDelimiter{\nrm}{\|}{\|}

\DeclarePairedDelimiter{\ceil}{\lceil}{\rceil}

\let\P\undefined
\DeclareMathOperator{\En}{\mathbb{E}}

\DeclareMathOperator{\P}{P}

\newcommand{\mb}[1]{\boldsymbol{#1}}
\renewcommand{\bm}[1]{\boldsymbol{#1}}
\newcommand{\wt}[1]{\widetilde{#1}}

\newcommand{\wb}[1]{\widebar{#1}}

\let\underbar\undefined

\makeatletter
\let\save@mathaccent\mathaccent
\newcommand*\if@single[3]{%
	\setbox0\hbox{${\mathaccent"0362{#1}}^H$}%
	\setbox2\hbox{${\mathaccent"0362{\kern0pt#1}}^H$}%
	\ifdim\ht0=\ht2 #3\else #2\fi
}
\newcommand*\rel@kern[1]{\kern#1\dimexpr\macc@kerna}
\newcommand*\widebar[1]{\@ifnextchar^{{\wide@bar{#1}{0}}}{\wide@bar{#1}{1}}}
\newcommand*\underbar[1]{\@ifnextchar_{{\under@bar{#1}{0}}}{\under@bar{#1}{1}}}
\newcommand*\wide@bar[2]{\if@single{#1}{\wide@bar@{#1}{#2}{1}}{\wide@bar@{#1}{#2}{2}}}
\newcommand*\under@bar[2]{\if@single{#1}{\under@bar@{#1}{#2}{1}}{\under@bar@{#1}{#2}{2}}}
\newcommand*\wide@bar@[3]{%
	\begingroup
	\def\mathaccent##1##2{%
		\let\mathaccent\save@mathaccent
		\if#32 \let\macc@nucleus\first@char \fi
		\setbox\z@\hbox{$\macc@style{\macc@nucleus}_{}$}%
		\setbox\tw@\hbox{$\macc@style{\macc@nucleus}{}_{}$}%
		\dimen@\wd\tw@
		\advance\dimen@-\wd\z@
		\divide\dimen@ 3
		\@tempdima\wd\tw@
		\advance\@tempdima-\scriptspace
		\divide\@tempdima 10
		\advance\dimen@-\@tempdima
		\ifdim\dimen@>\z@ \dimen@0pt\fi
		\rel@kern{0.6}\kern-\dimen@
		\if#31
		\overline{\rel@kern{-0.6}\kern\dimen@\macc@nucleus\rel@kern{0.4}\kern\dimen@}%
		\advance\dimen@0.4\dimexpr\macc@kerna
		\let\final@kern#2%
		\ifdim\dimen@<\z@ \let\final@kern1\fi
		\if\final@kern1 \kern-\dimen@\fi
		\else
		\overline{\rel@kern{-0.6}\kern\dimen@#1}%
		\fi
	}%
	\macc@depth\@ne
	\let\math@bgroup\@empty \let\math@egroup\macc@set@skewchar
	\mathsurround\z@ \frozen@everymath{\mathgroup\macc@group\relax}%
	\macc@set@skewchar\relax
	\let\mathaccentV\macc@nested@a
	\if#31
	\macc@nested@a\relax111{#1}%
	\else
	\def\gobble@till@marker##1\endmarker{}%
	\futurelet\first@char\gobble@till@marker#1\endmarker
	\ifcat\noexpand\first@char A\else
	\def\first@char{}%
	\fi
	\macc@nested@a\relax111{\first@char}%
	\fi
	\endgroup
}
\newcommand*\under@bar@[3]{%
	\begingroup
	\def\mathaccent##1##2{%
		\let\mathaccent\save@mathaccent
		\if#32 \let\macc@nucleus\first@char \fi
		\setbox\z@\hbox{$\macc@style{\macc@nucleus}_{}$}%
		\setbox\tw@\hbox{$\macc@style{\macc@nucleus}{}_{}$}%
		\dimen@\wd\tw@
		\advance\dimen@-\wd\z@
		\divide\dimen@ 3
		\@tempdima\wd\tw@
		\advance\@tempdima-\scriptspace
		\divide\@tempdima 10
		\advance\dimen@-\@tempdima
		\ifdim\dimen@>\z@ \dimen@0pt\fi
		\rel@kern{0.6}\kern-\dimen@
		\if#31
		\underline{\rel@kern{-0.6}\kern\dimen@\macc@nucleus\rel@kern{0.4}\kern\dimen@}%
		\advance\dimen@0.4\dimexpr\macc@kerna
		\let\final@kern#2%
		\ifdim\dimen@<\z@ \let\final@kern1\fi
		\if\final@kern1 \kern-\dimen@\fi
		\else
		\underline{\rel@kern{-0.6}\kern\dimen@#1}%
		\fi
	}%
	\macc@depth\@ne
	\let\math@bgroup\@empty \let\math@egroup\macc@set@skewchar
	\mathsurround\z@ \frozen@everymath{\mathgroup\macc@group\relax}%
	\macc@set@skewchar\relax
	\let\mathaccentV\macc@nested@a
	\if#31
	\macc@nested@a\relax111{#1}%
	\else
	\def\gobble@till@marker##1\endmarker{}%
	\futurelet\first@char\gobble@till@marker#1\endmarker
	\ifcat\noexpand\first@char A\else
	\def\first@char{}%
	\fi
	\macc@nested@a\relax111{\first@char}%
	\fi
	\endgroup
}
\makeatother

\usepackage{accents}

\newcommand{\phibar}{\bar{\phi}}

\newcommand{\LRMDP}{Low-Rank MDP\xspace}
\newcommand{\lrmdp}{low-rank MDP\xspace}

\newcommand{\initd}{\rho}
\newcommand{\phistarh}[1][h]{\phi^{\star}_{#1}}
\newcommand{\phistarhpi}[1][\pi]{{\phi}^{\star,#1}_{h}}

\newcommand{\phistar}{\phi^{\star}}
\newcommand{\muh}[1][h+1]{\mu^{\star}_{#1}}

\newcommand{\muhb}[1][h+1]{\bar{\mu}^{\star}_{#1}}

\newcommand{\cXbar}{\widebar{\cX}}

\newcommand{\briee}{\texttt{BRIEE}}
\newcommand{\moffle}{\texttt{MOFFLE}}
\newcommand{\est}{\texttt{LinEst}}

\newcommand{\veceval}{\texttt{EstVec}}
\newcommand{\replearn}{\texttt{RepLearn}}

\newcommand{\spanner}{\texttt{RobustSpanner}}
\newcommand{\apx}{\texttt{LinOpt}}

\newcommand{\wwbar}{\bar{w}}

\newcommand{\thetabar}{\bar{\theta}}
\newcommand{\mubar}{\bar{\mu}}

\newcommand{\dd}{\mathrm{d}}

\newcommand{\nubar}{\bar\nu}

	\newcommand{\inner}[2]{\langle #1,#2\rangle}

\newcommand{\F}{\mathrm{F}}

\newcommand{\psd}{\mathbb{S}^{d\times d}_{+}}
\newcommand{\phih}{\phi}

\newcommand{\ldef}{\vcentcolon=}

\newcommand{\what}{\widehat}

\newcommand{\rhs}{right-hand side\xspace}
\newcommand{\cMbar}{\wb{\cM}}
\newcommand{\cAbar}{\wb{\cA}}

\renewcommand{\ln}{\log}        %

\newcommand{\tfrak}{\mathfrak{t}}
\newcommand{\term}{\mathfrak{t}}
\newcommand{\afrak}{\mathfrak{a}}

\newcommand{\pihat}{\hat{\pi}}
\newcommand{\pistar}{\pi_{\star}}

\let\oldparagraph\paragraph
\renewcommand{\paragraph}[1]{\oldparagraph{#1.}}
\renewcommand{\colon}{:}        %

\newcommand{\reals}{\mathbb{R}}

\renewcommand{\P}{\mathbb{P}}

\newcommand{\E}{\mathbb{E}}

\newcommand{\nn}{\nonumber} 
\newcommand{\ldotst}{%
	\mathinner{{\ldotp}{\ldotp}}%
}

\newcommand{\unifa}{\pi_\texttt{unif}}
\newcommand{\piunif}{\pi_\texttt{unif}}
\newcommand{\unif}{\texttt{unif}}
\newcommand{\stat}{\texttt{stat}}
\newcommand{\olive}{\texttt{OLIVE}}

\renewcommand{\a}{\bm{a}}

\newcommand{\x}{\bm{x}}

\newcommand{\supp}{\mathrm{supp}\,}

\newcommand{\bayes}{\texttt{bayes}}
\newcommand{\Pibar}{\wbar{\Pi}}

\newcommand{\wtilde}[1]{\widetilde{#1}}
\newcommand{\wbar}[1]{\widebar{#1}}

\newcommand{\phihat}{\hat{\phi}}
\newcommand{\Ebar}{\wbar\E}
\newcommand{\Pbar}{\wbar\P}

\newcommand{\dbar}{\bar{d}}
\newcommand{\Mbar}{\wbar{\cM}}
\newcommand{\Abar}{\wbar{\cA}}
\newcommand{\ghat}{\hat{g}}

\newcommand{\Pibarm}{\wbar{\Pi}_{\texttt{M}}}

\newcommand{\spanrl}{\texttt{VoX}\xspace}
\newcommand{\mainalg}{\spanrl}

\newcommand{\psdp}{\texttt{PSDP}\xspace}

\newcommand{\Pim}{\Pi_{\texttt{M}}}

\newcommand{\bpi}{\bm{\pi}}

\newcommand{\mustarh}{\mu^\star_h}
\newcommand{\mustar}{\mu^\star}

\renewcommand{\emptyset}{\varnothing}

\newcommand{\algcommentlight}[1]{\textcolor{blue!70!black}{\transparent{0.5}\footnotesize{\texttt{\textbf{//\hspace{2pt}#1}}}}}

\newcommand{\algcommentbiglight}[1]{\textcolor{blue!70!black}{\transparent{0.5}\footnotesize{\texttt{\textbf{/* #1~*/}}}}}

\newcommand{\trn}{\top}

\newcommand{\bigoh}{O}
\newcommand{\bigoht}{\wt{O}}

\newcommand{\poly}{\mathrm{poly}}
\newcommand{\polylog}{\mathrm{polylog}}

\newcommand{\ee}{\mathbb{E}}
\newcommand{\rr}{\mathbb{R}}
\DeclareMathOperator{\lspan}{span}
\newcommand{\norm}[1]{\|#1 \|}
\newcommand{\inprod}[2]{{#1}^\top {#2} }

\renewcommand{\crefzak}{\cref}

\Crefformat{line}{#2Line #1#3}

\usepackage[suppress]{color-edits}
\addauthor{df}{ForestGreen}
\addauthor{zm}{red}
\addauthor{ab}{teal}

\usepackage{parskip}
\usepackage{autonum}

\title{\mbox{Efficient Model-Free Exploration in Low-Rank MDPs}}

\author{Zakaria Mhammedi\\{\small \texttt{mhammedi@mit.edu}}\and    Adam Block\\{\small \texttt{ablock@mit.edu}} \and    Dylan J. Foster\\{\small \texttt{dylanfoster@microsoft.com}} \and Alexander Rakhlin\\{\small \texttt{rakhlin@mit.edu}}
}
\date{}

\begin{document}

	\maketitle

	\begin{abstract}
A major challenge in reinforcement learning is to develop practical,
sample-efficient algorithms for exploration in high-dimensional
domains where generalization and function approximation is
required. \emph{Low-Rank Markov Decision Processes}---where transition
probabilities admit a low-rank factorization based on an unknown
feature embedding---offer a simple, yet expressive framework for RL
with function approximation, but existing algorithms are either (1) 
computationally intractable, or (2) reliant upon restrictive statistical
assumptions such as latent variable structure, access to
  model-based function approximation, or reachability. In this work,
we propose the first provably sample-efficient algorithm for
exploration in Low-Rank MDPs that is both computationally efficient
and model-free, allowing for general function approximation and requiring no additional structural assumptions. Our algorithm, \mainalg, uses the notion of a \emph{barycentric spanner} for the feature embedding as an efficiently computable
basis for exploration, performing efficient barycentric spanner computation by interleaving representation learning and policy optimization. Our analysis---which is appealingly simple and modular---carefully combines several 
techniques, including a new approach to error-tolerant
  barycentric spanner computation and an improved analysis of a certain minimax representation learning 
objective found in prior work.

 	\end{abstract}
	
	\tableofcontents
	\section{Introduction}
	\label{sec:intro}
In reinforcement learning and control, many of the most promising
application domains require the agent to navigate complex,
high-dimensional state and action spaces, where generalization and function approximation
is necessary. The last decade has
witnessed impressive empirical success in domains where
data are abundant \citep{mnih2015human,
  silver2016mastering,kober2013reinforcement,lillicrap2015continuous,li2016deep},
but when data are limited, ensuring efficient exploration in
large domains is a major research question. For
\emph{statistical efficiency}, the foundations have recently begun to
take shape, with a line
of research providing structural conditions that facilitate
sample-efficient exploration, as well as fundamental limits
\citep{russo2013eluder,jiang2017contextual,sun2019model,wang2020provably,du2021bilinear,jin2021bellman,foster2021statistical,foster2023tight}. \emph{Computational
  efficiency}, however, remains a major challenge: outside of simple
settings \citep{azar2017minimax,jin2020provably}, existing algorithms
with provable sample complexity guarantees are computationally
inefficient, and typically require solving intractable non-convex
optimization problems
\cite{jiang2017contextual,dann2018oracle,jin2021bellman,cheng2023improved}. The
prospect of
developing practical algorithms for exploration in
high-dimensional state spaces that are both computationally and
statistically efficient raises three fundamental questions:
\begin{enumerate}[leftmargin=20pt]
    \item What are the right computational primitives for exploration?
    That is, how can one efficiently represent and compute exploratory policies that
    allow the learner
    to explore the state
    space and gather useful data?
  \item How should one leverage function approximation---for
    example, via
    representation learning---to 
    discover such primitives in a computationally and statistically
    efficient fashion?
  \item Given answers to the first two questions, how can one efficiently interleave function approximation and exploration to provide provably efficient algorithms?
  \end{enumerate}

In this paper, we investigate these questions through the \emph{\LRMDP}
model \citep{RendleFS10,YaoSPZ14,agarwal2020flambe}. In a \LRMDP, the state space is large
and potentially continuous, but the transition probabilities admit an
(unknown) low-rank factorization. Concretely, for a finite-horizon
\LRMDP with horizon $H$, the transition densities for layer
$h\in\brk{H}$ satisfy
\begin{align}
  T_h(x_{h+1}\mid{}x_h,a_h) = \muh[h+1](x_{h+1})^{\trn}\phistarh(x_h,a_h),
  \label{eq:low_rank_mdp}
\end{align}
where $\phistarh(\cdot,\cdot)\in\bbR^{d}$ and
$\muh(\cdot)\in\bbR^{d}$ are state-action and next-state
embeddings. The low-rank structure in \eqref{eq:low_rank_mdp}
facilitates tractable exploration: if the embedding $\phistarh$ is known
to the learner, one can efficiently learn a near-optimal policy with sample
complexity polynomial in the feature dimension $d$, and independent of
the size of the state space \citep{jin2020provably}; in this regard,
$\phistarh$ can be thought of as a low-dimensional \emph{representation} that enables
sample-efficient RL. Following
\citet{agarwal2020flambe}, we consider the challenging setting in
which both $\phistarh$ and $\muh$ are \emph{unknown} to the
learner. This formulation generalizes well-known frameworks such as
the \emph{Block MDP} (BMDP) model \citep{du2019latent,misra2019kinematic},
and necessitates the use of \emph{representation
  learning}: the agent must learn an embedding that approximates
$\phistarh$ as it explores the environment, and must use this learned embedding
to drive subsequent exploration.  This form of function approximation allows
for great flexibility, as $\phistarh$ can be an arbitrary, nonlinear
function of the state; in practice, it is common to model $\phistarh$ as a neural net \citep{zhang2022efficient}.

The \LRMDP is perhaps the simplest MDP structure that demands
systematic exploration and nonlinear function approximation while allowing for a continuum of states, yet understanding of
\emph{efficient} algorithm design for this model is surprisingly
limited. Existing algorithms suffer from at least one of the following drawbacks:
\begin{enumerate}[leftmargin=20pt]
\item Computational intractability \citep{jiang2017contextual,jin2021bellman,du2021bilinear,chen2022partially,xie2022role}. 
\item Strong modeling assumptions (e.g., ability to model
  $\muh[h+1](\cdot)$, which facilitates application of model-based
  RL techniques)
  \citep{agarwal2020flambe,uehara2022representation,cheng2023improved};
  in this work, we aim for \emph{model-free} methods that only require
  learning $\phistarh$.
\item Restrictive structural assumptions (e.g., 
  non-negativity or latent variable
  structure for the embeddings in \eqref{eq:low_rank_mdp}) \citep{modi2021model,zhang2022efficient}.
\end{enumerate}
At the root of these limitations is the complex interplay between
exploration and representation learning: %
the agent must learn a high-quality representation to guide
in exploring
the state space, but learning such a representation requires gathering
diverse and informative data, which is difficult to acquire without
having already explored the state space to begin with. Overcoming
this challenge---particularly where computational efficiency is
concerned---requires (1) representation learning procedures that lead to sufficiently expressive
representations for downstream applications, (2) efficient exploration procedures that are
robust to errors in learned representations, and 3) understanding the
interaction between these procedures, which must be interleaved. In
this work, we propose an algorithm that addresses each of these challenges, as detailed below.

\paragraph{Contributions}
We provide the first provably computationally efficient and model-free
algorithm for general Low-Rank MDPs.
Our algorithm, \mainalg{} (``Volumetric Exploration''), uses
the notion of a \emph{barycentric spanner} for the
embedding $\phistarh$ as an efficiently computable
basis for exploration, and combines this with a minimax representation
learning objective \citep{modi2021model,zhang2022efficient}. \mainalg 
interleaves exploration with representation learning in a layer-wise
fashion, learning a new representation at each layer $h$ using exploratory
data gathered at previous layers, then uses this representation to
facilitate computation of a collection of exploratory policies (a
\emph{policy cover}), which act as an approximate barycentric spanner
for the features at layer $h+1$, ensuring good coverage for subsequent
iterations.  $\mainalg$ is simple and modular, and its analysis is
surprisingly compact given the greater generality compared to prior
work
\citep{zhang2022efficient,modi2021model,mhammedi2023representation}. \loose

\mainalg accommodates general-purpose function approximation
to learn the representation $\phistar$ (e.g., neural
nets or other flexible classes), and is efficient whenever a certain minimax
representation learning objective \citep{modi2021model,zhang2022efficient} can be solved efficiently for the
function class of interest. Compared to efficient algorithms from
prior work, \mainalg: (1) is model-free (i.e., only requires access to a function class
$\Phi$ capable of modeling $\phistar$, and does not need to model
$\muh$), and (2) applies to general Low-Rank MDPs, removing
  the need for strong assumptions such as reachability or non-negativity of the feature embeddings
(so-called \emph{latent variable} structure); see
\Cref{tb:resultscomp+}).
As a secondary benefit, the algorithm is reward-free.
Our analysis carefully combines several new techniques, including (1) a reduction from barycentric spanner computation to policy optimization, and (2) a new analysis of a minimax representation learning
objective introduced in \citep{modi2021model,zhang2022efficient},
which leads to faster rates and shows for
the first time that this objective can lead to meaningful guarantees in general Low-Rank
MDPs without latent variable structure.

\paragraph{Organization}
\pref{sec:setting} formally introduces the \LRMDP model and the online
reinforcement learning framework we consider. In 
\cref{sec:main}, we highlight challenges faced
by previous approaches, introduce our main algorithm, \mainalg, and
show how it overcomes these challenges, and then present its main
sample complexity guarantee. We conclude
with discussion in \cref{sec:discussion}.

\paragraph{Comparison to previous versions of the paper} An initial version of this paper, presented at NeurIPS 2023, also used barycentric spanners, but required reachability (the original algorithm was named \texttt{SpanRL}). A later revision, published on arXiv, removed the need for reachability at the cost of a larger sample complexity (albeit still polynomial) by using a slightly different algorithm that relied on a generalized optimal design computation (instead of a barycentric spanner computation) to compute a basis for exploration; this optimal design approach used novel algorithmic techniques and analysis, such as the use of the Frank Wolfe algorithm \citep{frank1956algorithm} for efficient approximate optimal design computation, and may be of independent interest. The current version of the paper reverts back to barycentric spanner computation with a slightly modified objective to remove reachability and enjoy the optimal $O(1/\veps^2)$ sample complexity.

\newcommand{\No}{\xmark}
\newcommand{\Yes}{\cmark}
\begin{table}[tp]
	\caption{Comparison of sample complexity required learn an $\veps$-optimal
		policy.
		$\Phi$ denotes the
		feature class, and $\textcolor{red!70!black}{\Upsilon}$ denotes an
		additional feature class capturing model-based
		function approximation.
		For approaches that require non-negative (latent variable)
		structure, $d_{\texttt{LV}}$ [resp.~$\gamma$] denotes the latent
		variable dimension [resp.~the reachability parameter in the latent representation], and for BMDPs,
		$\abs{\cS}$ denotes the size
		of the latent state space. For \mainalg, $\eta$
		denotes the reachability parameter.
	}
	\label{tb:resultscomp+}
	\renewcommand{\arraystretch}{1.6}
	\fontsize{9}{10}\selectfont
	\centering 
	\begin{tabular}{ccccc}
		\hline
		& Comp. efficient & Model-free & No addt'l assumptions %
		& Sample comp. \\
		\hline
		\makecell{$\olive$ \citep{jiang2017contextual}\tablefootnote{See
				also \citep{jin2021bellman,du2021bilinear,chen2022partially,xie2022role}} } & \No &
		\Yes
		& \Yes & 
		$\frac{d^3A
			H^5 \ln
			|\Phi|}{\veps^2}$\\
		$\texttt{FLAMBE}$ \citep{agarwal2020flambe} & \Yes
		&\No &\Yes{}\tablefootnote{For the
			stated sample complexity,
			\texttt{FLAMBE} requires
			access to a 
			sampling oracle for the
			learner model. Without
			this oracle, the results
			require additional
			latent variable structure
			and a reachability assumption.}
		& $\frac{d^{7}A^9H^{22}\ln (|\Phi|\textcolor{red!70!black}{|\Upsilon|})}{\veps^{10}}$
		\\
		\makecell{$\texttt{Rep-UCB}$
			\citep{uehara2022representation}\\(see also \citep{cheng2023improved})} &
		\Yes & \No & \Yes & $\frac{ d^4A^2 H^5\ln (|\Phi|\textcolor{red!70!black}{|\Upsilon|}) }{\veps^2}$\\
		$\texttt{MOFFLE}$
		\citep{modi2021model}\tablefootnote{We compare to
			the variant of \texttt{MOFFLE} that uses the same
			representation learning objective we consider. Other
			variants have improved sample complexity, but make
			use of stronger oracles.} & \Yes&
		\Yes
		& \No{}\makecell{Non-negative/\\latent
			variable} &
		$\frac{d^{19}_{\texttt{LV}}A^{32}H^{19}\ln
			|\Phi|}{(\veps^6\gamma^3
			\wedge
			\gamma^{11})}$
		\\		
		$\texttt{BRIEE}$ \citep{zhang2022efficient} &
		\Yes
		& \Yes & \No{} Block MDP & $ \frac{\abs{\cS}^8 A^{14} H^9 \ln |\Phi|}{\veps^4}$ \\
		\rowcolor[gray]{0.9}
		\mainalg (this paper) & \textbf{\Yes} &
		\textbf{\Yes}
		&
		\textbf{\Yes}
		& $\frac{ d^{13}A^2H^6 (d + \ln |\Phi|) }{\veps^2}$\\
		\hline
	\end{tabular}
\end{table}

	\section{Problem Setting}
	\label{sec:setting}

\subsection{\LRMDP Model}

We work in an episodic, finite-horizon reinforcement learning framework, where $H\in\bbN$ denotes the horizon. A \emph{\LRMDP} \citep{RendleFS10,YaoSPZ14,agarwal2020flambe} is a tuple $\cM=(\cX,\cA, (\phistarh)_{h\in [H]},(\muh[h])_{h\in[H]},\initd)$ consisting of a \emph{state space} $\cX$, \emph{action space} $\cA$ with $\abs{\cA}=A$, distribution over initial states $\initd \in \Delta(\cX)$, and mappings $\muh:\cX\rightarrow \reals^d$ and $\phistarh: \cX \times \cA \rightarrow \reals^d$.\footnote{We emphasize that neither $\muh[h]$ nor $\phistarh$ is known to the agent, in contrast to the linear MDP setting \citep{YangW19,jin2020provably}.}
Beginning with $\x_1\sim{} \initd$, an episode proceeds in $H$ steps, where for each step $h\in\brk{H}$, the state $\bx_h$ evolves as a function of the agent's action $\ba_h$ via
\begin{align}
  \label{eq:transition_kernel}
  \bx_{h+1}\sim{}T_h(\cdot\mid{}\bx_h,\ba_h),
\end{align}
where $T_h$ is a probability transition kernel, which is assumed to factorize based on $\phistarh$ and $\mustarh$. In detail, we assume that there exists a $\sigma$-finite measure $\nu$ on $\cX$ such that for all $1 \leq h \leq H-1$, and for all $x \in \cX$ and $a \in \cA$, the function $x' \mapsto \muh(x')^\top \phistarh(x, a)$ is a probability density with respect to $\nu$ (i.e. the function is everywhere non-negative and integrates to $1$ under $\nu$). For any $\cX'\subseteq\cX$, the probability that $\bx_{h+1}\in\cX'$ under $\bx_{h+1}\sim{}T_h(\cdot\mid{}x_h,a_h)$ is then assumed to follow the law 
\begin{align}
  \label{eq:transition_factor}
T_h(\cX'\mid{}x_h,a_h) = \int_{\cX'} \muh(x)^\top \phistarh(x_h, a_h) \dd\nu(x).
\end{align}
For notational compactness, we assume (following, e.g., \citet{jiang2017contextual}) that the MDP $\cM$ is \emph{layered} so that $\cX = \cX_1\cup \dots\cup  \cX_H$ for $\cX_i \cap \cX_j=\emptyset$ for all $i\neq j$, where $\cX_h\subseteq \cX$ is the subset of states in $\cX$ that are reachable at layer $h\in[H]$. This can be seen to hold without loss of generality (modulo dependence on $H$), by augmenting the state space to include the layer index.

\begin{remark}[Comparison to previous formulations]
  Our formulation, in which the transition dynamics \eqref{eq:transition_factor} are stated with respect to a base measure $\nu$, are a rigorous generalization of \LRMDP formulations found in previous works \citep{jin2020provably,agarwal2020flambe}, which tend to implicitly assume the state space is countable and avoid rigorously defining integrals. We adopt this more general formulation to emphasize the applicability our results to continuous domains. However, in the special case where state space is countable, choosing $\nu$ as the counting measure yields $T_h(\cX'\mid{}x_h,a_h) = \sum_{x\in\cX'}\muh(x)^\top \phistarh(x_h, a_h)$, which is consistent with prior work.
\end{remark}

\paragraph{Policies and occupancy measures}
We define $\Pim=\crl*{\pi:\cX\to\Delta(\cA)}$ as the set of all randomized, Markovian policies. For a policy $\pi\in\Pim$, we let $\bbP^{\pi}$ denote the law of $(\bx_1,\ba_1),\ldots,(\bx_H,\ba_H)$ under $\ba_h\sim\pi(\bx_h)$, and let $\En^{\pi}$ denote the corresponding expectation. For any $\cX'\subseteq \cX_h$, we let $\P_h^{\pi}[\cX']\coloneqq \P^\pi[\x_h \in \cX']$ denote the marginal law of $\bx_h$ under $\pi$. For $x\in\cX_h$, we define the \emph{occupancy measure} $d^\pi(x) \coloneqq \frac{\dd \P_h^\pi}{\dd \nu}(x)$ as the density of $\bbP^{\pi}_h$ with respect to $\nu$.

\subsection{Online Reinforcement Learning and Reward-Free Exploration}
\label{sec:onlineRL}
We consider a standard \emph{online reinforcement learning} framework where the Low-Rank MDP $\cM$ is unknown, and the learning agent interacts with it in \emph{episodes}, where at each episode the agent executes a policy of the form $\pi:\cX\to\Delta(\cA)$ and observes the resulting trajectory $(\x_1,\a_1),\ldots,(\x_H,\a_H)$.
While the ultimate goal of reinforcement learning is to optimize a policy with respect to a possibly unknown reward function, here we focus on the problem of
\emph{reward-free exploration}, which entails learning a collection of policies that almost optimally ``covers'' the state space, and can be used to efficiently optimize any downstream reward function \citep{du2019latent,misra2020kinematic,efroni2021provably,mhammedi2023representation}. To wit, we aim to construct an \emph{policy cover}, a collection of policies that can reach any state with near-optimal probability.
\begin{definition}[Approximate policy cover]
	\label{def:polcover101}
        For $\alpha,\veps\in(0,1]$, a subset $\Psi \subseteq \Pim$ is an $(\alpha,\veps)$-policy cover for layer $h$ if
	\begin{align}
          \label{eq:polcover101}
          \max_{\pi \in \Psi} d^{\pi}(x)\geq  \alpha \cdot \max_{\pi' \in \Pim} d^{\pi'}(x) \quad \text{for all $x\in \cX_{h}$ such that} \quad  \max_{\pi'\in \Pi} d^{\pi'}(x)\geq \veps \cdot \|\muh[h](x)\|.
	\end{align}
\end{definition}
Informally, an $(\alpha,\veps)$-policy cover $\Psi$ has the property that for every state $x\in\cX$ that is reachable with probability at least $\veps\cdot \|\muh[h](x)\|$, there exists a policy in $\Psi$ that reaches it with probability at least $\alpha\cdot\veps \cdot \|\muh[h](x)\|$. We show (\cref{sec:reward_based}) that given access to such a policy cover with $\alpha =\poly(\veps, d^{-1} ,A^{-1})$, it is possible to optimize any downstream reward function to $\bigoh(\veps)$ precision with polynomial sample complexity.

\begin{remark}
\pref{def:polcover101} generalizes the notion of approximate policy cover used by \citet{mhammedi2023representation} for the Block MDP setting; as in that work, the definition allows one to sacrifice states for which the maximum occupancy is small, which is necessary in the absence of reachability-style assumptions \citep{misra2019kinematic,modi2021model,agarwal2022}. Compared to \citet{mhammedi2023representation}, we replace the Block MDP condition $\max_{\pi \in \Pim} d^{\pi}(x)  \geq \veps$ by $\max_{\pi \in \Pim} d^{\pi}(x)  \geq \veps\cdot \|\muh[h](x)\|$. As our analysis shows, the latter condition turns out to be better suited to the $\ell_2$ geometry of the \LRMDP model, and is sufficient for the purpose of optimizing downstream reward functions up to $O(\veps)$ precision (\cref{sec:reward_based}). %
\end{remark}
In the analysis, it will be convenient to slightly generalize \cref{def:polcover101}.
\begin{definition}
	\label{def:randcover}
	For $\alpha, \veps\in (0,1]$, a distribution $P\in \Delta(\Pim)$ is an $(\alpha,\veps)$-randomized policy cover for layer $h$ if 
	\begin{align}
		\E_{\pi \sim P}[d^{\pi}(x)] \geq \alpha \cdot \max_{\pi'\in \Pim} d^{\pi'}(x) \quad \text{for all $x\in \cX_h$ such that  } \max_{\pi'\in \Pim} d^{\pi'}(x)\geq \veps \cdot \|\mu_h^\star(x)\|. 
	\end{align}	
\end{definition}
\begin{remark}
	\label{rem:backandforth}
	We note that if $P$ is an $(\alpha,\veps)$-randomized policy cover for layer $h$ (\cref{def:randcover}) and has finite support $\Psi = \supp P$, then $\Psi$ is an $(\alpha, \veps)$-policy cover for layer $h$ according to \cref{def:polcover101}. Conversely, if $\Psi$ is an $(\alpha,\veps)$-policy cover for layer $h$, then $\unif(\Psi)$ is a randomized policy cover $(\alpha/|\Psi|, \veps)$-policy cover for layer $h$ according to \cref{def:randcover}.
\end{remark}

      \paragraph{Function approximation and desiderata}
      We do not assume that the true features $(\phistarh)_{h\in[H]}$ or the mappings $(\muh[h])_{h\in[H]}$ are known to the learner.
To provide sample-efficient learning guarantees we make use of function approximation as in prior work \citep{AgarwalKKS20,modi2021model}, and assume access to a \emph{feature class} $\Phi\subseteq\{ \phi : \cX\times \cA\to\reals^d\}$ that contains $\phistarh$, for $h\in[H-1]$.
\begin{assumption}[Realizability]
	\label{assum:real}
	The feature class $\Phi\subseteq\{ \phi : \cX\times \cA\to\reals^d\}$ has $\phistarh\in \Phi$ for all {$h\in[H]$}.  Moreover, for all $\phi \in \Phi$, $x \in \cX$, and $a \in \cA$, it holds that $\norm{\phi(x, a)} \leq 1$.
\end{assumption}
The class $\Phi$ may consist of linear functions, neural networks, or other standard models depending on the application, and reflects the learner's prior knowledge of the underlying MDP. We assume that $\Phi$ is finite to simplify presentation, but extension to infinite classes is straightforward, as our results only invoke finiteness through standard uniform convergence arguments.
Note that unlike model-based approaches \citep{agarwal2020flambe,uehara2022representation,cheng2023improved,agarwal2020model}, we do not assume access to a class capable of realizing the features $\mustarh$, and our algorithm does not attempt to learn these features; this is why we distinguish our results as \emph{model-free}.\loose

Beyond realizability, we assume (following \cite{agarwal2020flambe,modi2021model}) for normalization that, for all $h\in[H]$ and $(x,a)\in \cX_h\times \cA$, $\nrm*{\phistar_h(x,a)}\leq{}1$, and that for all $g:\cX_h\to\brk{0,1}$,
  \begin{align}
    \label{eq:normalization}
\nrm*{\int_{\cX_h} \muh[h](x)g(x) \dd\nu(x)} \leq \sqrt{d}.
\end{align}
For $\veps\in(0,1)$, our goal is to learn an $(\alpha,\veps)$-policy cover with $\alpha= \poly(\veps,d^{-1},A^{-1})$using  \[
\poly(d,A,H,\log\abs{\Phi},\veps^{-1})
\] episodes of interaction. 
This guarantee scales with the dimension $d$ of the feature map and the complexity $\log\abs{\Phi}$ of the feature class but, critically, does not depend on the size of the state space $\cX$; note that by \cite{cheng2023improved}, dependence on both $H$ and $A=\abs{\cA}$ is necessary when $\phistar$ is unknown. Given such a guarantee, we show in \Cref{sec:reward_based} that it is possible to optimize any downstream reward function to error $\veps$ with polynomial sample complexity.\loose

\paragraph{Additional preliminaries}
For any $m,n \in\mathbb{N}$, we denote by $[m\ldotst{}n]$ the integer interval $\{m,\dots, n\}$. We also let $[n]\coloneqq [1\ldotst{}n]$. For any sequence of objects $o_1, o_2,\dots$, we define $o_{m:n}\coloneqq (o_{i})_{i\in[m \ldotst n]}$. 
A \emph{partial policy} is a policy defined over a contiguous subset of layers $\brk{\ell\ldotst{}r}\subseteq\brk{H}$. We denote by $\Pim^{\ell:r} \coloneqq \left\{\pi \colon \bigcup_{h=\ell}^r  \cX_h \rightarrow \Delta(\cA)\right\}$ the set of all partial policies over layers $\ell$ to $r$; note that $\Pim \equiv \Pim^{1:H}$. For a policy $\pi\in\Pim^{\ell:r}$ and $h\in\brk{\ell\ldotst{}r}$, $\pi(x_h)$ denotes the action distribution for the policy at layer $h$ when $x_h\in\cX_h$ is the current state. For $1\leq t\leq h\leq H$ and any pair of partial policies $\pi \in \Pim^{1:t-1}, \pi'\in \Pim^{t:h}$, we define $\pi \circ_t \pi'\in\Pim^{1:h}$ as the partial policy given by $(\pi \circ_t \pi')(x_{\ell}) = \pi(x_{\ell})$ for all $\ell<t$ and $(\pi \circ_t \pi')(x_{\ell}) = \pi'(x_{\ell})$ for all $\ell \in [t\ldotst h]$. We define $\pi \circ_t \pi'$ in the same fashion for $\pi \in \Pim^{1:\ell}$ for $\ell\geq t$.

We use the $\x_h\sim \pi$ as shorthand to indicate that $\bx_h$ is drawn from the law $\bbP^{\pi}$, and likewise for $(\x_h,\ba_h)\sim \pi$ and so on. For a set of partial policies $\Psi \coloneqq \{\pi^{(i)}\colon i \in [N]\}$, we define $\unif(\Psi)$ as the random partial policy obtained by sampling $\bi\sim \unif([N])$ and playing $\pi^{(\bi)}$. We define $\unifa\in\Pim$ as the random policy that selects actions in $\cA$ uniformly at random at each layer.
We use $\nrm*{\cdot}$ to denote the Euclidean norm, $\nrm*{\cdot}_\infty$ to denote the supremum norm on functions, and let $\cB(r)\subseteq\bbR^{d}$ denote the Euclidean ball of radius $r$. We let $\cB_{\F}(r)$ be the Frobenius ball of radius $r>0$ in $\reals^{d\times d}$. We denote by $\psd$ the set of positive semi-definite matrices in $\reals^{d\times d}$, and by ``$\preceq$'' the corresponding partial order. For a vector $v\in \reals^d$, we denote by $v[i]$ its $i$th coordinate. %
We refer to a scalar $c>0$ as an \emph{absolute constant} to indicate that it is independent of all problem parameters and use $\bigoht(\cdot)$ to denote a bound up to factors polylogarithmic in parameters appearing in the expression.

	\section{\mainalg: Algorithm and Main Results}
	\label{sec:main}

In this section, we present the \mainalg algorithm. We begin by describing 
challenges in deriving efficient, model-free algorithms using existing
approaches (\cref{sec:challenges}). We then formally describe \mainalg (\cref{sec:algorithm}) and build intuition as to
how it is able to overcome these challenges, and finally state our main sample
complexity guarantee (\cref{sec:main_theorem}).

\subsection{Challenges and Related Work}
\label{sec:challenges}

Designing algorithms with provable guarantees in the Low-Rank MDP setting is challenging because of the complicated interplay between representation learning and exploration. Indeed, while there are many efficient algorithms for the so-called \emph{linear MDP} setting where the feature maps $(\phistarh)_{h\in[H]}$ are known (removing the need for representation learning) \citep{jin2020provably,zhang2021feel,agarwal2022vo,wang2020reward}, these approaches do not readily generalize to accommodate unknown features. For Low-Rank MDPs, previous algorithms suffer from at least one of the following three drawbacks: (1) the algorithms are computationally inefficient; (2) the algorithms are model-based; or (3) the algorithms place strong assumptions on the MDP that are unlikely to hold in practice. To motivate the {\mainalg} algorithm, we briefly survey these results, highlighting several key challenges in avoiding these pitfalls.

Let us first discuss the issue of computational efficiency. While there are a number of algorithms---all based on the principle of \emph{optimism in the face of uncertainty}---that provide tight sample complexity guarantees for Low-Rank MDPs in reward-based \citep{jiang2017contextual,jin2021bellman,du2021bilinear} and reward-free \citep{chen2022partially,xie2022role} settings, these algorithms involve intractable optimization problems, and cannot be implemented efficiently even when the learner has access to an optimization oracle for the representation class $\Phi$ \citep{dann2018oracle}. This intractability arises because these algorithms implement optimism via a ``global'' approach, in which the algorithm explores at each round by choosing the most optimistic value function in a certain \emph{version space} of candidate value functions; optimizing over this version space is challenging, as it involves satisfying non-convex constraints with a complicated dependence on the learned representation that are coupled globally across layers $h\in\brk{H}$.

To avoid the intractability of global optimism, several works have restricted attention to a simpler \emph{model-based} setting. Here, in addition to assuming that the feature maps $(\phistarh)_{h\in[H]}$ are realizable with respect to $\Phi$, one assumes access to a second feature class $\Upsilon$ capable of modeling the mappings $(\mustarh)_{h\in[H]}$; this facilitates direct estimation of the transition probability kernel $T_h(\cdot\mid{}x,a)$. For the model-based setting, it is possible to efficiently implement certain ``local'' forms of optimism \citep{uehara2022representation,cheng2023improved,zhang2022making}, as well as certain non-optimistic exploration techniques based on policy covers \citep{agarwal2020flambe}. \arxiv{For example, one can estimate features using maximum likelihood, and then apply efficient algorithms for the known-feature setting with the estimated features plugged-in \citep{jin2020provably,zhang2021feel,agarwal2022vo,wang2020reward}; here, a key insight is that model-based estimation leads to strong distribution transfer guarantees for the learned features. As a result, there are now a number of efficient model-based algorithms \citep{agarwal2020flambe,uehara2022representation,cheng2023improved}, some of which have been practically implemented \citep{zhang2022making}.} Unfortunately, model-based realizability is a restrictive assumption, and falls short of the model-free guarantees we aim for in this work; indeed, in general, one cannot hope to estimate the feature map $\muh$ without sample complexity scaling with the number of states.\footnote{For example, in the special case of the Block MDP setting \citep{du2019latent,misra2019kinematic}, model-based realizability entails modeling a certain emission process, which is not required by model-free approaches.} \loose

When one moves from model-based learning to model-free learning, representation learning becomes substantially more challenging---both for optimistic and non-optimistic approaches. Here, a key challenge is to develop representation learning procedures that are (1) efficient, yet (2) provide meaningful guarantees when the learned features are used downstream for exploration.
To our knowledge, the only proposal for a representation learning procedure satisfying both desiderata comes from the work of \citet{modi2021model}, who introduced a promising ``minimax'' representation learning objective (described in detail in the sequel; cf. \cref{alg:newreplearn}), which \citet{zhang2022efficient} subsequently showed to have encouraging empirical performance. However, to provide guarantees for this objective, both works place substantial additional restrictions on the low-rank factorization. In particular, \citet{modi2021model} make the so-called \emph{latent variable} assumption \citep{agarwal2020flambe}, which asserts that $\phistarh$ and $\mustarh$ are non-negative coordinate-wise, and \citet{zhang2022efficient} further restrict to the Block MDP model \citep{du2019latent,misra2019kinematic}. %
Non-negativity is a substantial restriction, as the best non-negative factorization can have exponentially large dimension relative to the best unrestricted factorization \citep{agarwal2020flambe}. Beyond non-negativity, many prior works \citep{du2019latent,misra2019kinematic,modi2021model} require \emph{reachability} assumptions, the weakest of which asserts that there exists $\eta>0$ such that for all $x\in\cX_h$,
  \begin{align}
    \label{eq:reachability_basic}
  \max_{\pi \in \Pim} d^{\pi}(x)\geq \eta \cdot \|\muh[h](x)\|.
\end{align}
These works give sample complexity bounds that scale polynomially in $\eta^{-1}$, and do not give any guarantee when $\eta=0$.\footnote{When specialized to tabular MDPs, reachability asserts that for each state $x\in\cX$, there exists a policy that reaches $x$ with probability at least $\eta$.} The source of both restrictions is the problem of how to quantify how close a learned representation $\phi$ is to the ground truth $\phistar$, which depends strongly on the downstream exploration strategy. In what follows, we show that with the right exploration strategy, this challenge can be ameliorated, but prior to our work it was unclear whether the minimax objective could lead to meaningful guarantees in the absence of non-negativity. %

\subsection{The \mainalg Algorithm}
\label{sec:algorithm}

Our algorithm, \mainalg, is presented in \cref{alg:spanRL}. The
algorithm proceeds by building a policy cover layer-by-layer in an
inductive fashion. For each layer $h\geq{}2$, $\mainalg$ uses a policy cover
$\Psi\ind{h}$ built at a previous iteration within a subroutine, $\replearn$
(\cref{alg:newreplearn}; deferred to \cref{sec:replearn}) to produce a
feature map $\phihat\ind{h}$ that approximates $\phistarh$. Using this feature map, the algorithm invokes a second subroutine, \spanner{} (\cref{alg:spanner} in \cref{sec:spanner}) to produce a collection of policies $\pi_1,\ldots,\pi_d$ that act as a \emph{barycentric spanner} for the
feature map, ensuring maximal coverage in a certain sense; given these policies, a new policy cover for layer $h+2$ is formed via $\Psi\ind{h+2}=\{ \pi_i\circ_{h+1}
\pi_\unif : i\in[d] \}$. To invoke the \spanner{}
subroutine, \mainalg makes use of 
additional subroutines for policy optimization
(\psdp{}; \cref{alg:PSDP} in
\cref{sec:nonnegative}) and estimation of certain
vector-valued functionals (\veceval{}; \cref{alg:veceval}
in \cref{sec:spannerSpanRL}). We now describe
each component of the algorithm in detail,
highlighting how they allow us to overcome the
challenges in the prequel.

\paragraph{Barycentric spanners}
At the heart of \mainalg is the notion of a \emph{barycentric spanner}
\citep{awerbuch2008online} as an efficient basis for exploration. We
begin by defining a barycentric spanner for an abstract set of vectors
in $\bbR^d$.
\begin{definition}[\citet{awerbuch2008online}]\label{def:barycentricspanner}
	Given a set $\cW \subset \rr^d$ such that $\lspan(\cW) = \rr^d$, we say that a set $\{ w_1, \dots, w_d \}\subseteq\cW$ is a $(C, \veps)$-approximate barycentric spanner for $\cW$ if for every $w \in \cW$, there exist $\beta_1, \dots, \beta_d \in [-C, C]$ such that $\norm{w - \sum_{i = 1}^d \beta_i w_i}\leq \veps$.\footnote{Note that our definition is a slight generalization of \cite[Definition 2.1]{awerbuch2008online}; the latter is recovered with $\veps = 0$.}
\end{definition}
The utility of barycentric spanners for reward-free exploration is
highlighted in the following lemma. 
\begin{lemma}\label{lem:barycentricspannerknownphi}
	If $\Psi\subseteq \Pim$ is a collection of policies such that $\{\ee^\pi\left[
	\phistarh(\x_h, \a_h) \right]\mid \pi \in \Psi \}\subseteq
	\rr^d$ is a $(C,
	\veps)$-approximate barycentric spanner for $\cW_h\coloneqq \{\ee^\pi\left[
	\phistarh(\x_h, \a_h) \right]\mid \pi \in \Pim \}$, then $\Psi$ is an $(\alpha,2\veps)$-policy cover for layer $h+1$ with $\alpha = (2dC)^{-1}$.
\end{lemma}

	\Cref{lem:barycentricspannerknownphi}, proven in \Cref{ssec:pf_baryspanner_knownphi}, shows that to compute a policy
	cover for layer $h+1$, it suffices to find a barycentric spanner for the
	set $\cW_h\coloneqq \{ \ee^\pi\left[
	\phistarh(\x_h, \a_h) \right]\mid \pi \in \Pim \}\subseteq
	\rr^d$. 
	It turns out that building a barycentric spanner for the set $\cW'_h \coloneqq \{\ee^\pi\left[
	\phistarh(\x_h, \a_h)\cdot \mathbb{I}\{\phistarh(\x_h, \a_h)^\top \theta\geq 0 \}  \right]\mid (\pi, \theta) \in \Pim \times \reals^d\}\subseteq
	\rr^d$ instead of $\cW_h$ is also sufficient for computing a
        policy cover. Going forward, we target the set $\cW'_h$
        (henceforth simply $\cW_h$), as
        this turns out to enable an analysis that handles distribution
        shift more effectively, avoiding compounding
        errors.\footnote{\label{foot:note}In the first version of this
          work, we targeted the set $\cW_h$, which necessitated a
          certain reachability assumption. By targeting the new set
          $\cW'_h$, we are able to remove the need for
          reachability. At a high level, this works because computing
          a barycentric spanner for the set $\cW_h'$ can be reduced to
          policy optimization with \emph{non-negative} rewards, while
          $\cW_h$ requires signed rewards. Working with non-negative
          rewards enables us to generalize the extended MDP analysis from \citet{mhammedi2023representation} (which relies crucially on non-negative rewards) to the Low-Rank MDP setting.}

Of course, even if $\phistarh$ is known, \cref{lem:barycentricspannerknownphi} is only useful if we
can compute a spanner without explicitly enumerating over the set
$\Pim$, since our goal is to develop an \emph{efficient} algorithm. In what follows, we will show:\footnote{While barycentric spanners have been used in a number of recent
	works on sample-efficient RL
	\citep{golowich2022learning,huang2023reinforcement}, the
	motivation for their use within our algorithm and
	analysis are quite different; see \cref{sec:additional_related}.
	}\loose
\begin{enumerate}[leftmargin=20pt]
	\item Using, \spanner{}, a novel adaptation of the classical
          spanner computation algorithm of
	\citet{awerbuch2008online}, it holds that for any $\phi\in\Phi$,
	spanner computation for the set $\{\ee^\pi\left[
	\phistarh(\x_h, \a_h)\cdot \mathbb{I}\{\phistarh(\x_h, \a_h)^\top \theta\geq 0 \}  \right]\mid (\pi, \theta) \in \Pim \times \reals^d\}$ can be performed efficiently whenever, for any $\theta\in\cB(1)$, one can (approximately) solve linear optimization problems of the form
	\begin{align}
		\label{eq:linopt_policy}
		\argmax_{\pi\in\Pim}\ee^\pi\left[
		\phistarh(\x_h, \a_h)\cdot \mathbb{I}\{\phistarh(\x_h, \a_h)^\top \theta\geq 0 \}\right].
	\end{align}
	\item Given access to policy covers $\Psi\ind{1:h}$ for layers $1$ to $h$, one can efficiently solve the optimization problem in \eqref{eq:linopt_policy} by
	appealing to the \psdp algorithm for policy
	optimization (\cref{alg:PSDP}).
\end{enumerate}
To handle the fact that $\phistarh$ is unknown, \cref{alg:spanRL} computes policies $\pi_{1:d}$ that induce a barycentric spanner for the set $\{ \ee^\pi[\phihat\ind{h}(\x_h, \a_h) \cdot \mathbb{I}\{\phihat\ind{h}(\x_h, \a_h)^\top \theta\geq 0 \}]\mid (\pi,\theta) \in \Pim \times \reals^d\}$, where
$\phihat\ind{h}\in \Phi$ is a learned feature map. In what follows, we
first give a detailed explanation of the two points above, before showing how
to complete the argument by learning a feature map through representation learning.

\begin{algorithm}[h]
	\caption{\mainalg: Volumetric Exploration and Representation Learning via Barycentric Spanner}
	\label{alg:spanRL}
	\begin{algorithmic}[1]\onehalfspacing
		\Require Feature class $\Phi$ and parameters $\veps,\cfrak>0$ and $\delta\in(0,1)$.
\State Set $\Psi\ind{1}=\emptyset$, $\Psi\ind{2}=\{\pi_{\unif}\}$.
\State Set $\eta = \veps/(4 H d^{3/2})$, $n_{\replearn}= {\cfrak}{}  \cdot \eta^{-2} A^2 d^8\ln(|\Phi| /\delta)$ and $n_{\veceval}={\cfrak}{} \cdot  \eta^{-2} d^5 { \ln (1/\delta)}$.
\State Set $n_{\psdp} =   { \cfrak }{}\cdot  \eta^{-2} A^2 d^8  H^2  \cdot (d +\ln (|\Phi|/\delta))$. 
		\State Define $	\cF \coloneqq 
		\left\{ \left. f\colon x \mapsto \max_{a\in \cA}\theta^\top \phi(x,a)  \,  \right|   \, \theta\in \cB(1), \phi \in \Phi
		\right\}$.\label{line:func}
		  \State Define $\cG = \{g:(x,a)\mapsto \phi(x,a)^\top w \mid \phi \in \Phi , w \in \cB(2\sqrt{d})\}$.
		\For{$h=1,\ldots, H-2$} \label{line:mainiter}
		\Statex~~~~~ \algcommentbiglight{Learn feature representation for layer $h$.}
		\State Set $\phih\ind{h} = \replearn(h, \cF,\Phi,P\ind{h},n_{\replearn})$, with $P\ind{h}=\unif(\Psi\ind{h})$.\label{line:oldreplearn} \hfill\algcommentlight{\cref{alg:newreplearn}.}
		\Statex~~~~~  \algcommentbiglight{Computing an approximate
                  spanner using learned features.}
              \State For $\theta \in \reals^d$ and $(x,a)\in \cX\times \cA$, define \label{line:reward}
              \begin{align} r_{t}(x,a;\theta)\ldef{} \left\{\begin{array}{ll} \phih\ind{h}(x,a)^\top{\theta} \cdot \mathbb{I}\{\phi\ind{h}(x,a)^\top \theta \geq 0\}, &  \text{for }
              t=h, \\ 0, &  \text{otherwise}.
              \end{array}\right. \label{eq:rewards}
              \end{align} 
        \State For each $t\in[h]$, set $\cG_{t}=\cG$ and $P\ind{t}=\unif(\Psi\ind{t})$.
                \State For $\theta\in\reals^d$, define
                $\apx(\theta)=(\psdp(h, r_{1:h}(\cdot,\cdot;\theta), \cG_{1:h},P\ind{1:h},
                n_{\psdp}), \theta) \in \Pim \times \reals^d$. \label{line:psdp}\label{line:linopt} \hfill\algcommentlight{\cref{alg:PSDP}.}
                   \State For $\theta \in \reals^d$ and $\pi\in\Pim$, define
                $\est(\pi,\theta)=\veceval(h,\phih\ind{h} \cdot\mathbb{I} \{ \theta^\top \phih\ind{h}\geq 0\}, \pi,
                n_{\veceval})$. \label{line:est} \hfill\algcommentlight{\cref{alg:veceval}.}
                \State \label{line:spanner}Set $((\pi_{1}, \theta_{1}), \dots, (\pi_d,\theta_d)) =  \spanner(\apx(\cdot), \est(\cdot), 2,  \frac{\eta}{36 d^{5/2}})$. \hfill\algcommentlight{\cref{alg:spanner}.} 
		\State Set $\Psi\ind{h+2}=\{ \pi_i\circ_{h+1}
                \pi_\unif : i\in[d] \}$. \label{line:cover}
		\EndFor
		\State \textbf{Return:} Policy cover
		$\Psi\ind{1:H}$. 
	\end{algorithmic}
\end{algorithm}

\paragraph{Barycentric spanner computation via approximate linear optimization}

To describe spanner computation in \mainalg, we take a brief detour and consider an abstract approach to
barycentric spanner computation, which generalizes our problem. Suppose that we wish
to compute a spanner for an implicitly specified set
$\cW=\crl*{w^z}_{z\in\cZ}\subseteq\bbR^{d}$ indexed by an abstract set
$\cZ$. The set $\cZ$ (which will be set to $\Pim$ when we return to
RL) may be exponentially large and cannot be efficiently enumerated. In addition, given $z\in\cZ$, we
cannot explicitly compute $w^{z}$, and have to settle for a noisy approximation.

To allow for efficient spanner computation, we assume access to two
oracles for the set $\cW$, a \emph{linear optimization} oracle $\apx:\cB(1)\to\cZ$ and
an \emph{index-to-vector} oracle $\est:\cZ\to\bbR^{d}$. We assume that for some $\veps>0$:
\begin{enumerate}[leftmargin=20pt]
	\item For all $\theta\in\bbR^{d}$ with $\nrm*{\theta}=1$, the output
	$\hat{z}_{\theta}\ldef \apx(\theta)$ satisfies
	$\theta^{\top}w^{\hat{z}_\theta} \geq     \sup_{z\in \cZ}
	\theta^\top w^{z} - \veps$.
\item For all $z\in\cZ$, the output $\hat{w}_z\ldef{}\est(z)$
satisfies
	$\|\hat{w}_z - w^{z}\| \leq \veps$.
	\end{enumerate}
The \spanner{} algorithm
(\cref{alg:spanner}) computes a $(C,\veps)$-approximate spanner for
$\cW$ using
$\bigoh(d\log(d/\veps))$ total calls to $\apx$ and $\est$. \spanner{} is an error-tolerant variant of the classical spanner computation algorithm of
\citet{awerbuch2008online}, which was originally introduced and
analyzed for
spanner computation with an \emph{exact} linear optimization
oracle. Tolerance to approximation errors in the linear optimization oracle
is critical for our application to RL, where additive
errors will arise from sampling trajectories, as well as estimating
the feature maps $(\phistarh)_{h\in[H]}$. \spanner{} achieves error tolerance by
perturbing the vectors returned by $\apx(\theta)$ in the direction of
$\theta$, which amounts to running the classical algorithm on an $\veps$-fattening of $\cW$, and is necessary in order to ensure that the approximation error of $\apx$ does not swamp the signal in directions $\theta$ in which $\cW$ is too ``skinny.''  This technique may be of independent interest; see \cref{sec:spanner}
for additional details and formal guarantees.

\arxiv{
\begin{algorithm}[htp]
  \caption{\spanner: Barycentric Spanner via Approximate
    Linear Optimization}
	\label{alg:spanner}
	\begin{algorithmic}[1]\onehalfspacing
		\Require~
		\begin{itemize}[leftmargin=*]

                        \item Approximate linear optimization subroutine
                          $\apx:\reals^d\to \cZ$. \hfill \algcommentbiglight{See \cref{sec:algorithm}}
                        \item Approximate index-to-vector subroutine
                        $\est:\cZ\rightarrow  \reals^d$.
                              \item Parameters $C,\veps>0$.
		\end{itemize}
		\State Set $W =(w_1,\dots,w_d)= (e_1,\dots,e_d)$.
		\For{$i=1,\dots, d$} \label{line:firstfor}
		\State Set $\theta_i =(\det(e_j,W_{-i}))_{j\in[d]}\in
                \reals^d$. \hfill\algcommentlight{$W_{-i}$ is defined to be
                  $W$ without the $i$th column}
		\State Set $z_i^+ = \apx(\theta_i/\|\theta_i\|)$ and $w_i^+= \est(z_i^+)$. 
		\State Set $z_i^- = \apx(-\theta_i/\|\theta_i\|)$ and $w_i^-= \est(z_i^-)$.
		\If{$\theta_i^\top w^+_i\geq - \theta_i^\top w^-_i$} \label{line:if} 
		\State Set $\wtilde w_i= w^+_i$, $z_i = z_i^+$, and $w_i = \wtilde w_i + \veps \theta_i/\|\theta_i\|$.
		\Else
		\State Set $\wtilde w_i= w^-_i$, $z_i = z_i^-$, and $w_i = \wtilde w_i - \veps \theta_i/\|\theta_i\|$.
		\EndIf \label{line:endif}
		\EndFor
		\For{$n=1,2,\dots$} \label{line:for}
		\State Set $i=1$.
		\While{$i\leq d$} 
		\State Set $\theta_i =(\det(e_j,W_{-i}))_{j\in[d]}\in \reals^d$.
		\State Set $z_i^+ = \apx(\theta_i/\|\theta_i\|)$ and $w_i^+= \est(z_i^+)$. 
\State Set $z_i^- = \apx(-\theta_i/\|\theta_i\|)$ and $w_i^-= \est(z_i^-)$.
		\If{$\theta_i^\top w_i^+ +\veps \cdot \|\theta_i\|  \geq C \cdot |\det(w_i, W_{-i})|$}
		\State Set $\wtilde w_i = w_i^+$, $z_i = z_i^+$, and $w_i = \wtilde w_i  + \veps \cdot \theta_i/\|\theta_i\|$.
		\State \textbf{break}
		\ElsIf{$-\theta_i^\top w_i^- +\veps \cdot \|\theta_i\|  \geq C \cdot |\det(w_i, W_{-i})|$}
		\State  Set $\wtilde w_i = w_i^-$, $z_i = z_i^-$, and $w_i = \wtilde w_i  - \veps \cdot \theta_i/\|\theta_i\|$.
		\State \textbf{break}
		\EndIf
		\State Set $i = i+1$.
		\EndWhile
		\If{$i=d+1$}
		\State  \textbf{break}
		\EndIf 
		\EndFor
		\State \textbf{Return:} $(z_1, \dots, z_d)$.
	\end{algorithmic}
\end{algorithm}
 }

\paragraph{Representation learning}

Ideally, we would
like to use \spanner{} to construct a barycentric spanner for the set $\{
\ee^\pi[\phistar_h(\x_h, \a_h)] \cdot \mathbb{I}\{\phistar_h(\x_h,\a_h)^\top \theta  \geq 0\}  \mid  (\pi,\theta) \in \Pim\times \reals^d\}$ with $\cZ=\Pim$.
Because we do not have access to $\phistar_h$, we instead apply \spanner{} with $\cW \coloneqq \{ \ee^\pi[\phihat\ind{h}(\x_h, \a_h) \cdot \mathbb{I}\{\phihat\ind{h}(\x_h, \a_h)^\top \theta\geq 0 \}]\mid (\pi,\theta) \in \Pim \times \reals^d\}$,
where $\phihat\ind{h}$ is a learned
representation. We now describe how the feature map
$\phihat\ind{h}$ is learned, then show how to use these learned features to
efficiently implement the oracles $\apx(\cdot)$ and $\est(\cdot)$.

To learn a representation for layer $h$, we use the $\replearn$ algorithm (\cref{alg:newreplearn}),
which was originally introduced in
\citet{modi2021model,zhang2022efficient}. The algorithm gathers a
collection of triples $(\x_h, \a_h, \x_{h+1})$ by rolling in to
$\bx_h$ with a policy sampled
uniformly from the policy cover $\Psi\ind{h}$ and selecting $\ba_h$
uniformly at random. Using this dataset, the algorithm
solves a sequence of adversarial training sub-problems
(\cref{line:replearn} of \cref{alg:newreplearn}) which involve
the feature class $\Phi$ and an auxiliary discriminator class $\cF:
\cX \to \rr$. As we discuss in detail in the sequel, these
sub-problems, described in \eqref{eq:adversarial},
are amenable to standard gradient-based training methods. The
sub-problems are designed to approximate the following ``idealized''
max-min-max representation learning objective:
\begin{align}
\label{eq:replearn}
\phihat\ind{h} \in \argmin_{\phi \in \Phi} \sup_{f \in \cF} \inf_{w } \ee^{\unif(\Psi\ind{h})\circ_h\piunif}\left[\left(
\inprod{\phi(\x_h, \a_h)}{w} - \En\brk*{f(\x_{h+1})\mid\bx_h,\ba_h}
\right)^2 %
\right].
\end{align}
The intuition for
this objective comes from the fact that in a Low-Rank MDP, for any function $f:\cX\to\bbR$, the quantity $\ee[ f(\x_{h+1})
\mid{} \x_h=x, \a_h=a ]$ is linear in
$\phistar_h(x, a)$. Thus, if $\cF$ is sufficiently expressive, we may hope that $\phihat\ind{h}$
and $\phistar$ are close. We adopt the simple discriminator class \iftoggle{neurips}{$\cF = \{ \left. x \mapsto \max_{a\in \cA}\inprod{\theta}{\phi(x, a)} \ \right| \  \theta \in \cB(1), \, \phi \in \Phi \}$.}
{\begin{align}
	\cF = \left\{ \left. x \mapsto \max_{a\in \cA}\inprod{\theta}{\phi(x, a)} \ \right| \  \theta \in \cB(1), \, \phi \in \Phi \right\}.
	\end{align}}
	We show that solving
	\eqref{eq:replearn} with this choice for $\cF$, which is simpler than that in \citep{modi2021model,zhang2022efficient}, yields an approximation
	guarantee for $\phihat\ind{h}$ that is suitable for downstream use in spanner computation
	for general Low-Rank MDPs.

	\begin{remark}[Improved analysis of \replearn]
\label{rem:replearn}
To facilitate an analysis of \mainalg{} that does not require reachability assumptions, we use
slightly different parameter values for \replearn{} than in
\citet{modi2021model,zhang2022efficient}, and provide a tighter sample
complexity bound (\cref{thm:newreplearn}) which may be of independent interest.

In more detail, prior work shows that the $\replearn$ algorithm solves
a variant of \eqref{eq:replearn} with
$w\in\cB(d^{1/2}\cdot\poly(\veps^{-1}))$, where $\veps>0$ is the desired
bound on mean-squared error. Due to the polynomial dependence on
$\veps^{-1}$, such a result would lead to vacuous
guarantees when invoked within our analysis of \mainalg. Our improved
analysis of $\replearn$, which is based on a determinantal potential
argument, shows that $w\in\cB(\poly(d))$ (independent of $\veps^{-1}$)
suffices. A secondary benefit of our improved bound is a faster rate with
respect to the number of trajectories.
\end{remark}

\paragraph{Putting everything together} Having learned $\phihat\ind{h}$ using $\replearn$, in \mainalg we apply \spanner{} with $\cW \coloneqq \{ \ee^\pi[\phihat\ind{h}(\x_h, \a_h) \cdot \mathbb{I}\{\phihat\ind{h}(\x_h, \a_h)^\top \theta\geq 0 \}]\mid (\pi,\theta) \in \Pim \times \reals^d\}$,
$\cZ=\Pim\times \reals^d$, and $C = 2$; that is, we plug-in the learned
representation $\phihat\ind{h}$ for the true representation
$\phistarh$.\footnote{Though the policies produced by the algorithm may not necessarily induce a spanner for $\cW_h= \{ \ee^\pi[
\phistarh(\x_h, \a_h) \mathbb{I}\{\phistarh(\x_h,\a_h)^\top \theta\geq 0 \} ]\mid (\pi,\theta) \in \Pim\times \reals^d \}$ (this would require ``point-wise'' representation learning guarantees, which we do not have), our analysis shows that they still suffice to build a policy cover for layer $h+2$.}
With this choice, implementing $\apx$ essentially entails (approximately) solving \[\argmax_{\pi \in \Pim}\ee^\pi[ {\theta}^\top\phihat\ind{h}(\x_h, \a_h) \cdot \mathbb{I}\{\theta^\top\phihat\ind{h}(\x_h, \a_h)\geq 0\}]\] for a given $\theta\in\cB(1)$, and implementing the $\est$ oracle entails estimating $\ee^\pi[\phihat\ind{h}(\x_h, \a_h) \cdot \mathbb{I}\{\theta^\top\phihat\ind{h}(\x_h, \a_h)\geq 0\}]$ for a given $(\pi,\theta)\in\Pim\times \reals^d$.  
We instantiate $\est(\pi)$ as the Monte Carlo algorithm $\veceval$
(\Cref{alg:veceval}), which simply samples trajectories according to $\pi$ and returns the sample average of $\phihat\ind{h}(\x_h, \a_h) \cdot \mathbb{I}\{\theta^\top\phihat\ind{h}(\x_h, \a_h)\geq 0\}$.
To
implement $\apx(\theta)$, we appeal to \psdp
(\cref{alg:PSDP}). \psdp, given an arbitrary reward function $r_{1:h}:\cX
\times \cA \rightarrow \reals$ and a function class $\cG\subseteq \{g:
\cX\times \cA\rightarrow \reals\}$ capable of realizing all possible value
functions induced by these rewards, can use the policy covers
$\Psi\ind{1:h}$ to efficiently compute a policy $\pihat = \psdp(h,r_{1:h}, \cG,
\unif(\Psi\ind{1}), \dots, \unif(\Psi\ind{h}), n)$ that approximately solves \iftoggle{neurips}{$\argmax_{\pi \in \Pim} \E^{\pi}[\sum_{t=1}^{h} r_t(\x_t,\a_t)]$,}{
\begin{align}
	\argmax_{\pi \in \Pim} \E^{\pi}\left[\sum_{t=1}^{h} r_t(\x_t,\a_t)\right],
\end{align}
}
and does so using polynomially many episodes; see \cref{sec:nonnegative} for
details and formal guarantees.\footnote{This is the main
place where the analysis uses the inductive hypothesis
that $\Psi\ind{1:h}$ are policy covers.} Thus, implementing $\apx(\theta)$ is as
simple as invoking \psdp with the rewards 
\begin{align} 
r_{t}(x,a;\theta)\ldef{} \left\{\begin{array}{ll} \phihat\ind{h}(x,a)^\top\theta \cdot \mathbb{I}\{\phihat\ind{h}(x,a)^\top\theta\geq 0\}, &  \text{for }
	t=h, \\ 0, &  \text{otherwise}.
\end{array}\right. 
\end{align}
With this, we have all the
ingredients needed for spanner computation, and the algorithm is complete.

\subsection{Main Guarantee for \mainalg}
\label{sec:main_theorem}
The following result is the main sample complexity guarantee for \mainalg{} (\cref{alg:spanRL}).
\begin{theorem}[Main theorem for \mainalg]
\label{thm:spanrlmain}
Let $\delta, \veps \in(0,1)$ be given. Given a sufficiently large $\cfrak = \polylog(A,H,d,\ln
(|\Phi|/\delta))$, the policies $\Psi\ind{1:H}$
produced by $\mainalg(\Phi, \veps, \cfrak, \delta)$ are a
$(\frac{1}{8 Ad},\veps)$-policy cover with probability at least
$1-\delta$. 
The total number of episodes used by \mainalg is at most:
\begin{align}
	\label{eq:spanrl_main}
	\bigoht\left(    { A^2 d^{13} H^6 (d + \ln (|\Phi|/\delta))} \cdot 1/\veps^2\right).
\end{align}
\end{theorem}
The sample complexity bound in \cref{thm:spanrlmain} scales with $1/\veps^2$, which is optimal for reward-based RL in general \citep{jin2018q}.
\cref{thm:spanrlmain} is the first provable, model-free sample complexity
guarantee for general Low-Rank MDPs that is attained by an
efficient algorithm. Prior to our work, all efficient model-free algorithms required non-negative features (latent
variable structure) or
reachability \citep{modi2021model,zhang2022efficient}.\footnote{In the
  first version of this work, we also required reachability, albeit a
  weaker version compared to that used in
  \citep{modi2021model,zhang2022efficient}. The current version removes the need for reachability by targeting a slightly different set for barycentric spanner computation; see \cref{foot:note}.}
	While our guarantee is polynomial in
	all relevant problem parameters (and optimal in $1/\veps$), improving the dependence further
	(e.g., to match that of the best known inefficient algorithms) is
	an interesting direction for future research, as is removing the reachability assumption.

	\paragraph{Application to reward-based RL}
	By using the policy cover produced by \mainalg within \psdp (\cref{alg:PSDP}),
	we can optimize any downstream reward function to error $\veps$ using
	$\poly(d,A,H,\log\abs{\Phi})\cdot 1/\veps^2$ episodes. See
	\cref{sec:reward_based} for details.

	\paragraph{Efficiency and practicality}  We observe that $\mainalg$ is
	simple and practical. Defining $\cL_{\cD}(\phi, w, f) \coloneqq  \sum_{(x, a,
x')\in\cD} (\phi(x,a)^\top w  - f(x'))^2  + \lambda \|w\|^2$, where
$\cD$ is a dataset consisting of $(\bx_h,\ba_h,\br_h,\bx_{h+1})$
tuples, the algorithm is provably efficient whenever the adversarial objective
\begin{align}
\label{eq:adversarial}
f\ind{t} \in \argmax_{f\in \cF} \max_{\tilde \phi\in \Phi} \left\{ \min_{w} \cL_{\cD}(\phi\ind{t}, w, f)  - \min_{\tilde w} \cL_{\cD}(\tilde \phi, \tilde w, f)  \right\},
\end{align}
in \cref{line:replearn} of \replearn{} (\cref{alg:newreplearn}),
can be implemented efficiently (note that by the definition of
$\cL_{\cD}$, the ``inner'' minima over $w$ and
$\tilde w$ in \eqref{eq:adversarial} can be solved in closed
form). This objective was also assumed to be efficiently solvable in
\citep{modi2021model,zhang2022efficient} and was empirically shown to
be practical in \citep{zhang2022efficient}; note that the objective is
amenable to standard gradient-based optimization techniques, and that
$\cF$ can be over-parameterized. While a detailed
experimental evaluation is outside of the scope of this paper, we are
optimistic about the empirical performance of the algorithm in light
of the encouraging results based on the same objective in
\citet{zhang2022efficient}

Outside of representation learning, the only overhead in \mainalg is
the $\spanner$ subroutine, which has polynomial runtime. Indeed,
$\spanner$ requires only polynomially many calls to the
linear optimization oracle, instantiated as $\psdp$, which is 
efficient whenever standard least-squares regression problems based on
the class $\Phi$ can be solved efficiently, analogous to \cite{misra2020kinematic,mhammedi2023representation}.

\paragraph{Analysis and proof techniques}

The proof of \cref{thm:spanrlmain}, which is given in
\cref{sec:analysis}, is appealing in its simplicity and
modularity. The crux of the proof is to show that the
representation learning guarantee in \eqref{eq:replearn} is strong
enough to ensure that the downstream spanner computation in \spanner{}
succeeds. It is straightforward to show that spanner
computation would succeed if we had access to an estimated
representation that $\phihat\ind{h}$ that approximates $\phistarh$
point-wise (i.e., uniformly for all $(x,a)$ pairs), but the key challenge is that the guarantee in
\eqref{eq:replearn} only holds \emph{on average} under the roll-in
distribution $\unif(\Psi\ind{h})$. Prior works that make use of the same representation
learning objective ($\briee$ \citep{zhang2022efficient} and $\moffle$
\citep{modi2021model}) do not make use of spanners; instead, they
appeal to exploration strategies based on elliptic bonuses, addressing the issue of approximation
errors through additional assumptions (non-negativity of the factorization for $\moffle$, and Block MDP structure for $\briee$). 
\arxiv{For example, $\briee$ uses $\phihat$ as a plug-in for $\phistar$ within an optimistic
algorithm tailored to linear MDPs, and uses Block MDP
structure to facilitate a change-of-measure argument
that lifts the average-case approximation guarantee
for $\phihat$ to a uniform guarantee. However, this
type of
argument does not readily extend to general low-rank
MDPs. As such, } perhaps the most important observation in our proof is that
barycentric spanners are robust to the average-case approximation
error guarantee in \eqref{eq:replearn} as-is, without
additional structural assumptions. Intuitively, this
benefit seems to arise from the fact that the
spanner property only concerns the \emph{first
moment} of the feature map $\phistar$, while
algorithms based on elliptic bonuses require
approximation guarantees for the \emph{second
moment}; understanding this issue more
deeply is an interesting question for future work.

		\section{Discussion}
		\label{sec:discussion}

Our work shows for the first time how to achieve efficient, model-free
exploration in general Low-Rank MDPs. On the technical side, our
results leave open a number of interesting technical questions,
including (1) regret (as opposed to PAC) guarantees, and
(2) matching the minimax rate achieved by
inefficient algorithms using an efficient
algorithm.

More broadly, our work highlights the power of non-optimistic
algorithms that explore by building policy covers. In light of this, perhaps the most interesting question
is how to extend our techniques to more general function approximation
settings beyond the Low-Rank MDP model; this will likely entail
replacing the notion of barycentric spanner with a more general form of
exploration basis.

	\newpage
\subsection*{Acknowledgements}
We thank Noah Golowich, Dhruv Rohatgi, and Ayush Sekhari for
several helpful discussions. We also thank Yassir Jedra for a suggestion that lead to an improved final sample complexity. ZM and AR acknowledge support from the ONR through awards N00014-20-1-2336 and N00014-20-1-2394, and ARO through award W911NF-21-1-0328. AB acknowledges support from the National Science Foundation Graduate Research Fellowship under Grant No.~1122374.

	\bibliography{refs.bib}
	
	\newpage

	\appendix

		\section{Additional Related Work}
		\label{sec:additional_related}

In this section, we discuss relevant related work not already covered.

\paragraph{Block MDPs}
A particularly well-studied special case low-rank MDPs is the \emph{Block MDP (BMDP) model} \citet{du2019provably,misra2019kinematic,zhang2022efficient,mhammedi2023representation}. For this setting, \citet{du2019provably,misra2019kinematic} provide algorithms that conduct exploration in a provably oracle-efficient manner under a reachability assumption. This reachability assumption was removed by subsequent work of \citet{zhang2022efficient} (with a suboptimal rate) and \citet{mhammedi2023representation} (with optimal error dependence). These works are tailored to the BMDP model, and it is unclear whether it is possible to extend them to general low-rank MDPs.

\paragraph{Barycentric spanners}
\citet{huang2023reinforcement} consider a variant of the \LRMDP
framework in which we are given a class $\Upsilon$ that realizes the
next-state feature map $\mustar$, but do not have access to a class
$\Phi$ for the feature map $\phistar$, which is unknown. Their
algorithm, like \mainalg, is based on barycentric spanners, though the algorithm
design considerations and analysis are significantly
different. Notably, their algorithm is not computationally efficient,
and their analysis takes advantage of the fact that realizability of
$\mustar$ facilitates estimation of the occupancies $\crl{d^{\pi}(\cdot)}_{\pi\in\Pim}$ in $\ell_1$-error.  Barycentric spanners were also in the work of \citet{golowich2022learning} for reinforcement learning in Partially Observable MDPs (POMDPs).  Their analysis is substantially different from ours, and their algorithm appeals to the barycentric spanner computation approach in \citet{awerbuch2008online} in an off-the-shelf fashion.

	\clearpage

	\section{Organization of the Appendix}
This appendix is organized as follows:
\begin{itemize}
\item In \cref{sec:reach0}, we introduce key analysis tools that we use in the proof of our main result (\cref{thm:spanrlmain}). 
	\item \cref{sec:nonnegative}, \cref{sec:spanner}, and \cref{sec:replearn} contain results we rely on in the proof of \cref{thm:spanrlmain}. In particular,  \cref{sec:nonnegative}, \cref{sec:spanner}, and \cref{sec:replearn} provide generic guarantees for the subroutines \psdp{} (\cref{alg:PSDP}), \spanner{} (\cref{alg:spanner}), and \replearn{} (\cref{alg:newreplearn}) of \mainalg{} (\cref{alg:spanRL}), respectively. 
		\item \cref{sec:analysis} contains the analysis of \cref{thm:spanrlmain}.
	\item In \cref{sec:reward_based}, we show how an approximate policy cover can be used to optimize downstream reward functions.
	\item \cref{sec:helper} contains a set of helper
	results used throughout the analysis.
\end{itemize}

\section{Analysis Tools: Extended Low-Rank MDP and Truncated Policies}
\label{sec:reach0}

In this section, we present two tools, the \emph{extended MDP} and a \emph{truncated policy class}, that will be used throughout the analysis of \mainalg, and facilitate an analysis that does not require reachability assumptions. The definitions we give generalize analogous definitions given in \cite{mhammedi2023representation} for the special case of Block MDPs, though the generalization to the low-rank MDP setting is non-trivial.

\subsection{Extended MDP} 
As in \cite{mhammedi2023representation}, we define the extended MDP $\cMbar$ to be the result of augmenting the true MDP $\cM$ by adding a set of $H$ terminal states $\tfrak_{1:H}$, and a terminal action $\afrak$ with the property that taking $\afrak$ from any state at layer $h\in [H-1]$ leads to $\tfrak_{h+1}$ deterministically, and any action in $\cA\cup \{\afrak\}$ at latent state $\tfrak_h$ transitions to $\tfrak_{h+1}$ deterministically. To express $\cMbar$ as a low-rank MDP, we increase the feature dimension by $1$. First, for any $\phi \in \Phi$, we define the extension
\begin{align}
	\bar \phi(x,a) & = \left\{ \begin{array}{ll}
		[\phi(x,a)^\top, 0]^\top\in \reals^{d+1}, & \forall a\in \cA, \forall x\in \cX,\\
		e_{d+1}\in \reals^{d+1}, & a = \afrak,  \forall x\in \cX, \\
		e_{d+1}\in \reals^{d+1}, & \forall a\in \cA, x \in  \{\tfrak_1,\dots, \tfrak_H\}, 
	\end{array}  \right. \label{eq:extension} \shortintertext{with $\bar{\phi}^{\star}$ denoting the extension of $\phi^{\star}$. We similarly define} \muhb[h](x) &  = \left\{ \begin{array}{ll}
		[\muh[h](x)^\top, 0]^\top\in \reals^{d+1}, & \forall x\in \cX,\\
		e_{d+1}\in \reals^{d+1}, &  x=\tfrak_h,
	\end{array}  \right.  \label{eq:newphi0}
\end{align}
for $h\in[H]$. With these definitions, we formally define $\cMbar=(\cX\cup \{\tfrak_1,\cdots, \tfrak_H\}, \cA\cup\{\afrak\}, \rho, (\muhb[h])_{h\in[H]}, (\bar{\phi}_h^\star)_{h\in[H]})$ as the extended MDP, which one can verify is indeed a low-rank MDP in $d+1$ dimensions.

We let $\Pibarm$ be the set of all randomized Markov policies in $\cMbar$, with the convention that $\pi(\term_h)=\afrak$ for all $\pi\in\Pibarm$ and $h\in [H]$. For any policy $\pi \colon \cX \rightarrow \cA$, we extend it to $\cXbar \coloneqq \cX \cup \{\tfrak_1, \dots, \tfrak_H\}$ by taking $\pi(\tfrak_h)=\afrak$ for all $h\in[H]$. Moving forward, for any $h\in[H]$, we let $\cXbar_h \coloneqq \cX_h \cup \{\tfrak_h\}$, and define $\cAbar=\cA\cup\crl{\afrak}$.

We denote expectations and probability laws for trajectories in $\cMbar$ by $\Ebar$ and $\Pbar$, respectively, and for any $\cX'\subseteq \cXbar_h$, we let $\Pbar_h^{\pi}[\cX']\coloneqq \Pbar^\pi[\x_h \in \cX']$ denote the induced law of $\bx_h$ under a policy $\pi$ in $\cMbar$. Furthermore, for any $x\in\cX_h$, we define the \emph{occupancy} measure $\dbar^\pi(x) \coloneqq \frac{\dd \Pbar_h^\pi}{\dd \bar\nu}(x)$ as the density of $\Pbar^{\pi}_h$ with respect to $\bar\nu= \nu +\sum_{h\in[H]}\mathbb{I}_{\tfrak_h}$. 

We define $\widebar\Phi$ be the set of all extended feature maps $\phibar$ (as in \eqref{eq:newphi0}) for $\phi \in \Phi$. In some proofs, it will be convenient to work with the restriction of the extended feature maps to their first $d$ coordinates; for any $\phi \in \Phi$, we define 
\begin{align}
	\tilde \phi(\cdot,\cdot) \coloneqq 	(\bar \phi(\cdot,\cdot)[1], \dots, \bar \phi(\cdot,\cdot)[d])^\top.
\end{align}

Finally, we the extend the notion of a randomized policy cover (\cref{def:randcover}) to the extended MDP as follows.
\begin{definition}[Relative policy cover]
	\label{def:approxcover}
	For $\alpha\in(0,1], \eta \geq 0$, a distribution $P\in \Delta(\Pim)$ is a $(\alpha, \eta)$-randomized policy cover relative to $\Pi\subseteq \Pibarm$ for layer $h$ in $\Mbar$ if 
	\begin{align}
		\E_{\pi \sim P} [\dbar^{\pi}(x)] \geq \alpha \cdot \max_{\pi'\in \Pi} \dbar^{\pi'}(x),\quad \text{for all $x\in \cX_{h}$ such that} \quad  \max_{\pi'\in \Pi} \dbar^{\pi'}(x)\geq \eta \cdot \|\muhb[h](x)\|.
	\end{align}
\end{definition}

\subsection{Truncated Policy Class}
Next, we introduce the notion of the \emph{truncated policy class}, generalizing \citet{mhammedi2023representation}. We begin with some preliminary definitions.

For any $h \in [H]$, given a collection of policies $\Pi'\subseteq\Pibarm$, we let \begin{align}\cP_{h}(\Pi') \coloneqq  \left\{\tilde \phi^{\star,\pi}_{h} \mid \pi \in \Pi'\right\}, \quad \text{where}\quad \tilde \phi^{\star,\pi}_{h} \coloneqq \Ebar^{\pi}[\tilde {\phi}^{\star}_h(\x_h, \a_h)].
	\label{eq:polytope0}
\end{align}
Using this, we define the notion of \emph{$\eta$-reachable states relative to $\Pi'$}.%
\begin{definition}[$\eta$-reachable states]
	\label{def:reachable}
	For $h\in[H]$ and a policy class $\Pi'\subseteq \Pibarm$, we define the set of $\eta$-\emph{reachable} states at layer $h$ relative to the set $\Pi'$ as:
	\begin{align}
		\cX_{h, \eta}(\Pi') \coloneqq \left\{x\in \cX_h \mid  \exists u \in \cP_{h-1}(\Pi') : \muh[h](x)^\top u \geq \|\muh[h](x)\| \cdot\eta   \right\}. \label{eq:theolddef}
	\end{align}
\end{definition}
Given a parameter $\eta>0$, we now define the truncated policy class $\Pibar_{\eta}$ inductively as follows: Let $\Pibar_{0,\eta}\coloneqq \Pibarm$, and for each $h\geq 1$, let $\Pibar_{h, \eta}$ be the set of policies defined by 
\begin{align}
\pi \in \Pibar_{h,\eta}  \iff    \exists \pi'\in \Pibar_{h-1,\eta} : \forall t \in[H], \forall x \in \cXbar_t, \ \ \pi(x) = \left\{  \begin{array}{ll} \pi'(x), & \text{if } t=h \text{ and }  x \in \cX_{h,\eta}\big(\Pibar_{h-1,\eta}\big),\\  \afrak, & \text{otherwise}. \end{array} \right. \label{eq:equiv0}
\end{align}
Finally, we define $\Pibar_\eta\coloneqq \Pibar_{H,\eta}$.

As in \citet{mhammedi2023representation}, the utility behind the extended MDP and truncated policy class is as follows:
\begin{enumerate}
\item 
While the extended BMDP $\wb{\cM}$ does not necessarily have the property that every state is $\eta$-reachable---a property that facilitates handling distribution shifts and avoiding error exponentiation \citep{agarwal2022,zanette2020provably}---it emulates certain properties of reachable MDPs, but only if we compare performance to policies in $\Pibar_\eta$. 
\item For all reward functions of interest, the best reward that can be achieved by a policy in $\Pibar_\eta$ is close to what can be achieved using arbitrary policies in $\Pibarm$.
\end{enumerate}

\subsection{Structural Results for Extended Low-Rank MDP}
\label{sec:structural}
We now present some structural results involving the extented MDP and truncated policy class defined in \cref{sec:reach0}. First, we recall the definition of the truncated policy class. Given a parameter $\eta>0$, let $\Pibar_{0,\eta}\coloneqq \Pibarm$, and for each $h\geq 1$, let $\Pibar_{h, \eta}$ be the set of policies defined by 
\begin{align}
	\pi \in \Pibar_{h,\eta}  \iff    \exists \pi'\in \Pibar_{h-1,\eta} : \forall t \in[H], \forall x \in \cXbar_t, \ \ \pi(x) = \left\{  \begin{array}{ll} \pi'(x), & \text{if } t=h \text{ and }  x \in \cX_{h,\eta}\big(\Pibar_{h-1,\eta}\big),\\  \afrak, & \text{otherwise}, \end{array} \right. \label{eq:equiv0+}
\end{align}
where for a set of policies $\Pi'\subseteq \Pibarm$, we let
\begin{align}
		\cX_{h, \eta}(\Pi') \coloneqq \left\{x\in \cX_h \  \left|  \    \max_{\pi \in \Pi'} \dbar^{\pi}(x) \geq \|\bar{\mu}_h^\star(x)\| \cdot\eta  \right.  \right\}. \label{eq:thexdef}
	\end{align}
Note that this matches the definition in \eqref{eq:theolddef} because $[\bar{\mu}^\star_h(x)]_{d+1}=0$, for all $x\neq \tfrak_h$. Finally, we let $\Pibar_\eta\coloneqq \Pibar_{H,\eta}$.

The next lemma bounds the probability of the set of states that are not reachable with sufficiently high probability.
\begin{lemma}[Probability of non-reachable states]
  \label{lem:reachable}
  Under the normalization assumption \eqref{eq:normalization}, we have that for any $t\in[H]$, 
	\begin{align}
		\sup_{\pi \in \Pibar_{\eta}} \Pbar^{\pi}[\x_t \in \cX_t \setminus \cX_{t,\eta}(\Pibar_\eta)] \leq \eta \cdot d^{3/2}.
	\end{align}
\end{lemma}
\begin{proof}[\pfref{lem:reachable}]
	Fix $t\in [H]$. By definition of $\cX_{t,\eta}(\Pibar_\eta)$, we have that 
	\begin{align}
		\forall x\in \cX_t \setminus \cX_{t,\eta}(\Pibar_\eta),\quad 	\sup_{\pi \in \Pibar_{\eta}} \dbar^{\pi}(x)  \leq \eta \cdot \|\mubar^\star_t(x)\|. \label{eq:bounde}
	\end{align}
	Thus, integrating over $x\in \cX_t \setminus \cX_{t,\eta}(\Pibar_\eta)$, we obtain 
	\begin{align}
		\sup_{\pi \in \Pibar_{\eta}} \Pbar^{\pi}[\x_t \in \cX_t \setminus \cX_{t,\eta}(\Pibar_\eta)]  & = \sup_{\pi \in \Pibar_{\eta}} \int_{\cX_t \setminus \cX_{t,\eta}(\Pibar_\eta)}  \dbar^{\pi}(x) \dd \nubar(x), \nn \\
		& = \eta \cdot \int_{\cX_t \setminus \cX_{t,\eta}(\Pibar_\eta)} \|\bar{\mu}^\star_t(x)\| \dd \nubar(x), \quad \text{(by \eqref{eq:bounde})} \\
		& \leq \eta \cdot \int_{\cX_{t}}   \|\bar{\mu}^\star_t(x)\| \dd \bar{\nu}(x),\nn \\ 
		& = \eta \cdot \int_{\cX_{t}}   \|{\mu}^\star_t(x)\| \dd {\nu}(x),\quad \text{(since $[\mubar_t(x)]_{d+1}=0, \forall x \neq \tfrak_t$)} \\ 
		& \leq \eta d^{3/2},
	\end{align}	
	where the last inequality follows by \cref{lem:normalization}; this is a consequence of the normalization assumption \eqref{eq:normalization}.
\end{proof}

The next lemma generalizes \citet[Lemma A.1]{mhammedi2023representation} to \lrmdp{}s.
\begin{lemma}
	\label{lem:pih_max}
	For all $h \in[H]$, $x\in \cX_h$, and $\ell\in[h\ldotst H]$, we have $\max_{\pi \in \Pibar_{\ell-1,\eta}}\dbar(x)= 	\max_{\pi \in \Pibar_{\ell,\eta}}\dbar(x)$. Further,
	\begin{align}
		\forall x\in \cX_h, \quad  \max_{\pi\in \Pibar_{h-1, \eta}} \dbar^{\pi}(x) =  \max_{\pi\in \Pibar_{\eta}} \dbar^{\pi}(x) .\label{eq:key}
	\end{align}
\end{lemma}
\begin{proof}[\pfref{lem:pih_max}]
	We will show that for all $\ell\in[h\ldotst{}H]$,
	\begin{align}
		\forall x\in \cX_h, \quad \max_{\pi \in \Pibar_{\ell-1,\eta}}\dbar(x)= 	\max_{\pi \in \Pibar_{\ell,\eta}}\dbar(x). \label{eq:tele}
	\end{align}
	This implies \eqref{eq:key} by summing both sides of \eqref{eq:tele} over $\ell=h,\dots, H$, telescoping, and using that $\Pibar_{\eta}=\Pibar_{H, \eta}$. To prove the result, let $\ell\in[h\ldotst{}H]$, $x\in \cX_h$, and $\tilde\pi \in \argmax_{\pi'\in \Pibar_{\ell-1,\eta}} \dbar^{\pi'}(x)$. Further, let $\pi\in \Pibar_{\ell, \eta}$ be as in \eqref{eq:equiv0+} with $\pi'=\tilde \pi$. In this case, by \eqref{eq:equiv0+}, we have $\tilde \pi(x')=\pi(x')$, for all $x'\in\cX_\tau$, and $\tau \leq [\ell-1]$. Using this and the fact that $x\in \cX_h$ and $\ell\geq h$, we have 
	\begin{align}
		\max_{\breve \pi\in \Pibar_{\ell-1,\eta}} \dbar^{\breve \pi}(x) =\dbar^{\tilde \pi}(x)= \dbar^{\pi}(x) \leq \max_{\breve \pi \in \Pibar_{\ell, \eta}} \dbar^{\breve \pi}(x).\nn
	\end{align} 
	We now show the inequality in the other direction. Let $\ell\in[h\ldotst{}H]$, $x\in \cX_h$, and $\tilde \pi \in \argmax_{ \breve\pi\in \Pibar_{\ell,\eta}} \dbar^{\breve \pi}(x)$. Further, let $\pi'\in \Pibar_{\ell-1, \eta}$ be as in \eqref{eq:equiv0+} for $\pi = \tilde \pi$. In this case, by \eqref{eq:equiv0+}, we have $\tilde \pi(x)=\pi'(x)$, for all $\tau \in [\ell-1]$. Using this and the fact that $x\in \cX_h$ and $\ell\geq h$, we have 
	\begin{align}
		\max_{\breve \pi\in \Pibar_{\ell,\eta}} \dbar^{\breve \pi}(x) =\dbar^{\tilde \pi}(x)= \dbar^{\pi'}(x) \leq \max_{\breve \pi \in \Pibar_{\ell-1, \eta}} \dbar^{\breve \pi}(x).\nn
	\end{align} 
	This shows \eqref{eq:tele} and completes the proof.
\end{proof}

Using \cref{lem:pih_max} and the definition of $\cX_{h,\eta}(\cdot)$ in \eqref{eq:thexdef}, we obtain the following corollary.
\begin{corollary}
	\label{cor:thecor}
	For all $h\in[H]$, it holds that 
	\begin{align}
		\cX_{h,\eta}(\Pibar_{h-1,\eta}) = 	\cX_{h,\eta}(\Pibar_{\eta}).
	\end{align}
\end{corollary}

The next lemma quantifies the ``cost of truncation'' incurred by optimizing reward functions using policies in the truncated class $\Pibar_{\eta}$ instead of $\Pibarm$
\begin{lemma}[Cost of truncation]
	\label{lem:truncation}
	Let $\eta \in(0,1)$, and $B_{1:H}>0$, and consider reward functions $r_{1}: \cX_1\times \cA \rightarrow [-B_1,B_1],\dots,r_{H}: \cX_H\times \cA \rightarrow [-B_H,B_H]$. We have
	\begin{align}
		\sup_{\pi\in \Pibar_\eta} \Ebar^\pi \left[  \sum_{h=1}^H \bar{r}_h(\x_h,\a_h) \right] \geq 	\sup_{\pi\in \Pibarm} \Ebar^\pi \left[  \sum_{h=1}^H \bar{r}_h(\x_h,\a_h) \right]  - 2 H  d^{3/2} \eta \sum_{h=1}^H B_h,
		\end{align}
	where, for each $h\in[H]$, $\bar{r}_h(x,a)=r_h(x,a)$ for all $(x,a)\in \cX_h\times \cA$, and $\bar{r}_h(x,a)=0$ when $x=\tfrak_h$ or $a=\afrak$.
	\end{lemma}
	\begin{proof}[\pfref{lem:truncation}]
	Let $\bar r_{1:H}$ be the ``extended'' reward functions as in the lemma's statement. Let $h\in[H]$ and $\pi_{h-1}\in \argmax_{\pi \in \Pibar_{h-1,\eta}}\Ebar^{\pi}\left[\sum_{h=1}^H \bar r_h(\x_h,\a_h)\right]$. Further, define $\pi_h$ as $\pi\in \Pibar_{h,\eta}$ in \eqref{eq:equiv0+} with $\pi'=\pi_{h-1}$. Note that since for all $t\in[h-1]$ and $x\in \cX_t$, $\pi_{h}(x)=\pi_{h-1}(x)$ (by \eqref{eq:equiv0+}), we have  
	\begin{align}
		\Ebar^{\pi_{h-1}}\left[\sum_{t=1}^{h-1} \bar r_t(\x_t,\a_t)\right] = \Ebar^{\pi_{h}}\left[\sum_{t=1}^{h-1} \bar r_t(\x_t,\a_t)\right]. \label{eq:firsthalf}
	\end{align}
On the other hand, for $\cX_{h,\eta}\coloneqq \cX_{h,\eta}(\Pibar_{h-1,\eta})$ we have
	\begin{align}
		&	\Ebar^{\pi_{h-1}}\left[\sum_{t=h}^H \bar r_t(\x_t,\a_t)\right]\nn \\
		&  = 	\Ebar^{\pi_{h-1}}\left[\sum_{t=h}^H \bar r_t(\x_t,\a_t)\right],\nn \\
		&  = \Ebar^{\pi_{h-1}}\left[ \mathbb{I}\{\x_h \in \cX_{h,\eta} \} \cdot \sum_{t=h}^H \bar r_t(\x_t,\a_t)\right]+ \Ebar^{\pi_{h-1}}\left[ \mathbb{I}\{\x_h \not\in \cX_{h,\eta} \} \cdot \sum_{t=h}^H \bar r_t(\x_t,\a_t)\right] , \nn \\
		& = \Ebar^{\pi_{h}}\left[ \mathbb{I}\{\x_h \in \cX_{h,\eta} \}\cdot  \sum_{t=h}^H \bar r_t(\x_t,\a_t)\right] + \Ebar^{\pi_{h-1}}\left[ \mathbb{I}\{\x_h \not\in \cX_{h,\eta} \} \cdot \sum_{t=h}^H \bar r_t(\x_t,\a_t)\right] , \ \ (\text{by definition of $\cX_{h,\eta}$ and $\pi_h$}) \nn \\
		& = \Ebar^{\pi_{h}}\left[  \sum_{t=h}^H \bar r_t(\x_t,\a_t)\right]  - \Ebar^{\pi_{h}}\left[ \mathbb{I}\{\x_h \not\in \cX_{h,\eta} \} \cdot \sum_{t=h}^H \bar r_t(\x_t,\a_t)\right] + \Ebar^{\pi_{h-1}}\left[ \mathbb{I}\{\x_h \not\in \cX_{h,\eta} \} \cdot \sum_{t=h}^H \bar r_t(\x_t,\a_t)\right],  \nn \\
		& = \Ebar^{\pi_{h}}\left[  \sum_{t=h}^H \bar r_t(\x_t,\a_t)\right]  - \Ebar^{\pi_{h}}\left[ \mathbb{I}\{\x_h \in \cX_h\setminus  \cX_{h,\eta} \} \cdot \sum_{t=h}^H \bar r_t(\x_t,\a_t)\right] + \Ebar^{\pi_{h-1}}\left[ \mathbb{I}\{\x_h \in \cX_h\setminus  \cX_{h,\eta} \} \cdot \sum_{t=h}^H \bar r_t(\x_t,\a_t)\right], 
		\end{align}
	where the last equality follows by the fact that I) if $\x_h =\tfrak_h$, then $\x_t=\tfrak_t$ for all $t\in [h\ldotst H]$, and II) $\bar{r}_t(\tfrak,\cdot)\equiv 0$, for all $t\in [h \ldots H]$. Now, using the range assumption on the rewards, we get 
	\begin{align}
			\Ebar^{\pi_{h-1}}\left[\sum_{t=h}^H \bar r_t(\x_t,\a_t)\right] \leq   \Ebar^{\pi_{h}}\left[  \sum_{t=h}^H \bar r_t(\x_t,\a_t)\right] +\left(\Pbar^{\pi_h}[\x_h \in \cX_h \setminus \cX_{h,\eta}] + \Pbar^{\pi_{h-1}}[\x_h \in \cX_h \setminus \cX_{h,\eta}]\right) \sum_{t=h}^H B_t. \label{eq:ue}
		\end{align}
	On the other hand, by \cref{lem:pih_max} and the fact that $\pi_{h-1}\in \Pibar_{h-1,\eta}$ and $\pi_{h}\in \Pibar_{h,\eta}$, we have that 
	\begin{align}
			\Pbar^{\pi_{h-1}}[\x_h \in \cX_h \setminus \cX_{h,\eta}] \vee \Pbar^{\pi_h}[\x_h \in \cX_h \setminus \cX_{h,\eta}]\leq\sup_{\pi \in \Pibar_\eta} \Pbar^{\pi}[\x_h \in \cX_h \setminus \cX_{h,\eta}].  \label{eq:intere}
		\end{align}
	Furthermore, by \cref{cor:thecor}, we have $ \cX_{h,\eta} = \cX_{h,\eta}(\Pibar_\eta)$. Combining this with \eqref{eq:intere} and \cref{lem:reachable}, we get  
	\begin{align}
		\Pbar^{\pi_{h-1}}[\x_h \in \cX_h \setminus \cX_{h,\eta}] \vee \Pbar^{\pi_h}[\x_h \in \cX_h \setminus \cX_{h,\eta}]\leq\sup_{\pi \in \Pibar_\eta} \Pbar^{\pi}[\x_h \in \cX_h \setminus \cX_{h,\eta}(\Pibar_\eta)]  \leq \eta d^{3/2}. \label{eq:telesc2}
		\end{align}
Plugging this into \eqref{eq:ue} and using \eqref{eq:firsthalf} implies that 
\begin{align}
		\Ebar^{\pi_{h-1}}\left[\sum_{t=h}^H \bar r_t(\x_t,\a_t)\right] \leq  \Ebar^{\pi_{h}}\left[  \sum_{t=h}^H \bar r_t(\x_t,\a_t)\right]+  2  \eta d^{3/2} \sum_{h=1}^H B_h.
	\end{align}
	Summing both sides of \eqref{eq:telesc2} for $h=1,\dots, H$, telescoping, and using that $\Pibar_{0,\eta}= \Pibarm$ and $\Pibar_{H,\eta}= \Pibar_{\eta}$, we get 
	\begin{align}
		\max_{\pi \in \Pibarm}	\Ebar^{\pi}\left[\sum_{t=1}^{H} \bar r_t(\x_t,\a_t)\right]  \leq \max_{\pi \in \Pibar_{\eta}}	\Ebar^{\pi}\left[\sum_{t=1}^{H} \bar r_t(\x_t,\a_t)\right]  + 2H \eta d^{3/2} \sum_{h=1}^H B_h.
	\end{align}
\end{proof}
Using this, we now state and prove a result that will allow us to transfer any guarantees in the extended MDP and truncated policies $\Pibar_\eta$ back to the original MDP with the unrestricted policy class $\Pim$.
	\begin{lemma}
	\label{lem:transfer00}
	Let $h\in [H]$, $\alpha\in (0,1)$, and $\eta >0$ be given. 
	If $\unif(\Psi\ind{h})\in \Delta(\Pim)$ is an $(\alpha,\eta)$-randomized policy cover relative to $\Pibar_\eta$ for layer $h$ in $\Mbar$ (\cref{def:approxcover}), then $\unif(\Psi\ind{h})$ is an $(\alpha/2,\veps)$-randomized policy cover relative to $\Pim$ for layer $h$ in the true MDP $\cM$ (\cref{def:randcover}), where $\veps \coloneqq 4 H d^{3/2}\eta$.
\end{lemma}
\begin{proof}[Proof of \cref{lem:transfer00}]
	Fix $h\in[H]$, and let $y\in \cX_{h}$ be such that $\mu_h^\star(y)>0$. To prove \cref{lem:transfer00}, we will instantiate \cref{lem:truncation} with rewards $(r_t)$ given by
\begin{align}
	r_t(x,a) = \left\{  
	\begin{array}{ll}  
		\frac{\mu_{h}^\star(y)^\top}{\mu_{h}^\star(y)} \phi^\star_{h-1}(x,a), & \text{if } t=h \text{ and } (x,a)\in \cX_h\times \cA,  \\ 0, & \text{otherwise}. 
	 \end{array}  \right.
	\end{align}
We define the extended rewards $(\bar{r}_t)$ such that for all $t\in[H]$, $\bar{r}_t(x,a)=r_t(x,a)$ for all $(x,a)\in \cX_t\times \cA$, and $\bar{r}_t(x,a)=0$ when $x=\tfrak_t$ or $a=\afrak$. By applying \cref{lem:truncation} (with $B_h =1$ and $B_t=0$ for all $t\neq h$) and using that $|r_h(\cdot,\cdot)|\leq 1$ (since $\|\phi^\star_{h-1}(\cdot, \cdot)\|\leq 1$), we get 
\begin{align}
		\max_{\pi \in \Pibarm}	\Ebar^{\pi}\left[\sum_{t=1}^{H} \bar r_t(\x_t,\a_t)\right]  \leq \max_{\pi \in \Pibar_{\eta}}	\Ebar^{\pi}\left[\sum_{t=1}^{H} \bar r_t(\x_t,\a_t)\right]  + 2H \eta d^{3/2}. \label{eq:sens}
	\end{align}
On the other hand, the definition of $(r_t)$ implies that for any $\pi \in \Pibarm$, 
\begin{align}
	\Ebar^{\pi}\left[\sum_{t=1}^{H} \bar r_t(\x_t,\a_t)\right] =  	\frac{\mu_{h}^\star(y)^\top}{\mu_{h}^\star(y)} \tilde\phi^{\star,\pi}_{h-1}, \label{eq:un}
	\end{align}
where $\tilde\phi^{\star,\pi}_{h-1} \coloneqq \Ebar^{\pi}[\tilde\phi^\star_{h-1}(\x_{h-1},\a_{h-1})]$ and $\tilde\phi^\star_{h-1}$ is the restriction of $\phibar^\star_{h-1}$ to its first $d$ coordinates ($\phibar^\star_{h-1}$ is defined in \cref{sec:reach0}). Now, since $y\neq \tfrak_h$, we have $[\bar\mu_h^\star(y)]_{d+1}=0$, and so $\mu^\star_h(y)^\top \tilde\phi^{\star,\pi}_{h-1}= \mubar^\star_h(y)^\top \bar{\phi}^{\star, \pi}_{h-1}$. Thus, plugging this into \eqref{eq:un} and using \cref{lem:newnegative}, we get 
\begin{align}
	\forall \pi \in \Pibarm, \quad 	\Ebar^{\pi}\left[\sum_{t=1}^{H} \bar r_t(\x_t,\a_t)\right] =  	\frac{\mubar_{h}^\star(y)^\top}{\mu_{h}^\star(y)} \bar\phi^{\star,\pi}_{h-1}= \frac{\dbar^\pi(y)}{\|\mu^\star_h(y)\|}.
	\end{align} 
Plugging this into \eqref{eq:sens} and using that $\Pim\subseteq\Pibarm$, we have 
\begin{align}
\max_{\pi \in \Pim}	\frac{d^\pi(y)}{\|\mu^\star_h(y)\|} =\max_{\pi \in \Pim}	\frac{\dbar^\pi(y)}{\|\mu^\star_h(y)\|} \leq \max_{\pi \in \Pibarm}\frac{\dbar^\pi(y)}{\|\mu^\star_h(y)\|}\leq \max_{\pi \in \Pibar_{\eta}}\frac{\dbar^\pi(y)}{\|\mu^\star_h(y)\|} + 2H\eta d^{3/2}.\label{eq:thiseee}
	\end{align}
Now, suppose that $y$ is such that $\max_{\pi \in \Pim}	\frac{d^\pi(y)}{\|\mu^\star_h(y)\|}\geq 4 H \eta d^{3/2}$. By \eqref{eq:thiseee}, this implies that 
\begin{align}
	\max_{\pi \in \Pibar_{\eta}}\frac{\dbar^\pi(y)}{\|\mu^\star_h(y)\|} \geq 2H \eta d^{3/2} \geq \eta,
	\end{align}
and so since $P\coloneqq \unif(\Psi\ind{h})$ is a $(\alpha,\eta)$-randomized policy cover relative to $\Pibar_\eta$ for layer $t$ in $\cMbar$, we have that 
\begin{align}
	\max_{\pi \in \Pibar_{\eta}}\frac{\dbar^\pi(y)}{\|\mu^\star_h(y)\|} \leq \alpha^{-1} \E_{\pi\sim P} \E^{\pi} \left[\frac{\bar{d}^\pi(y)}{\|\mu^\star_h(y)\|}\right].
\end{align}
Combining this with \eqref{eq:thiseee} implies that 
\begin{align}
\max_{\pi \in \Pim}	\frac{d^\pi(y)}{\|\mu^\star_h(y)\|} & \leq \alpha^{-1}	\E_{\pi\sim P} \E^{\pi} \left[\frac{\bar{d}^\pi(y)}{\|\mu^\star_h(y)\|}\right] + 2H\eta d^{3/2},\nn \\
&\leq \alpha^{-1}	\E_{\pi\sim P} \E^{\pi} \left[\frac{\bar{d}^\pi(y)}{\|\mu^\star_h(y)\|}\right] +\frac{1}{2} \max_{\pi \in \Pim}	\frac{d^\pi(y)}{\|\mu^\star_h(y)\|},
	\end{align}
where the last inequality follows by the fact that $y$ is such that $\max_{\pi \in \Pim}	\frac{d^\pi(y)}{\|\mu^\star_h(y)\|}\geq 4 H \eta d^{3/2}$.  Rearranging the previous display and using that $\dbar^{\pi}(\cdot)\equiv d^{\pi}(\cdot)$ for all policies $\pi$ that never take the terminal action, we get:
\begin{align}
	\frac{\alpha}{2} \max_{\pi \in \Pim}	\frac{d^\pi(y)}{\|\mu^\star_h(y)\|} \leq 	\E_{\pi\sim P} \E^{\pi} \left[\frac{{d}^\pi(y)}{\|\mu^\star_h(y)\|}\right].
	\end{align}
This shows that $P= \unif(\Psi\ind{h})$ is a $(\frac{\alpha}{2}, 4 H\eta d^{3/2})$-randomized policy cover.
\end{proof}

\section{Generic Guarantees for \psdp{}}
\label{sec:nonnegative}
In this section, we present self-contained guarantees for \psdp{} (\cref{alg:PSDP}). We show that given any reward functions $r_{1:h}:\cX \times \cA \rightarrow \reals_{\geq 0}$ and function classes $\cG_{1:h}$, where $\cG_t\subseteq \{g: \cX_t\times \cA\rightarrow \reals\}$ for $t\in[h]$, that ``realize'' these reward functions (we formalize this in the next definition), that if $P\ind{1:h}$ are (approximate) policy covers for layers 1 through $h$, then for sufficiently large $n\geq 1$ and with high probability, the output $\pihat = \psdp(h,r_{1:h}, \cG_{1:h}, P\ind{1:h}, n)$ is an approximate maximizer of the objective 
\begin{align}
	\max_{\pi \in \Pim} \E^{\pi}\left[\sum_{t=1}^{h} r_t(\x_t,\a_t)\right].
\end{align}	
To formalize this result, we define the notion of realizability we require for the function classes $\cG_{1:h}$.
\begin{definition}
	\label{def:funcrealize}
	We say that function classes $\cG_{1:h}$, where $\cG_t\subseteq \{g: \cX_t\times \cA\rightarrow \reals\}$ for $t\in[h]$, realize  reward functions $r_{1:h}:\cX\times \cA \rightarrow \reals$ if for all $t\in[h]$ and all $\pi\in \Pim^{t+1:h}$,
	\begin{align}
		Q_t^{\pi}\in \cG_t, \quad \text{where}\qquad 	Q^{\pi}_t(x,a)\coloneqq r_t(x,a)+\E^{\pi}\left[\left.\sum_{\ell=t+1}^{h} r_\ell(\x_\ell,\a_\ell)\ \right| \ \x_t=x,\a_t=a\right]. \label{eq:real}
	\end{align}
\end{definition}
Note that $Q^\pi_t$ in \eqref{eq:real} represents the \emph{state-action value function} ($Q$-function) at layer $t\in[h]$ with respect to the rewards $r_{1:h}$ and partial policy $\pi$.

\begin{algorithm}[tp]
	\caption{$\texttt{PSDP}(h, r_{1:h},\cG_{1:h} ,P\ind{1:h},n)$: Policy Search by Dynamic Programming (cf. \citet{bagnell2003policy})}
	\label{alg:PSDP}
	\begin{algorithmic}[1]\onehalfspacing
		\Require~
		\begin{itemize}
			\item Target layer $h\in[H]$.
			\item Reward functions $r_{1:h}$.
			\item Function classes $\cG_{1:h}$.
			\item Policy covers $P\ind{1},\dots,P\ind{h}$. 
			\item Number of samples $n\in \mathbb{N}$.
		\end{itemize}
		\For{$t=h, \dots, 1$} 
		\State $\cD\ind{t} \gets\emptyset$. 
		\For{$n$ times}
		\State Sample $\pi \sim P\ind{t}$.
		\State Sample $(\x_t, \a_t, \sum_{\ell=t}^h r_{\ell}(\x_\ell,\a_\ell))\sim
		\pi\circ_{t} \pi_{\unif} \circ_{t+1} \pihat \ind{t+1}$.
		\State Update dataset: $\cD\ind{t} \gets \cD\ind{t} \cup \left\{\left(\x_t, \a_t,  \sum_{\ell=t}^h {r}_{\ell}(\x_\ell,\a_\ell)\right)\right\}$.
		\EndFor
		\State Solve regression: 
		\[\ghat\ind{t}\gets\argmin_{g\in \cG_t}  \sum_{(x, a, R)\in\cD\ind{t}} (g(x,a)  - R)^2.\] \label{eq:mistake}
		\State Define $\pihat\ind{t}\in\Pim^{t:h}$ via
		\begin{align}
			\label{eq:pol}
			\pihat \ind{t}(x) = \left\{
			\begin{array}{ll}
				\argmax_{a\in \cA}  \ghat\ind{t}(x,a),&\quad x\in\cX_t,\\
				\pihat\ind{t+1}(x),& \quad x\in \cX_{\ell},\;\;  \ell \in [t+1 \ldotst h].
			\end{array}
			\right.
		\end{align}
		\EndFor
		\State \textbf{Return:} Near-optimal policy $\pihat \ind{1}\in \Pim$. 
	\end{algorithmic}
\end{algorithm}
 
In what follows, given a function class $\cG\subseteq \{g: \cX\times \cA\rightarrow \reals \}$, we use $\cN_{\cG}(\veps)$ to denote the $\veps$-covering number of $\cG$ in $\ell_\infty$ distance. 
\begin{definition}
\label{def:covering}
A set of functions $\{g_1, \dots, g_N\} \subset \{g: \cX \times \cA\rightarrow \reals\}$ is an $\veps$-cover of $\cG\subseteq \{g:\cX \times \cA\rightarrow \reals\}$ in $\ell_\infty$-distance if for all $g\in \cG$, there exists $i \in [N]$ such that 
\begin{align}
\|g - g_i\|_{\infty}\leq \veps.
\end{align}	
The $\veps$-covering number $\cN_{\cG}(\veps)$ is the size $N$ of the smallest $\veps$-cover of $\cG$.
\end{definition}

\subsection{Intermediate Results for \psdp{}}
\label{sec:intermediate}
To prove our main guarantees for \psdp{} (stated in the next subsection), we first two intermediate lemmas. The first shows that for any poly $\pi$, the corresponding $Q$-function is the Bayes-optimal predictor for the regression problem solved in \psdp{} when $\pi$ is executed.
\begin{lemma}
	\label{lem:bayes00}
        Let reward functions $r_{1:h}:\cX\times \cA \rightarrow \reals$, $P\in\Delta(\Pim)$, and $\pihat\in \Pim^{t+1:h}$ be given. Fix $t\in\brk{h}$, and let $g^{P,\pihat}_{\bayes}$ denote the Bayes-optimal predictor\footnote{Observe that because this loss is strongly convex with respect to the prediction, the Bayes-optimal predictor is unique up to sets of measure zero.} for the sum of rewards under a policy $\pi$ sampled from $P$ and composed with $\pihat$ via $\pi\circ_t\piunif\circ_{t+1}\pihat$; that is,
	\begin{align}
		g^{P,\pihat}_{\bayes} \in \argmin_{ g : \cX_t \times \cA \rightarrow \reals} \E_{\pi \sim P} \E^{\pi\circ_t \pi_{\unif}\circ_{t+1} \pihat} \left[\left( g(\x_t, \a_t) - \sum_{\ell=t}^{h} r_\ell(\x_\ell,\a_\ell) \right)^2\right]. \label{eq:normal00}
	\end{align}
	Then, $g^{P,\pihat}_{\bayes}(\cdot,\cdot)\equiv Q^{\pihat}_t(\cdot,\cdot)$, where $Q^{\pihat}_t$ is the $Q$-function defined in \eqref{eq:real} for the partial policy $\pihat\in \Pim^{t+1,h}$ and rewards $r_{1:h}$. 
\end{lemma}%
\begin{proof}[\pfref{lem:bayes00}]
	The least-squares solution $g^{P,\pihat}_{\bayes}$ of the problem in \eqref{eq:normal00} satisfies, for all $a\in\cA$ and $x\in \cX_t$,
	\begin{align}
		g^{P,\pihat}_{\bayes} (x,a)& = \E_{\pi\sim P} \E^{\pi\circ_t \pi_{\unif}\circ_{t+1} \pihat} \left[ \left. \sum_{\ell=t}^{h} r_\ell(\x_\ell,\a_\ell)\ \right|\ \x_t =x ,\a_t =a \right], \nn \\
		& = \E[ r_t(\x_t,\a_t)\mid \x_t = x,\a_t = a]+ \E_{\pi\sim P}\E^{\pi\circ_t \pi_{\unif}\circ_{t+1} \pihat} \left[ \left. \sum_{\ell=t+1}^{h} r_\ell(\x_\ell,\a_\ell)\  \right| \  \x_t = x, \a_t =a\right], \nn \\
		& =   r_t(x,a) +\E^{\pihat} \left[ \left. \sum_{\ell=t+1}^{h} r_\ell(\x_\ell,\a_\ell)\  \right|\   \x_t = x, \a_t =a\right], \label{eq:tojust} \quad  \text{(see below)}  \\
		& = Q_t^{\pihat}(x,a),
	\end{align}
where \eqref{eq:tojust} follows by the fact that conditioned on $(\x_{t},\a_t)=(x,a)$, the sum of rewards $ \sum_{\ell=t+1}^{h} r_\ell(\x_\ell,\a_\ell)$ depend only on $\pihat$ and not on the policy used to roll-in to layer $t$.
\end{proof}

The next lemma shows that the solution $\ghat\ind{t}$ to the least-squares problem in \eqref{eq:mistake} of \cref{alg:PSDP} is close to the $Q$-function in the appropriate sense. 
\begin{lemma}
	\label{lem:reg0}
	Let $\delta \in(0,1)$, $B>0$, $n\geq 1$, and $h \in[H]$ be fixed. Further, let $(\veps_\stat, r_{1:h}, \cG_{1:h}, P\ind{1:h})$ be such that
	\begin{itemize}
		\item $\veps_\stat(n,\delta)^2 =  \frac{cB^2A}{n} (\max_{t\in[h]}{\ln \cN_{\cG_t}(1/n)+\ln (n/\delta)})$, where $c>0$ is a sufficiently large absolute constant.
		\item The function classes $\cG_{1:h}$ realize the reward functions $r_{1:h}: \cX\times \cA\rightarrow \reals$ (in the sense of \Cref{def:funcrealize}).
                \item The functions in $\cG_{1:h}$ are bounded in absolute value by $B$ uniformly.
		\item $P\ind{1},\dots,P\ind{h}\in \Delta(\Pim)$.
		\end{itemize} 
	 Then, for $t\in[h]$, the solution $\ghat\ind{t}$ to the least-squares problem in \eqref{eq:mistake} in \cref{alg:PSDP} when invoked as $\psdp(h, r_{1:h}, \cG_{1:h}, P\ind{1:h}, n)$ satisfies with probability at least $1-\delta$,	
	\begin{align}
	\E_{\pi \sim P\ind{t}}	\E^{\pi}\left[ \max_{a\in\cA}\left( \ghat\ind{t}(\x_t,a)	- Q_t^{\pihat\ind{t+1}}(\x_t, a) \right)^2 \right]\leq    \veps^2_\stat(n,\delta),\nn 
	\end{align}
	where $\pihat\ind{t+1}\in \Pim^{t+1:h}$ is defined as in \cref{alg:PSDP}.
\end{lemma}

\begin{proof}[\pfref{lem:reg0}]
	Fix $t\in[h]$ and abbreviate \[g\ind{t}_{\bayes}\coloneqq g^{P\ind{t},\pihat\ind{t+1}}_{\bayes},\] where $g^{P\ind{t},\pihat\ind{t+1}}$ is defined as in \cref{lem:bayes00} (with $P= P\ind{t}$, $\pihat = \pihat\ind{t+1}$, and reward functions $r_{1:h}$ as in the lemma statement). By \cref{lem:bayes00}, $ g\ind{t}_{\bayes}$ is the Bayes-optimal solution to the least-squares problem in \eqref{eq:mistake} of \cref{alg:PSDP}. Thus, since $\cG_{1:h}$ realize the reward functions $r_{1:h}$, a standard uniform-convergence guarantee for least-square regression (see e.g. \citet[Proposition B.1]{mhammedi2020learning} with $\be = 0$ almost surely) implies that there exists an absolute constant $c>0$ (independent of $t,h$, and any other problem parameters) such that with probability at least $1-\delta$,
		\begin{align}\label{eq:exp}
		\E_{\pi \sim P\ind{t}}	{\E}^{\pi\circ_t\pi_{\unif} \circ_{t+1}\pihat\ind{t+1}} \left[ \left( \ghat\ind{t}(\x_t,\a_t)	-    g\ind{t}_{\bayes}(\x_t,\a_t) \right)^2 \right]\leq  c\cdot{} B^2 \cdot \frac{\ln \cN_{\cG_t}(1/n)+\ln (n/\delta)}{n}.
	\end{align} 
Since actions at layer $t$ are taken uniformly at random, \eqref{eq:exp} implies that
	\begin{align}\label{eq:express}
	\E_{\pi \sim P\ind{t}}	{\E}^{\pi\circ_t\pi_{\unif} \circ_{t+1}\pihat\ind{t+1}} \left[ \max_{a\in\cA}\left( \ghat\ind{t}(\x_t,a)	-    g\ind{t}_{\bayes}(\x_t,a) \right)^2 \right]\leq  c\cdot{} B^2A \cdot \frac{\ln \cN_{\cG_t}(1/n)+\ln (n/\delta)}{n}.
	\end{align}
	The desired result follows by observing that:
	\begin{itemize}
		\item  For all $(x,a)\in \cX_t\times \cA$, $g\ind{t}_{\bayes}(x,a)=Q^{\pihat\ind{t+1}}_t(x,a)$, by \cref{lem:bayes00}.
		\item The term $\max_{a\in\cA}( \ghat\ind{t}(\x_t,a)	-    g\ind{t}_{\bayes}(\x_t,a) )^2$ in \eqref{eq:express} does not depend on the actions $\a_{t:h}$, and so the expectation $\E_{\pi \sim P\ind{t}}	{\E}^{\pi\circ_t\pi_{\unif} \circ_{t+1}\pihat\ind{t+1}}\brk{\cdot}$ can be simplified to $\E_{\pi \sim P\ind{t}}	{\E}^{\pi}\brk{\cdot}$.
	\end{itemize}
\end{proof}

\subsection{Main Guarantee for \psdp{} With Non-Negative Rewards}
\label{sec:pspdnongar}
We now state and prove the main guarantee for \psdp{} used within \cref{thm:spanrlmain}, which is stated with respect to the extended MDP $\cMbar$ defined in \cref{sec:reach0}. This result requires non-negative rewards. For the rest of this section, we make use of the extended MDP notation and definitions introduced in \cref{sec:reach0}. In addition, given non-negative reward functions $r_{1:h}\colon \cX\times \cA \rightarrow \reals_{\geq 0}$, we define their extensions $\bar r_{1:h}$ in $\Mbar$ as
\begin{align} \bar{r}_t(x,a)\coloneqq   \left\{ \begin{array}{ll}  r_t(x,a), & (x,a)\in \cX_t\times\cA \\ 0, &  \text{if } x=\tfrak \text{ or } a=\afrak.
	\end{array}\right. \label{eq:extendedreward}
\end{align}

With this, we now state the guarantee of \psdp{}.
\begin{theorem}[PSDP with non-negative rewards]
	\label{thm:psdp}
	Let $\alpha, \delta,\eta \in(0,1)$, $B>0$, and $h\in[H]$ be given. Consider reward functions $r_{1:h}: \cX\times \cA \rightarrow \reals_{\geq 0}$, function classes $\cG_{1:h}$, policy distribution $P\ind{1:h}$, and a parameter $n\geq 1$ satisfying the following properties:
	\begin{itemize}
		\item The function classes $\cG_{1:h}$, where $\cG_t\subseteq  \{g: \cX_t\times \cA\rightarrow \reals\}$ for $t\in[h]$, realize the reward functions $r_{1:h}$ (in the sense of \Cref{def:funcrealize} with respect to the true MDP), and all functions in $\cG_{1:h}$ have range uniformly bounded by $B$.
		\item 
		For each $1 \leq t \leq h$, it holds that $P\ind{t}$ is a $(\alpha,\eta)$-randomized policy cover relative to $\Pibar_\eta$ for layer $t$ in $\Mbar$ (see \cref{def:approxcover}).
	\end{itemize} 
	Then, with probability at least $1 - \delta$, the policy $\pihat = \psdp(h, r_{1:h}, \cG_{1:h}, P\ind{1:h}, n)$ produced by \cref{alg:PSDP} (when applied to the true MDP), satisfies the following guarantee for $\bar{r}_{1:h}$ as in \eqref{eq:extendedreward}:
	\begin{align}
		\max_{\pi\in \Pibar_\eta}  \Ebar^{\pi}\left[\sum_{t=1}^{h} \bar{r}_t(\x_t,\a_t)\right] \leq  \Ebar^{\pihat}\left[\sum_{t=1}^{h} \bar{r}_t(\x_t,\a_t)\right] + \veps_{\psdp}(n,\delta),
	\end{align}
	where $\veps_\psdp(n,\delta)\coloneqq c\cdot{}H  \sqrt{\alpha^{-1} B^2 A n^{-1}\cdot (\max_{t\in[h]}{\ln \cN_{\cG_t}(1/n)+\ln (n/\delta)})}$ and $c>0$ is an absolute constant.
\end{theorem}
\begin{proof}[{Proof of \cref{thm:psdp}}]
  First, we define extensions of $Q$-functions to the extended MDP $\Mbar$ using the extended rewards $\bar{r}_{1:h}$ in \eqref{eq:extendedreward}; for all $t\in[h]$ and all $\pi\in \Pibarm^{t+1:h}$, define the $Q$-function at layer $t$ in the extended MDP with respect to the extended rewards $\bar{r}_{1:h}$ and partial policy $\pi$:
  \begin{align}
\forall (x,a)\in \cX_t \times \cA,\quad	\widebar{Q}^{\pi}_t(x,a) \coloneqq \bar{r}_t(x,a)+\Ebar^{\pi}\left[\left.\sum_{\ell=t+1}^{h} \bar{r}_\ell(\x_\ell,\a_\ell)\ \right| \ \x_t=x,\a_t=a\right]. \label{eq:extvsnot}
  \end{align}
Note that for any partial policy $\pi\in \Pibarm^{t+1:h}$ that \emph{never} takes the terminal action, we have 
	\begin{align}
	\widebar{Q}^{\pi}_t(x,a)= \left\{ \begin{array}{ll} Q^{\pi}_t(x,a)\geq 0,  &  \text{if } (x,a)\in \cX_t \times \cA, \\  0 , &  \text{if } x = \tfrak \text{ or } a = \afrak,
		\end{array}
	\right.
		 \label{eq:realextend}
	\end{align}
where the fact that $Q^{\pi}_t(\cdot,\cdot)\geq 0$ follows because the rewards are non-negative. Further, for the function $\hat g\ind{t}$ in \cref{alg:PSDP}, we define its (clipped) extension
\begin{align}
	 \bar{g}\ind{t}(x,a)\coloneqq \left\{ \begin{array}{ll} \max(0,\hat{g}\ind{t}(x,a)),  &  \text{if } (x,a)\in \cX_t \times \cA, \\  0 , &  \text{if } x = \tfrak \text{ or } a = \afrak. 
	 \end{array}
	 \right. \label{eq:gbar}
	\end{align}

        To begin, we will show that for any $t\in[h]$ and $\veps_\stat(\cdot,\cdot)$ as in \cref{lem:reg0}, there is an event $\cE_t$ of probability at least $1- \delta/H$ under which the learned partial policies $\pihat\ind{t},\pihat\ind{t+1}$ are such that
	\begin{align}
		\Ebar^{\pi_\star} \left[\widebar{Q}^{\pihat\ind{t+1}}_t(\x_t,\pi_\star(\x_t))- \widebar{Q}^{\pihat\ind{t+1}}_t(\x_t, \pihat\ind{t}(\x_t))\right] \leq  2 \alpha^{-1/2} \veps_\stat(n,\tfrac{\delta}{H}), \label{eq:new00}
	\end{align}
	where $\pi_\star \in \argmax_{\pi \in \Pibar_\eta}  \Ebar^{\pi}[\sum_{t=1}^{h} \bar{r}_t(\x_t,\a_t)]$ is the optimal policy with respect to the truncated policy set $\Pibar_\eta$ (definition in \cref{sec:reach0}) and $\widebar{Q}^{\pi}_t$ is the ${Q}$-function defined in \eqref{eq:realextend}. Once we establish \eqref{eq:new00} for all $t\in[h]$, we will apply the performance difference lemma (\cref{lem:pdl}) and the union bound to obtain the desired result.  
	
Let $\pi_\star \in \argmax_{\pi \in \Pibar_\eta} \Ebar^{\pi}[\sum_{\ell=1}^h \bar{r}_\ell(\x_\ell,\a_\ell)]$. Observe that the following properties hold:
\begin{itemize} 
	\item For all $x\not\in \cX_{t,\eta}(\Pibar_\eta)$, $\pi_\star(x)=\afrak$ (by definition of $\Pibar_\eta$); and 
	\item For all policies $\pi \in \Pibarm^{t+1:h}$ that never take the terminal action, $\widebar{Q}^{\pi}_t(\cdot,\afrak)\equiv 0 \leq  \min_{a\in \wbar{\cA}, y\in \cXbar_t}\widebar{Q}^{\pi}_t(y,a)$ (see \eqref{eq:realextend}),
	\end{itemize} As a result, we have that for any $t\in[h]$ and $\cX_{t,\eta}\coloneqq \cX_{t,\eta}(\Pibar_\eta)$, %
	\begin{align}
		&	\Ebar^{\pi_\star} \left[\widebar{Q}^{\pihat\ind{t+1}}_t(\x_t,\pi_\star(\x_t))- \widebar{Q}^{\pihat\ind{t+1}}_t(\x_t, \pihat\ind{t}(\x_t))\right] \nn\\ 
		& \leq 
		\Ebar^{\pi_\star} \left[ \mathbb{I}\{\x_t \in \cX_{t,\eta}\} \cdot \left( \widebar{Q}^{\pihat\ind{t+1}}_t(\x_t,\pi_\star(\x_t)) - \widebar{Q}^{\pihat\ind{t+1}}_t(\x_t, \pihat\ind{t}(\x_t))\right) \right],\nn \\
		& = 
		\Ebar^{\pi_\star} \left[ \mathbb{I}\{\x_t \in \cX_{t,\eta}\} \cdot \left(\widebar{Q}^{\pihat\ind{t+1}}_t(\x_t,\pi_\star(\x_t))-\bar{g}\ind{t}(\x_t,\pi_\star(\x_t)) + \bar{g}\ind{t}(\x_t,\pi_\star(\x_t))- \widebar{Q}^{\pihat\ind{t+1}}_t(\x_t, \pihat\ind{t}(\x_t))\right)\right],\nn \\
		& \leq \Ebar^{\pi_\star} \left[ \mathbb{I}\{\x_t \in \cX_{t,\eta}\} \cdot \left(\widebar{Q}^{\pihat\ind{t+1}}_t(\x_t,\pi_\star(\x_t))-\bar{g}\ind{t}(\x_t,\pi_\star(\x_t)) + \bar{g}\ind{t}(\x_t,\pihat\ind{t}(\x_t))- \widebar{Q}^{\pihat\ind{t+1}}_t(\x_t, \pihat\ind{t}(\x_t))\right)\right],\nn
		\end{align}
where the last inequality follows by the facts that:
\begin{itemize}
\item $\pihat\ind{t}(x)\in \argmax_{a\in \cA} \ghat\ind{t}(x,a)$, for all $x\in \cX_t$, by the definition of $\pihat\ind{t}$ in  \eqref{eq:pol}.
\item $\bar{g}\ind{t}(\cdot, \afrak)\equiv 0 \leq \bar{g}\ind{t}(\cdot, a)$, for all $a\in\cA$, by definition of $\bar{g}\ind{t}$ in \eqref{eq:gbar}.
\end{itemize}
Continuing from the previous display, we have 
\begin{align}	
	& \Ebar^{\pi_\star} \left[\widebar{Q}^{\pihat\ind{t+1}}_t(\x_t,\pi_\star(\x_t))- \widebar{Q}^{\pihat\ind{t+1}}_t(\x_t, \pihat\ind{t}(\x_t))\right] \nn \\	& \leq 2  \cdot \Ebar^{\pi_\star} \left[\mathbb{I}\{\x_t \in \cX_{t,\eta}\} \cdot \max_{a\in \Abar}\left| \wbar{Q}^{\pihat\ind{t+1}}_t(\x_t,a)-\bar{g}\ind{t}(\x_t,a)\right| \right],\nn \\
	& = 2  \cdot \Ebar^{\pi_\star} \left[\mathbb{I}\{\x_t \in \cX_{t,\eta}\} \cdot \max_{a\in \cA}\left| \wbar{Q}^{\pihat\ind{t+1}}_t(\x_t,a)-\bar{g}\ind{t}(\x_t,a)\right| \right],\quad (\text{since }\wbar{Q}^{\pihat\ind{t+1}}_t(\cdot,\afrak)\equiv \bar{g}\ind{t}(\cdot,\afrak)\equiv 0)  \nn \\
	& \leq 2  \cdot  \sqrt{\Ebar^{\pi_\star} \left[\mathbb{I}\{\x_t \in \cX_{t,\eta}\} \cdot \max_{a\in \cA}\left( \wbar{Q}^{\pihat\ind{t+1}}_t(\x_t,a)-\bar{g}\ind{t}(\x_t,a)\right)^2 \right]}, \quad \text{(Jensen's inequality)}\nn \\
	& = 2 \sqrt{ \int_{\cX_t} \mathbb{I}\{x \in \cX_{t,\eta}\} \cdot \max_{a\in \cA}\left( \wbar{Q}^{\pihat\ind{t+1}}_t(x,a)-\bar{g}\ind{t}(x,a)\right)^2 \dbar^{\pistar}(x) \dd\bar{\nu}(x)},\nn \\
	& \leq 2 \sqrt{\alpha^{-1}   \int_{\cX_t} \mathbb{I}\{x \in \cX_{t,\eta}\} \cdot \max_{a\in \cA}\left( \wbar{Q}^{\pihat\ind{t+1}}_t(x,a)-\bar{g}\ind{t}(x,a)\right)^2  \E_{\pi \sim P\ind{t}}[\dbar^{\pi}(x)] \dd\bar{\nu}(x)},  \label{eq:cover} \quad \text{(justified below)} \\
	& \leq 2 \sqrt{\alpha^{-1} 	\E_{\pi \sim P\ind{t}}\left[ \int_{\cX_t} \max_{a\in \cA}\left( \wbar{Q}^{\pihat\ind{t+1}}_t(x,a)-\bar{g}\ind{t}(x,a)\right)^2 \dbar^{\pi}(x) \dd\bar{\nu}(x)\right]},\quad \text{(Fubini's theorem)}\nn\\
	& = 2 \sqrt{\alpha^{-1}     \cdot	\E_{\pi \sim P\ind{t}}\Ebar^{\pi} \left[\max_{a\in \cA}\left( \wbar{Q}^{\pihat\ind{t+1}}_t(\x_t,a)-\bar{g}\ind{t}(\x_t,a)\right)^2 \right]},\nn \\
	& = 2\sqrt{ \alpha^{-1}   \cdot	\E_{\pi \sim P\ind{t}}\E^{\pi} \left[\max_{a\in \cA}\left( {Q}^{\pihat\ind{t+1}}_t(\x_t,a)-\max(0,\ghat\ind{t}(\x_t,a))\right)^2 \right]},
	\label{eq:preperf2} \\
	& \leq  2 \sqrt{\alpha^{-1}   \cdot	\E_{\pi \sim P\ind{t}}\E^{\pi} \left[\max_{a\in \cA}\left( {Q}^{\pihat\ind{t+1}}_t(\x_t,a)-\ghat\ind{t}(\x_t,a)\right)^2 \right]}, 	\label{eq:preperf}
	\end{align}
where \eqref{eq:cover} follows from the fact that $P\ind{t}$ is an $(\alpha,\eta)$-cover relative to $\Pibar_\eta$ for layer $t$ in $\Mbar$ and $\pi_\star \in \Pibar_\eta$, and \eqref{eq:preperf2} follows because:
\begin{itemize}
	\item The policies in the support of $P\ind{t}$ never take the terminal action; and
	\item  $| Q^{\pihat\ind{t+1}}_t(x',a')-\ghat\ind{t}(x',a')| = | \widebar{Q}^{\pihat\ind{t+1}}_t(x',a')-\max(0,\bar{g}\ind{t}(x',a'))|$, $\forall (x',a')\in \cX_{t}\times \cA$ (see \eqref{eq:gbar} and \eqref{eq:pol}).
	\end{itemize}
Finally, \eqref{eq:preperf} follows by the fact that the $Q$-functions are non-negative (since the rewards are non-negative), and so replacing $\max(0,\hat{g}\ind{t}(\x_t,a))$ by $\hat{g}\ind{t}(\x_t,a)$ on the right-hand side of \eqref{eq:preperf2} only increases the value of the latter.

 Now, from \cref{lem:reg0} and the fact that $\cG_{1:h}$ realize $r_{1:h}$, we have that for any $t\in[h]$, there is an absolute constant $c>0$ (independent of $t$ and other problem parameters) and an event $\cE_t$ of probability at least $1-\delta/H$ under which the solution $\ghat\ind{t}$ to the least-squares regression problem on \eqref{eq:mistake} of \cref{alg:PSDP} satisfies
	\begin{align}
		  	\E_{\pi \sim P\ind{t}}\E^{\pi} \left[\max_{a\in \cA}\left( Q^{\pihat\ind{t+1}}_t(\x_t,a)-\ghat\ind{t}(\x_t,a)\right)^2 \right]\leq  \veps_\stat(n,\tfrac{\delta}{H})^2, \label{eq:neverbefore00}
	\end{align}
	where $\veps_\stat(\cdot,\cdot)^2$ is defined as in \cref{lem:reg0}. Combining \eqref{eq:neverbefore00} with \eqref{eq:preperf} establishes \eqref{eq:new00} under the event $\cE_t$.
	
	To conclude the proof, we note that by the performance difference lemma (\cref{lem:pdl}), we have %
	\begin{align}
		& \Ebar^{\pistar}\left[\sum_{t=1}^{h} \bar{r}_t(\x_t,\a_t)\right] -  \Ebar^{\pihat}\left[\sum_{t=1}^{h} \bar{r}_t(\x_t,\a_t)\right]\nn \\ & = \sum_{t=1}^h	\Ebar^{\pi_\star} \left[\widebar{Q}^{\pihat\ind{t+1}}_t(\x_t,\pi_\star(\x_t))- \widebar{Q}^{\pihat\ind{t+1}}_t(\x_t, \pihat\ind{t}(\x_t))\right].
	\end{align}
	Thus, under the event $\cE \coloneqq \bigcup_{t=1}^h\cE_t$, we have that 
	\begin{align}
		\Ebar^{\pistar}\left[\sum_{t=1}^{h} \bar{r}_t(\x_t,\a_t)\right] -  \Ebar^{\pihat}\left[\sum_{t=1}^{h} \bar{r}_t(\x_t,\a_t)\right]  \leq 2H \alpha^{-1/2} \veps_\stat(n,\tfrac{\delta}{H}).
	\end{align}
	The desired result follows from the union bound, which gives $\P[\cE]\geq 1-\delta$.
\end{proof}

\begin{lemma}[Performance Difference Lemma \citep{kakade2003sample}]\label{lem:pdl}
	Let $\pi,\pistar \in \Pibarm$ be policies, and assume that $\pi$ never takes the terminal action. Let $\widebar{Q}_t^\pi$ be defined as in \eqref{eq:realextend}. Then for any $h\geq 1$,
	\begin{align}
		\Ebar^{\pistar}\left[ \sum_{t = 1}^h \bar{r}_t(\x_t, \a_t) \right] - \Ebar^{\pi}\left[ \sum_{t = 1}^h \bar{r}_t(\x_t, \a_t) \right] = \sum_{t=  1}^h \Ebar^{\pistar}\left[\widebar{Q}_t^{\pi}(\x_t, \pistar(\x_t)) - \widebar{Q}_t^{\pi}(\x_t, \pi(\x_t)) \right]. 
	\end{align}
\end{lemma}

\subsection{Main Guarantee for \psdp{} With Signed Rewards}
We now state and prove a guarantee for \psdp{} in the true MDP $\cM$, when invoked with signed rewards. We make use of the following lemma, which bounds the total probability mass for the set of states that are not reachable with sufficiently high probability.
\begin{lemma}[Probability of non-reachable states]
	\label{lem:reachable+}
	For any $t\in[H]$, it holds that 
	\begin{align}
		\sup_{\pi \in \Pim} \P^{\pi}[\x_t \in \cX_t \setminus \cX_{t,\eta}(\Pim)] \leq \eta \cdot d^{3/2}.
	\end{align}
\end{lemma}
\begin{proof}[\pfref{lem:reachable+}]
	Fix $t\in [H]$. By definition of $\cX_{t,\eta}(\Pim)$, we have that 
	\begin{align}
		\forall x\in \cX_t \setminus \cX_{t,\eta}(\Pim),\quad 	\sup_{\pi \in \Pim} d^{\pi}(x)  \leq \eta \cdot \|\mu^\star_t(x)\|. \label{eq:bounde+}
	\end{align}
	Thus, integrating over $x\in \cX_t \setminus \cX_{t,\eta}(\Pim)$, we obtain 
	\begin{align}
		\sup_{\pi \in \Pim} \P^{\pi}[\x_t \in \cX_t \setminus \cX_{t,\eta}(\Pim)]  & = \sup_{\pi \in \Pim} \int_{\cX_t \setminus \cX_{t,\eta}(\Pim)}  d^{\pi}(x) \dd \nu(x), \nn \\
		& = \eta \cdot \int_{\cX_t \setminus \cX_{t,\eta}(\Pim)} \|{\mu}^\star_t(x)\| \dd \nu(x), \quad \text{(by \eqref{eq:bounde+})} \\
		& \leq \eta \cdot \int_{\cX_{t}}   \|{\mu}^\star_t(x)\| \dd {\nu}(x),\nn \\ 
		& \leq \eta d^{3/2},
	\end{align}	
	where the last inequality follows by \cref{lem:normalization}; this is a consequence of the normalization assumption \eqref{eq:normalization}.
\end{proof}

With this, we now state the guarantee of \psdp{}.
\begin{theorem}[PSDP with signed rewards]
	\label{thm:psdp+}
	Let $\alpha, \delta,\veps \in(0,1)$, $B,B_{1:h}>0$, and $h\in[H]$ be given. Consider reward functions $r_{1}: \cX_1\times \cA \rightarrow [-B_1,B_1],\dots,r_{h}: \cX_h\times \cA \rightarrow [-B_h,B_h]$, function classes $\cG_{1:h}$, distributions over policies $P\ind{1:h}$, and a parameter $n\geq 1$ satisfying the following properties:
	\begin{itemize}
		\item The function classes $\cG_{1:h}$, where $\cG_t\subseteq  \{g: \cX_t\times \cA\rightarrow \reals\}$ for $t\in[h]$, realize the reward functions $r_{1:h}$ (in the sense of \Cref{def:funcrealize}), and all functions in $\cG_{1:h}$ have range uniformly bounded by $B$.
		\item 
		For each $1 \leq t \leq h$, it holds that $P\ind{t}$ is a $(\alpha,\veps)$-randomized policy cover for layer $t$ (see \cref{def:randcover}). 
	\end{itemize} 
	Then, with probability at least $1 - \delta$, the policy $\pihat = \psdp(h, r_{1:h}, \cG_{1:h}, P\ind{1:h}, n)$ produced by \cref{alg:PSDP} satisfies the following guarantee:
	\begin{align}
		\max_{\pi\in \Pim}  \E^{\pi}\left[\sum_{t=1}^{h} {r}_t(\x_t,\a_t)\right] \leq  \E^{\pihat}\left[\sum_{t=1}^{h} {r}_t(\x_t,\a_t)\right] + \veps_{\psdp}(n,\delta) + 2 \veps h d^{3/2} \cdot \sum_{t=1}^h B_t,
	\end{align}
	where $\veps_\psdp(n,\delta)\coloneqq c\cdot{}H  \sqrt{\alpha^{-1} B^2 A n^{-1}\cdot (\max_{t\in[h]}{\ln \cN_{\cG_t}(1/n)+\ln (n/\delta)})}$  and $c>0$ is an absolute constant.
\end{theorem}
\begin{proof}[{Proof of \cref{thm:psdp+}}]
	First, we define the $Q$-functions for the reward ${r}_{1:h}$; for all $t\in[h]$ and all $\pi\in \Pim^{t+1:h}$, define the $Q$-function at layer $t$ with respect to the rewards ${r}_{1:h}$ and partial policy $\pi$:
	\begin{align}
		\forall (x,a)\in \cX_t \times \cA,\quad	{Q}^{\pi}_t(x,a) \coloneqq {r}_t(x,a)+\E^{\pi}\left[\left.\sum_{\ell=t+1}^{h} {r}_\ell(\x_\ell,\a_\ell)\ \right| \ \x_t=x,\a_t=a\right]. \label{eq:extvsnot}
	\end{align}
	To begin, we will show that for any $t\in[h]$ and $\veps_\stat(\cdot,\cdot)$ as in \cref{lem:reg0}, there is an event $\cE_t$ of probability at least $1- \delta/H$ under which the learned partial policies $\pihat\ind{t},\pihat\ind{t+1}$ are such that
	\begin{align}
		\E^{\pi_\star} \left[{Q}^{\pihat\ind{t+1}}_t(\x_t,\pi_\star(\x_t))- {Q}^{\pihat\ind{t+1}}_t(\x_t, \pihat\ind{t}(\x_t))\right] \leq  2 \alpha^{-1/2} \veps_\stat(n,\tfrac{\delta}{H}) + 2 \veps d^{3/2} \cdot \sum_{\ell=1}^h B_\ell, \label{eq:new00+}
	\end{align}
	where $\pi_\star \in \argmax_{\pi \in \Pim}  \E^{\pi}[\sum_{t=1}^{h} r_t(\x_t,\a_t)]$ is the optimal policy. Once we establish \eqref{eq:new00+} for all $t\in[h]$, we will apply the performance difference lemma (\cref{lem:pdl} instantiated in the true MDP) and the union bound to obtain the desired result.  
	
	Let $\pi_\star \in \argmax_{\pi \in \Pim} \E^{\pi}[\sum_{\ell=1}^h r_\ell(\x_\ell,\a_\ell)]$. We have that for any $t\in[h]$ and $\cX_{t,\veps}\coloneqq \cX_{t,\veps}(\Pim)$, %
	\begin{align}
		&	\E^{\pi_\star} \left[{Q}^{\pihat\ind{t+1}}_t(\x_t,\pi_\star(\x_t))- {Q}^{\pihat\ind{t+1}}_t(\x_t, \pihat\ind{t}(\x_t))\right] \nn\\ 
	& = 	\E^{\pi_\star} \left[\mathbb{I}\{\x_t \in \cX_{t,\veps}\} \cdot  \left( {Q}^{\pihat\ind{t+1}}_t(\x_t,\pi_\star(\x_t))- {Q}^{\pihat\ind{t+1}}_t(\x_t, \pihat\ind{t}(\x_t))\right) \right] \nn\\  & \quad  +	\E^{\pi_\star} \left[\mathbb{I}\{\x_t \in \cX_t \setminus \cX_{t,\veps}\} \cdot  \left( {Q}^{\pihat\ind{t+1}}_t(\x_t,\pi_\star(\x_t))- {Q}^{\pihat\ind{t+1}}_t(\x_t, \pihat\ind{t}(\x_t))\right) \right]. \label{eq:lest}
	\end{align}
We now bound the last term in \eqref{eq:lest}. Note that by the range assumption on the rewards $r_{1:h}$ and the definition of the $Q$-function, we have ${Q}^{\pi}_t(x,a)\in [-\sum_{\ell=t}^h B_\ell, \sum_{\ell=t}^h B_\ell]$, for all $\pi \in \Pim^{t+1:h}$. Thus, we have   %
\begin{align}
	\E^{\pi_\star} \left[\mathbb{I}\{\x_t \in \cX_t \setminus \cX_{t,\veps}\} \cdot  \left( {Q}^{\pihat\ind{t+1}}_t(\x_t,\pi_\star(\x_t))- {Q}^{\pihat\ind{t+1}}_t(\x_t, \pihat\ind{t}(\x_t))\right) \right] &  \leq 2\P^{\pi_\star}[\x_t \in\cX_t \setminus \cX_{t,\veps}] \cdot \sum_{\ell=t}^h B_\ell, \nn  \\
	& \leq2 \veps \cdot d^{3/2} \cdot  \sum_{\ell=1}^h B_\ell,
	\label{eq:that}
\end{align}
where the last inequality follows by \cref{lem:reachable+}.
Plugging \eqref{eq:that} into \eqref{eq:lest} and using that $B_{1:h}\geq 0$ implies that
	\begin{align}
			&	\E^{\pi_\star} \left[{Q}^{\pihat\ind{t+1}}_t(\x_t,\pi_\star(\x_t))- {Q}^{\pihat\ind{t+1}}_t(\x_t, \pihat\ind{t}(\x_t))\right] - 2\veps d^{3/2}\cdot \sum_{\ell=1}^h B_\ell \nn\\ 
		& \leq 
		\E^{\pi_\star} \left[ \mathbb{I}\{\x_t \in \cX_{t,\veps}\} \cdot \left( {Q}^{\pihat\ind{t+1}}_t(\x_t,\pi_\star(\x_t)) - {Q}^{\pihat\ind{t+1}}_t(\x_t, \pihat\ind{t}(\x_t))\right) \right],\nn \\
		& = 
		\E^{\pi_\star} \left[ \mathbb{I}\{\x_t \in \cX_{t,\veps}\} \cdot \left({Q}^{\pihat\ind{t+1}}_t(\x_t,\pi_\star(\x_t))-\hat{g}\ind{t}(\x_t,\pi_\star(\x_t)) + \hat{g}\ind{t}(\x_t,\pi_\star(\x_t))- {Q}^{\pihat\ind{t+1}}_t(\x_t, \pihat\ind{t}(\x_t))\right)\right],\nn \\
		& \leq \E^{\pi_\star} \left[ \mathbb{I}\{\x_t \in \cX_{t,\veps}\} \cdot \left({Q}^{\pihat\ind{t+1}}_t(\x_t,\pi_\star(\x_t))-\hat{g}\ind{t}(\x_t,\pi_\star(\x_t)) + \hat{g}\ind{t}(\x_t,\pihat\ind{t}(\x_t))- {Q}^{\pihat\ind{t+1}}_t(\x_t, \pihat\ind{t}(\x_t))\right)\right],\nn
	\end{align}
	where the last inequality follows by the fact that $\pihat\ind{t}(x)\in \argmax_{a\in \cA} \ghat\ind{t}(x,a)$, for all $x\in \cX_t$, by the definition of $\pihat\ind{t}$ in  \eqref{eq:pol}. Continuing from the previous display, we have 
	\begin{align}	
		& \E^{\pi_\star} \left[{Q}^{\pihat\ind{t+1}}_t(\x_t,\pi_\star(\x_t))- {Q}^{\pihat\ind{t+1}}_t(\x_t, \pihat\ind{t}(\x_t))\right] - 2\veps d^{3/2}\cdot \sum_{\ell=1}^h B_\ell \nn \\	& \leq 2  \cdot \E^{\pi_\star} \left[\mathbb{I}\{\x_t \in \cX_{t,\veps}\} \cdot \max_{a\in \cA}\left| Q^{\pihat\ind{t+1}}_t(\x_t,a)-\hat{g}\ind{t}(\x_t,a)\right| \right],\nn \\
		& \leq 2  \cdot  \sqrt{\E^{\pi_\star} \left[\mathbb{I}\{\x_t \in \cX_{t,\veps}\} \cdot \max_{a\in \cA}\left( Q^{\pihat\ind{t+1}}_t(\x_t,a)-\hat{g}\ind{t}(\x_t,a)\right)^2 \right]}, \quad \text{(Jensen's inequality)}\nn \\
		& = 2 \sqrt{ \int_{\cX_t} \mathbb{I}\{x \in \cX_{t,\veps}\} \cdot \max_{a\in \cA}\left( Q^{\pihat\ind{t+1}}_t(x,a)-\hat{g}\ind{t}(x,a)\right)^2 d^{\pistar}(x) \dd\nu(x)},\nn \\
		& \leq 2 \sqrt{\frac{1}{\alpha}   \int_{\cX_t} \mathbb{I}\{x \in \cX_{t,\veps}\} \cdot \max_{a\in \cA}\left( Q^{\pihat\ind{t+1}}_t(x,a)-\hat{g}\ind{t}(x,a)\right)^2  \E_{\pi \sim P\ind{t}}[d^{\pi}(x)] \dd\nu(x)},  \label{eq:cover+} \ \ \text{(justified below)} \\
		& \leq 2 \sqrt{\frac{1}{\alpha} 	\E_{\pi \sim P\ind{t}}\left[ \int_{\cX_t} \max_{a\in \cA}\left( Q^{\pihat\ind{t+1}}_t(x,a)-\hat{g}\ind{t}(x,a)\right)^2 d^{\pi}(x) \dd\nu(x)\right]},\quad \text{(Fubini's theorem)}\nn\\
		& = 2\sqrt{ \frac{1}{\alpha}   \cdot	\E_{\pi \sim P\ind{t}}\E^{\pi} \left[\max_{a\in \cA}\left( {Q}^{\pihat\ind{t+1}}_t(\x_t,a)-\ghat\ind{t}(\x_t,a)\right)^2 \right]},
		\label{eq:preperf+}
	\end{align}
	where \eqref{eq:cover+} follows from the fact that $P\ind{t}$ is an $(\alpha,\veps)$-randomized policy cover for layer $t$.
	Now, from \cref{lem:reg0} and the fact that $\cG_{1:h}$ realize $r_{1:h}$, we have that for any $t\in[h]$, there is an absolute constant $c>0$ (independent of $t$ and other problem parameters) and an event $\cE_t$ of probability at least $1-\delta/H$ under which the solution $\ghat\ind{t}$ to the least-squares regression problem on \eqref{eq:mistake} of \cref{alg:PSDP} satisfies
	\begin{align}
		\E_{\pi \sim P\ind{t}}\E^{\pi} \left[\max_{a\in \cA}\left( Q^{\pihat\ind{t+1}}_t(\x_t,a)-\ghat\ind{t}(\x_t,a)\right)^2 \right]\leq  \veps_\stat(n,\tfrac{\delta}{H})^2, \label{eq:neverbefore00+}
	\end{align}
	where $\veps_\stat(\cdot,\cdot)^2$ is defined as in \cref{lem:reg0}. Combining \eqref{eq:neverbefore00+} with \eqref{eq:preperf+} establishes \eqref{eq:new00+} under the event $\cE_t$.
	
	To conclude the proof, we note that by the performance difference lemma (\cref{lem:pdl}), we have %
	\begin{align}
		& \E^{\pistar}\left[\sum_{t=1}^{h} r_t(\x_t,\a_t)\right] -  \E^{\pihat}\left[\sum_{t=1}^{h} r_t(\x_t,\a_t)\right]\nn \\ & = \sum_{t=1}^h	\E^{\pi_\star} \left[{Q}^{\pihat\ind{t+1}}_t(\x_t,\pi_\star(\x_t))- {Q}^{\pihat\ind{t+1}}_t(\x_t, \pihat\ind{t}(\x_t))\right].
	\end{align}
	Thus, under the event $\cE \coloneqq \bigcup_{t=1}^h\cE_t$, we have that 
	\begin{align}
		\E^{\pistar}\left[\sum_{t=1}^{h} r_t(\x_t,\a_t)\right] -  \E^{\pihat}\left[\sum_{t=1}^{h} r_t(\x_t,\a_t)\right]  \leq 2 H \alpha^{-1/2}  \veps_\stat(n,\tfrac{\delta}{H})  +2 \veps hd^{3/2}\cdot \sum_{t=1}^h B_t.
	\end{align}
	The desired result follows from the union bound, which gives $\P[\cE]\geq 1-\delta$.
\end{proof}

\section{Generic Guarantee for \spanner}
\label{sec:spanner}

In this section, we give a generic guarantee for the $\spanner$ algorithm when invoked with oracles $\apx$ and $\est$ satisfying the following assumption.
\begin{assumption}[$\apx$ and $\est$ as approximate Linear Optimization Oracles]
	\label{ass:spanner}
For some abstract set $\cZ$ and a collection of vectors $\{w^{z}\in \reals^d \mid z\in \cZ\}$ indexed by elements in $\cZ$, there exists $\veps'>0$ such that for any $\theta \in \reals^d\setminus\{0\}$ and $z\in \cZ$, the outputs $\hat z_{\theta} \coloneqq \apx(\theta/\|\theta\|)$ and $\hat w_z \coloneqq \est(z)$ satisfy 
\begin{align}
\sup_{z\in \cZ} \theta^\top w^{z}
 \leq 	\theta^\top w^{\hat z_\theta} +\veps' \cdot \|\theta\|,\quad \text{and} \quad 
  \|\hat w_z - w^{z}\| \leq \veps' .
	\end{align} 
\end{assumption}
Letting $\cW \coloneqq\{w^z \mid z\in \cZ\}$ and assuming that $\cW\subseteq \cB(1)$, the next theorem bounds the number of iterations of $\spanner(\apx(\cdot),\est(\cdot), \cdot,\cdot)$ under \cref{ass:spanner}, and shows that the output is an approximate barycentric spanner for $\cW$ (\cref{def:barycentricspanner}). Our result extends those of \citet{awerbuch2008online}, in that it only requires an \emph{approximate} linear optimization oracle, which is potentially of independent interest.
\begin{proposition}
\label{prop:spanner}
Fix $C>1$ and $\veps\in(0,1)$ and suppose that $\{w^z \mid z \in \cZ\}\subseteq \cB(1)$. If \spanner{} (\Cref{alg:spanner}) is run with parameters $C, \veps>0$ and oracles $\apx$, $\est$ satisfying \cref{ass:spanner} with $\veps'=\veps/2$, then it terminates after $d + \ceil{\frac{d}{2} \log_C\frac{100 d}{\veps^2}}$ iterations, and requires at most twice that many calls to each of $\apx$ and $\est$. Furthermore, the output $z_{1:d}$ has the property that for all $z\in \cZ$, there exist $\beta_{1},\dots,\beta_d\in[-C,C]$, such that 
\begin{align}
	\nrm*{w^z - \sum_{i=1}^d\beta_i w^{z_i}}\leq \frac{3Cd \cdot \veps}{2}. \label{eq:approxspanner}
	\end{align}
\end{proposition}

\begin{proof}[\pfref{prop:spanner}]
	The proof will follows similar steps to those in \citet[Lemma 2.6]{awerbuch2008online}, with modifications to account for the fact that linear optimization over the set $\cW\coloneqq \{w^z \mid z\in \cZ\}$ is only performed approximately. 
	
\paragraph{Part I: Bounding the number of iterations}	
In \Cref{alg:spanner}, there are two loops, both of which require two calls to $\apx$ and $\est$ per iteration. As the first loop has exactly $d$ iterations, it suffices to bound the number of iterations in the second loop.

Let $M\ind{i} \coloneqq (w_1,\dots, w_{i}, e_{i+1}, \dots, e_d)$ be the matrix whose columns are the vectors at end of the $i$th iteration of the first loop (\cref{line:firstfor}) of \cref{alg:spanner}; note that columns $i+1$ through $d$ are unchanged at this point in the algorithm. For $i\in[d]$, we define $\ell_i(w) \coloneqq \det(w,M\ind{i}_{-i})$ and $\theta_i\coloneqq \big(\det\big(e_j, M\ind{i}_{-i}\big)\big)_{j\in [d]}\in \reals^d$, where we recall that for any matrix $A$, the matrix $A_{-i}$ is defined as the result of removing the $i$th column from $A$. Note that $\ell_i$ is linear in $w$, and in particular \[\ell_i(w)\coloneqq w^\top \theta_i.\]

Let $W\ind{0} \coloneqq  M\ind{d} = (w_1, \dots, w_d)$, and let $W\ind{j}$ denote the resulting matrix after $j$ iterations of the second loop (\Cref{line:for}) of \cref{alg:spanner}. We will show that for any $J\geq 1$,
\begin{align}
	\det(W\ind{J}) \leq \det(W\ind{0}) \cdot\left(  \frac{100 d}{\veps^2}\right)^{\frac d2}. \label{eq:fact}
\end{align}
By construction of the loop, we have $\det(W\ind{j}) \geq C \cdot \det(W\ind{j-1})$ for each $j \in[J]$, and thus $\det(W\ind{J}) \geq \det(W\ind{0}) \cdot C^J$. Combining these two facts will establish the bound on the iteration complexity. We now prove \eqref{eq:fact}.

Let $u_i = {e^\top_i}{\big(M\ind{i}\big)^{-1}}$ (note that $u_i$ is a \emph{row vector}) and let $U$ denote the matrix whose $i$th row is $u_i$. We observe that for all $w \in \rr^d$,
\begin{align}
	{u_i}{w} = \frac{\ell_i(w)}{\ell_i(w_i)},
\end{align}
 where we note that $\ell_i(w_i) \neq 0$ by construction; indeed, the columns of $M\ind{i}$ are a basis for $\rr^d$ because $\det(M\ind{i}) \neq 0$, and the equality holds on the columns, so the two linear functions must be equal. Now, since \cref{ass:spanner} holds with $\veps'=\veps/2$, we have
\begin{align}\label{eq:linopt_guarantee}
	{\theta^\top_i}{w_i^+} \geq \sup_{z \in \cZ} {\theta^\top_i}{w^z} - \frac \veps 2 \norm{\theta_i}, \quad \text{and} \quad {\theta^\top_i}{w_i^-} \leq \inf_{z \in \cZ} {\theta^\top_i}{w^z} + \frac \veps 2 \norm{\theta_i},
\end{align}
where $w_i^{\pm} = \est(z_i^{\pm})$. We will now show that
\begin{align}\label{eq:spanner_eq1}
	\ell_i(w_i) \geq \frac{\veps}{2} \cdot \norm{\theta_i}.
\end{align}
There are two cases. First, suppose that ${\theta^\top_i}{w_i^+} \geq - {\theta^\top_i}{w_i^-}$, corresponding to the conditional in \Cref{line:if} of \cref{alg:spanner} being satisfied. Combining this with \eqref{eq:linopt_guarantee}, we have
\begin{align}
	\theta_i^\top w_i^+ & \geq \left( \sup_{z\in \cZ} \theta_i^\top w^z -\frac{\veps}{2}\|\theta_i\| \right) \vee (-\theta_i^\top w_i^-),\nn \\
	& \geq \left( \sup_{z\in \cZ} \theta_i^\top w^z -\frac{\veps}{2}\|\theta_i\| \right)\vee \left( \sup_{z\in \cZ} -\theta_i^\top w^z -\frac{\veps}{2}\|\theta_i\| \right), \quad (\text{by \eqref{eq:linopt_guarantee}})\nn \\
	& = \left( \sup_{z\in \cZ} \theta_i^\top w^z \right)\vee \left( \sup_{z\in \cZ} -\theta_i^\top w^z \right) - \frac{\veps}{2}\|\theta_i\|,\nn \\
	& \geq - \frac{\veps}{2}\|\theta_i\|. \label{eq:lineabove}
\end{align}
Because the conditional is satisfied, $w_i = w_i^+ + \veps \cdot \frac{\theta_i}{\norm{\theta_i}}$, and so by plugging this into \eqref{eq:lineabove}, we have
\begin{align}
	\ell_i(w_i) = {\theta^\top_i}{w_i} \geq \frac{\veps}{2} \cdot \norm{\theta_i}.
\end{align}
The case that ${\theta^\top_i}{w_i^+} \leq - {\theta^\top_i}{w_i^-}$ is essentially identical, establishing \eqref{eq:spanner_eq1}. Now, recall that $\cW \coloneqq  \left\{ w^z \mid z \in \cZ \right\}$ and let $\cW \oplus \cB\left( \frac{3\veps}{2} \right) \coloneqq  \left\{ w + b \mid  w \in \cW \text{ and } b \in  \cB\left( \frac{3\veps}{2} \right) \right\}$ denote the Minkowski sum with $\cB\left( \frac{3\veps}{2} \right)$. By Cauchy-Schwarz,  it holds that for all $w' \coloneqq  w + b \in \cW \oplus \cB\left( \frac{3\veps}{2} \right)$,
\begin{align}
	\ell_i(w') = {\theta^\top_i}{w'} = {\theta^\top_i}{w} + {\theta^\top_i}{b} \leq \left( 1 + \frac{3 \veps}{2} \right) \cdot \norm{\theta_i},
\end{align}
where we used that $\cW \subseteq \cB(1)$ (by assumption). Thus, for any $w' \in \cW \oplus \cB\left( \frac{3\veps}2 \right)$, we have
\begin{align}
	\abs{{u_i}{w'}} = \frac{\ell_i(w')}{\ell_i(w_i)} \leq 1+\frac {3 \veps}{2 }.
\end{align}
We now observe that by construction and the fact that \cref{ass:spanner} holds with $\veps'=\veps/2$, the $k$th column $w_k'$ of $W\ind{J}$ belongs to $\cW \oplus \cB\left( \frac{3 \veps}{2} \right)$, for any $k\in[d]$. Thus, the $(i,k)$ entry ${u_i}{w_k'}$ of $U W\ind{J}$ satisfies ${u_i}{w_k'} \in \left[-1 - \frac {3 \veps}{2}, 1+ \frac{3 \veps}{2} \right]$, and so the columns of $U W\ind{J}$ have Euclidean norm at most $\frac{10 \sqrt{d}}{\veps}$. Since the magnitude of the determinant of a matrix is upper bounded by the product of the Euclidean norms of its columns, it holds that $\abs{\det(U W\ind{J})} \leq \left( \frac{100 d}{\veps^2} \right)^{\frac d2}$. 

On the other hand, again by construction, we see that the columns $w_1,\dots, w_d$ of $W\ind{0}$ satisfy ${u_i}{w_j}=0$, for $j<i$, and ${u_i}{w_i}=1$. Thus, $U {W}\ind{0}$ is an upper-triangular matrix with $1$s on the diagonal, and hence has determinant $1$. Because determinants are multiplicative, this implies that $\det(U) \neq 0$. We now compute:
\begin{align}
	\abs{\det(W\ind{J})} = \frac{\abs{\det(U W\ind{J})}}{\abs{\det(U)}} = \frac{\abs{\det(U W\ind{J})}}{\abs{\det(U W\ind{0})}} \leq \left( \frac{100 d}{\veps^2} \right)^{\frac d2}.
\end{align}
Thus, the upper bound on $\abs{\det(W\ind{J})}$ holds and the claim is proven. Therefore, we have
\begin{align}
	C^J \leq \left( \frac{100 d}{\veps^2} \right)^{\frac d2},
\end{align}
and so $J \leq \left\lceil \frac d2 \log_C\left( \frac{100 d}{\veps^2} \right)\right\rceil $.

\paragraph{Part II: Spanner property for the output} Having shown that the algorithm terminates, we now show that the result is an approximate barycentric spanner for $\cW$. Let $W\coloneqq (w_1, \dots, w_d)$ be the matrix at termination of the algorithm. By definition, if the second loop (\Cref{line:for}) has terminated, then for all $i\in[d]$,   
	\begin{align}
		\max(\theta_i^\top w_i^+, - \theta_i^\top w_i^-) +\veps \cdot \|\theta_i\| \leq C \cdot |\det(w_i,W_{-i})|, 
	\end{align}
	where $\theta_i = (\det(e_j, W_{-i}))_{j\in[d]}\in \reals^d$. On the other hand, by \cref{ass:spanner}, \eqref{eq:linopt_guarantee} holds, and so  
	\begin{align}
		\forall z\in\cZ, \forall i \in [d],\quad   |\det(w^z,W_{-i})|  = |\theta_i^\top w^z| & \leq 	\max(\theta_i^\top w_i^+, - \theta_i^\top w_i^-) +\veps \cdot \|\theta_i\|,\nn \\ &\leq C\cdot  |\det(w_i,W_{-i})|. \label{eq:bod}
	\end{align}
	Now, fix $z\in \cZ$. Since $\det(W) \neq 0$, there exist $\beta_{1:d}\in \reals$ such that $w^z= \sum_{i=1}^d \beta_i w_i$. By plugging this into \eqref{eq:bod} and using the linearity of the determinant, we have
	\begin{align}
		\forall i\in[d], \quad C\cdot  |\det(w_i,W_{-i})| \geq  |\det(w^z,W_{-i})|  = \left|\sum_{j=1}^d \beta_i \det(w_j,W_{-i})\right| = |\beta_i| \cdot |\det(w_i,W_{-i})|.
	\end{align}
	Therefore, $|\beta_i|\leq C$, for all $i\in[d]$. Now, by definition of $w_{1:d}$ and $\wtilde w_{1:d}$, for all $i\in[d]$, we have that $\|w_i -\wtilde w_i\|\leq \veps$. Furthermore, by \cref{ass:spanner}, we also have that $\|\wtilde w_i -w^{z_i}\|\leq \veps/2$. Therefore, by the triangle inequality, we have  
	\begin{align}
		\left\|w^z- \sum_{i=1}^d \beta_i w^{z_i}\right\| \leq 	\left\|w^z- \sum_{i=1}^d \beta_i  w_i\right\| + \sum_{i=1}^d|\beta_i| \| \wtilde w_i - w^{z_i} \| + \sum_{i=1}^d|\beta_i| \| \wtilde w_i - w_i \|  \leq 3d C \veps/2. \label{eq:arm}
	\end{align}
This completes the proof.
\end{proof}

\section{Generic Guarantee for \replearn}
\label{sec:replearn}
In this section, we give a generic guarantee for \replearn{} (\cref{alg:newreplearn}). Compared to previous guarantees in \citet{modi2021model,zhang2022efficient}, we prove a fast $1/n$-type rate of convergence for \replearn{}, and show that the algorithm succeeds even when the norm of the weight $w$ minimized over in \eqref{eq:rep} does not grow with the number of iterations. We also use the slightly simpler discriminator class:
\begin{align}
	\cF \coloneqq 
	\left\{ \left. f\colon x \mapsto \max_{a\in \cA}\theta^\top \phi(x,a)  \,  \right|   \, \theta\in \cB(1), \phi \in \Phi
	\right\}. \label{eq:func}
\end{align}
\begin{algorithm}[tp]
	\caption{$\replearn(h, \cF,\Phi ,P,n)$: Representation Learning for Low-Rank MDPs \citep{modi2021model}}
	\label{alg:newreplearn}
	\begin{algorithmic}[1]\onehalfspacing
		\Require~
		\begin{itemize}
			\item Target layer $h\in[H]$.
			\item Discriminator class $\cF$.
			\item Feature class $\Phi$.
			\item Policy distribution $P\in\Delta(\Pim)$. 
			\item Number of samples $n\in \mathbb{N}$.
		\end{itemize}
		\State Set $\veps_\stat =O(\sqrt{c d^2 n^{-1}\ln (|\Phi|/\delta)})$ for sufficiently absolute constant $c>0$ (see \cref{sec:replearn}).
		\State Let $\phi\ind{1}\in \Phi$ be arbitrary.   
		\State Set $\cD \gets\emptyset$.
		\For{$n$ times}
		\State Sample $\bpi\sim P$.
		\State Sample $(\x_h, \a_h, \x_{h+1})\sim
		\bpi\circ_{h} \pi_{\unif}$.
		\State Update dataset: $\cD \gets \cD \cup \left\{\left(\x_h, \a_h, \x_{h+1}\right)\right\}$.
		\EndFor
		\State Define $\cL_{\cD}(\phi, w, f) =  \sum_{(x, a, x')\in\cD} (\phi(x,a)^\top w  - f(x'))^2  $.
		\For{$t=1,2, \dots$}
		\Statex~~~~~\algcommentbiglight{Discriminator selection}
		\State Solve  \begin{align} 
			f\ind{t} \in \argmax_{f\in \cF}
                               \what\Delta(f),\ \ \text{where}\ \
                               \what\Delta(f)\coloneqq \max_{\tilde
                               \phi\in \Phi} \left\{ \min_{w \in \cB({3} d^{3/2})}
                               \cL_{\cD}(\phi\ind{t}, w, f)  -
                               \min_{\tilde w\in \cB(2\sqrt{d})} \cL_{\cD}(\tilde \phi,
                               \tilde w, f)  \right\}.\label{eq:rep}
                             \end{align}\label{line:replearn}
		\If{$\what\Delta(f\ind{t})\leq 16 d t\veps_\stat^2$}
		\State Return $\phi\ind{t}$.
		\EndIf
		\Statex~~~~~\algcommentbiglight{Feature selection via least-squares minimization}
		\State \label{line:nonadvers}Solve 
		\begin{align}
			\phi\ind{t+1}\in \argmin_{\phi\in \Phi} \min_{(w_{1},\dots,w_{t})\in \cB(2\sqrt{d})^t} \sum_{\ell=1}^t \cL_{\cD}(\phi,w_\ell,f\ind{\ell}). \label{eq:phipayer}
			\end{align}
		\EndFor
	\end{algorithmic}
\end{algorithm}
 
The main guarantee for \replearn{} is as follows. %
\begin{theorem}
	\label{thm:newreplearn}
	Let $h\in [H]$, $\delta \in(0,e^{-1})$, and $n\in\mathbb{N}$ be given, and suppose that $\muh$ satisfies the normalization assumption in \cref{eq:normalization}. 
	For any function $f \in \cF$, define
	\begin{align}
		w_f &= \int_{\cX_{h+1}} f(x) \mustar_{h+1}(x) \dd \nu(x).
	\end{align}
        Let $P\in \Delta(\Pim)$ be a distribution over policies, $\cF$ be as \eqref{eq:func}, and
	$\Phi$ be a feature class satisfying \cref{assum:real}. With probability at least $1 - \delta$, \replearn{} with input $(h, \cF, \Phi, P, n)$ terminates after $t\leq  T\coloneqq  \ceil*{d \log_{{3}/{2}} (2n d^{-1/2})}$ iterations, and its output $\phi\ind{t}$ satisfies 
	\begin{align}
		\sup_{f\in \cF} \inf_{w \in \cB(3d^{3/2})} \E_{\pi\sim P}  \E^{\pi\circ_h \pi_\unif}\left[\left(w^\top \phi\ind{t}(\x_{h},\a_{h})- w_f^\top \phi_h^{\star}(\x_h,\a_h) \right)^2\right] \leq \veps_\replearn^2(n,\delta),\label{eq:replearnmain}
	\end{align}
where $\veps_\replearn^2(n,\delta) \coloneqq c T d^3 n^{-1} \log
(|\Phi|/\delta)$, for some sufficiently large absolute constant $c>0$.
\end{theorem}
To prove the theorem, we need a technical lemma, which follows from \citet[Lemma 14]{modi2021model}.
\begin{lemma}
	\label{lem:notstopper}
	Consider a call to $\replearn(h, \cF, \Phi, P, n)$ (\cref{alg:newreplearn}) in the setting of \cref{thm:newreplearn}. Further, let $\cL_\cD$ be as in \cref{alg:newreplearn} and define
	\begin{align}
		(\phi\ind{t},\what w\ind{t}_1,\dots, \what w\ind{t}_{t-1})\in \argmin_{\phi\in \Phi,(w_{1},\dots,w_{t-1})\in \cB(2\sqrt{d})^{t-1}} \sum_{\ell=1}^{t-1} \cL_{\cD}(\phi,w_\ell,f\ind{\ell}). \label{eq:what}
	\end{align}
 For any $\delta \in(0,1)$, there is an event $\cE\ind{t}(\delta)$ of probability at least $1-\delta$ such that under $\cE\ind{t}(\delta)$, if \cref{alg:newreplearn} does not terminate at iteration $t\geq 1$, then for $w\ind{\ell}\coloneqq w_{f\ind{\ell}}$:
	\begin{align}
		\sum_{\ell =1}^{t-1} \E_{\pi\sim P}  \E^{\pi\circ_h \pi_\unif}\left[\left( \phi\ind{t}(\x_{h},\a_{h})^\top \what w\ind{t}_\ell - \phi_h^{\star}(\x_h,\a_h)^\top  w\ind{\ell} \right)^2\right] &\leq  t \veps_\stat^2(n,\delta), \label{eq:firxt} \\  	\inf_{w \in \frac{3}{2}\cB(d^{3/2})} \E_{\pi\sim P}  \E^{\pi\circ_h \pi_\unif}\left[\left( \phi\ind{t}(\x_{h},\a_{h})^\top w-  \phi_h^{\star}(\x_h,\a_h)^\top w\ind{t} \right)^2\right] & > 8 d t\veps_\stat^2(n,\delta), \label{eq:xecond}
	\end{align}
where $\veps^2_\stat(n,\delta)\coloneqq c d^2 n^{-1}\ln
(|\Phi|/\delta)$ and $c\geq1$ is a sufficiently large absolute constant.
\end{lemma}
With this, we prove \cref{thm:newreplearn}.
\begin{proof}[Proof of \cref{thm:newreplearn}]
	Let us abbreviate $\veps \coloneqq \veps_\stat(n,\delta)$,
        with $\veps_\stat(n,\delta)$ defined as in \cref{lem:notstopper}. Further, let $N\coloneqq 1+ \ceil*{d \log_{{3}/{2}} (2d^{3/2}/\veps)}$, $\delta' \coloneqq \frac{\delta}{2N}$, and define 
	\begin{align}
		\tilde\veps_{\stat} \coloneqq \veps_\stat(n,\delta'). \label{eq:tildeeps}
		\end{align}
	 Note that $\veps \leq \tilde\veps_\stat$ and $N -1 \leq T$, where $T$ is the number of iterations in the theorem statement; the latter inequality follows by the facts that the absolute constant $c$ in \cref{lem:notstopper} is at least $1$ and $\ln (|\Phi|/\delta)\geq1$. We define an event $\cE\coloneqq \cE\ind{1}(\delta')\cap \dots \cap \cE\ind{N}(\delta')$, where $(\cE^{t}(\cdot))_t$ are the success events in \cref{lem:notstopper}. Note that $\P[\cE]\geq 1 - \delta/2$ by the union bound. Throughout this proof, we condition on the event $\cE$. 
	
         To begin the proof, we define a sequence of vectors $(v_{1:d}\ind{\ell})_{\ell\geq 0}$ in an inductive
         fashion, with $v_{i}\ind{\ell}\in\bbR^{d}$ for all
         $i\in\brk{d}$ and $\ell\geq{}0$. For $\ell=0$, we let
           $v_{i}\ind{0} = \veps e_i/d$, for all $i\in[d]$. For
           $\ell\geq 1$, we consider two cases:
           \begin{itemize}
           \item \textbf{Case I:} If
             \begin{align}
               \cJ\ind{\ell} \coloneqq	 \left\{j \in[d] \ \left| \  |\det(V_{-j}\ind{\ell-1}, w\ind{\ell})|>(1+C)\cdot |\det(V\ind{\ell-1})| \right. \right\} \neq \emptyset,
             \end{align}
             where
             $V\ind{\ell-1}\coloneqq (v_1\ind{\ell-1},\dots,
             v_d\ind{\ell-1})\in \reals^{d\times d}$ and
             $w\ind{\ell}\ldef{}w_{f\ind{\ell}}$, then we let
             $j\coloneqq \argmin_{j'\in\cJ\ind{\ell}}j'$ and define
             \begin{align}
               v_i\ind{\ell} \coloneqq \left\{ \begin{array}{ll}  w\ind{\ell} , & \text{if } i=j, \\ v_i\ind{\ell-1}, &  \text{otherwise}. \end{array} \right.
             \end{align}
           \item \textbf{Case II}: If $\cJ\ind{\ell}=\emptyset$, we let
             $v_i\ind{\ell} = v_i\ind{\ell-1}$, for all $i\in[d]$.
           \end{itemize}
	
	We first show that $\cJ\ind{t}\neq\emptyset$ at any iteration $t\in[N]$ where \replearn{} does not terminate. Let $t\in[N]$ be an iteration where the algorithm does not terminate, and suppose that $\cJ\ind{t}=\emptyset$. This means that 
	\begin{align}
		\forall j\in[d] ,\quad  |\det(V_{-j}\ind{t-1}, w\ind{t})|\leq (1+C)\cdot |\det(V\ind{t-1})|. \label{eq:detguar}
	\end{align}
	Now, since $\det(V\ind{t-1})\neq 0$ ({note that
          $\abs*{\det(V\ind{t})}$ is non-decreasing with $t$}), we have
        that $\mathrm{span}( V\ind{t-1})= \reals^d$. Thus, there exist
        $\beta_1,\dots, \beta_d\in\bbR$ be such that $w\ind{t}=
        \sum_{i=1}^d \beta_i  v\ind{t-1}_i$. By the linearity of the
        determinant and \eqref{eq:detguar}, we have 
	\begin{align}
		\forall j \in[d], \quad 	(1+C)|\cdot \det(V\ind{t-1})|  & \geq  |\det(V_{-j}\ind{t-1}, w\ind{t})|,\nn \\
		& =  \left|\det\left(V_{-j}\ind{t-1},  \sum_{i=1}^d \beta_i  v\ind{t-1}_i \right)\right|,\nn \\
                 &= \abs*{\sum_{i\in[d]} \beta_i\cdot  \det(V_{-j}\ind{t-1}, v_i\ind{t-1})},\nn \\
		& = |\beta_j| \cdot |\det(V\ind{t-1})|.
	\end{align}
	This implies that $|\beta_j|\leq (1+C)$ for all
        $j\in[d]$. Now, note that by the definition of $(v_i\ind{t-1})$, we have that for any $i\in[d]$ such that $v_i\ind{t-1}\neq \veps e_i/d$, there exists $\ell\in [t-1]$ such that $w\ind{\ell}= v_i\ind{t-1}$. Let \[\cI\ind{t}\coloneqq \{i\in[d]\mid   v_i\ind{t-1}\neq \veps e_i/d\},\] and for any $i\in\cI\ind{t}$, let $\ell_i\in[t-1]$ be such that $w\ind{\ell_i}= v_i\ind{t-1}$. Further, define 
	\begin{align}
		\wtilde{w}\ind{t} \coloneqq  \sum_{i\in\cI\ind{t}} \beta_{i} w\ind{\ell_i}= \sum_{i\in\cI\ind{t}} \beta_{i} v_i\ind{t-1},\label{eq:wtilt}
	\end{align}
	and note that by the triangle inequality and the fact that $w\ind{t}=\sum_{i=1}^d \beta_i v_i\ind{t-1}$, we have 
	\begin{align}
		\|\wtilde w\ind{t}- w\ind{t}\|\leq (1+C)\veps_\stat. \label{eq:thisee}
	\end{align}
	Finally, with the notation in \eqref{eq:what}, define
	\begin{align}
		\what w\ind{t}_t & \coloneqq \sum_{i\in\cI\ind{t}}  \beta_i \what{w}\ind{t}_{\ell_i},\label{eq:whatt}
		\shortintertext{and note that}
		\what w\ind{t}_t &\in (1+C) \cB(2d^{3/2}), \label{eq:contained}
	\end{align}
	since $|\beta_i| \leq (1+C)$ for all $i\in[d]$, $|\cI\ind{t}|\leq  d$, and $\what w\ind{t}_{\ell}\in \cB(2\sqrt{d})$, for all $\ell\in[t-1]$. Now, by \cref{lem:notstopper}, in particular \eqref{eq:firxt}, we have 
	\begin{align}
		\sum_{i\in \cI\ind{t}}	 \E_{\pi\sim P}  \E^{\pi\circ_h \pi_\unif}\left[\left( \phi\ind{t}(\x_{h},\a_{h})^\top \what w\ind{t}_{\ell_i} - \phi_h^{\star}(\x_h,\a_h)^\top  w\ind{\ell_i} \right)^2\right] &\leq t \tilde \veps_\stat^2, \label{eq:guar}
	\end{align} 
where $\tilde \veps_\stat$ is as in \eqref{eq:tildeeps}. Using the
expressions in \cref{eq:wtilt,eq:whatt} with \eqref{eq:guar} and Jensen's inequality, we have that under $\cE\ind{t}$,
	\begin{align}
		& \E_{\pi\sim P}  \E^{\pi\circ_h \pi_\unif}\left[\left( \phi\ind{t}(\x_{h},\a_{h})^\top \what w\ind{t}_{t} - \phi_h^{\star}(\x_h,\a_h)^\top  \wtilde w\ind{t} \right)^2\right]\\  & \leq  \left(\sum_{j\in \cI\ind{t}}  |\beta_j|\right) \cdot \sum_{i\in\cI\ind{t}}  \E_{\pi\sim P}  \E^{\pi\circ_h \pi_\unif}\left[\left( \phi\ind{t}(\x_{h},\a_{h})^\top \what w\ind{t}_{\ell_i} - \phi_h^{\star}(\x_h,\a_h)^\top  w\ind{\ell_i} \right)^2\right] ,\nn \\
		& \leq (1+C) d t \tilde\veps_\stat^2.
	\end{align}
	Now, using \eqref{eq:thisee} and the facts that $(a+b)^2 \leq 2a^2 + 2 b^2$ and $\|\phi^\star_h\|_{2}\leq 1$, we have that 
	\begin{align}
		\E_{\pi\sim P}  \E^{\pi\circ_h \pi_\unif}\left[\left( \phi\ind{t}(\x_{h},\a_{h})^\top \what w\ind{t}_{t} - \phi_h^{\star}(\x_h,\a_h)^\top  w\ind{t} \right)^2\right] & \leq 2(1+C)^2 \veps^2 + 2(1+C)dt \tilde\veps_\stat^2,\nn \\
		& \leq 2(1+C)^2 \tilde\veps^2_{\stat} + 2(1+C)dt \tilde\veps_\stat^2.
	\end{align}
	Using that $C=1/2$, we conclude that the right-hand side of this inequality is bounded by $8 d t\tilde\veps_\stat^2$ which is a contradiction, since $\what w\ind{t}_t \in (1+C)\cB(2d^{3/2}) =  \cB(3d^{3/2})$ and by \cref{lem:notstopper}, we must have  
	\begin{align}
		\inf_{w\in \cB(3d^{3/2})}	\E_{\pi\sim P}  \E^{\pi\circ_h \pi_\unif}\left[\left( \phi\ind{t}(\x_{h},\a_{h})^\top w-  \phi_h^{\star}(\x_h,\a_h)^\top w\ind{t} \right)^2\right]>  8 t \tilde \veps_\stat
	\end{align}
	if \replearn{} does not terminate at round $t$.
	Therefore, we have that $\cJ\ind{t}\neq \emptyset$, for any
        iteration $t\in[2 \ldotst N]$ where \replearn{} does not
        terminate.

We now bound the iteration count and prove that the guarantee in
\cref{eq:replearnmain} holds at termination. Note that whenever $\cJ\ind{\ell}\neq \emptyset$ for $\ell>1$, we have by construction:
	\begin{align}
		|\det(V\ind{\ell})| > 3/2  \cdot |\det(V\ind{\ell-1})|. 
	\end{align}
	Thus, if \replearn{} runs for $t\in[2\ldotst N]$ iterations, then 
	\begin{align}
		|\det(V\ind{t})| > (3/2)^{t-1}  \cdot |\det(V\ind{1})|.  \label{eq:loer}
	\end{align}
	On the other hand, since the determinant of a matrix is bounded by the product of the norms of its columns and $v_{1:d}\ind{t}\in \cB(2\sqrt{d})$, we have 
	\begin{align}
		|\det(V\ind{t})| \leq 2^d d^{d/2}. 
	\end{align}
	Note also that $|\det(V\ind{0})| = (\veps/d)^d$. Plugging this
        into \eqref{eq:loer}, we conclude that
	\begin{align}
		(3/2)^{t-1} < (2d^{3/2}/\veps)^{d}.
	\end{align}
	Taking the logarithm on both sides and rearranging yields
	\begin{align}
		t < 1+ d \log_{{3}/{2}} (2d^{3/2}/\veps)\leq N.
	\end{align}
Thus, the algorithm must terminate after at most $N-1$ iterations. Furthermore, by \citep[Lemma 14]{modi2021model}, we have that with probability at least $1-\frac{\delta}{2N}$, if the algorithm terminates at iteration $t$, then
	\begin{align}
		\max_{f\in \cF} \inf_{w \in  \cB(3d^{3/2})} \E_{\pi\sim P}  \E^{\pi\circ_h \pi_\unif}\left[\left(w^\top \phi\ind{t}(\x_{h},\a_{h})- w_f^\top \phi_h^{\star}(\x_h,\a_h) \right)^2\right] & \leq 32 t \tilde \veps_\stat^2,\nn \\
		& \leq 32 (N-1)\tilde \veps_\stat^2, \nn \\
		& \leq 32 T \tilde \veps_\stat^2.
	\end{align}
Applying a
        union bound completes the proof.
\end{proof}

	\section{Analysis: Proof of \crefzak{thm:spanrlmain}}
	\label{sec:analysis}
In this section, we prove the main guarantee for  \mainalg{} (\cref{thm:spanrlmain}). First, we outline the proof strategy in \cref{sec:strategy}. Then, in \cref{sec:PSDPspanner} and  \cref{sec:spannerSpanRL}, we present guarantees for the instances of \psdp{} (\cref{alg:PSDP}) and \spanner{} (\cref{alg:spanner}) used within \mainalg. We then combine these results in \cref{sec:mainproof} to complete the proof of \cref{thm:spanrlmain}.

\subsection{Proof Strategy} 
\label{sec:strategy}
The proof of \cref{thm:spanrlmain} is inductive. For fixed $h$, we assume that the policy set $\Psi\ind{1:h+1}$ produced by $\mainalg$ satisfies the properties
\begin{enumerate}
\item $\unif(\Psi\ind{1}),\dots,\unif(\Psi\ind{h+1})$ are $\big(\tfrac{1}{\tC  Ad},\eta\big)$-randomized policy covers relative to $\Pibar_\eta$ for layers $1$ through $h+1$ in $\Mbar$ (\cref{def:approxcover}); and\label{eq:induct}
\item and $\max_{t\in[h+1]}|\Psi\ind{t}|\leq d$.  \label{eq:induct0}
\end{enumerate}
Conditioned on this claim, we show that with high probability, the set $\Psi\ind{h+2}$  is a $(\frac{1}{4 A d},\eta)$-policy cover relative to $\Pibar_\eta$ for layer $h +2$ in $\Mbar$. To prove this, we use the inductive assumption to show that $\psdp$ acts as an approximate linear optimization oracle over $\cW = \{ \ee^\pi\left[ \phih\ind{h}(\x_h, \a_h) \cdot \mathbb{I}\{\phih\ind{h}(\x_h, \a_h)^\top \theta\geq 0 \}\right] \mid \pi \in \Pim,\theta \in \reals^d \}$ (\Cref{sec:PSDPspanner}). Using this, we then invoke the guarantee for \spanner{} from \cref{lem:spanner} with $\apx$ and $\est$ instantiated with \psdp{} and \veceval{}. To conclude the proof of the inductive step, we combine the main guarantee for \spanner{} together with that for \replearn{} (\Cref{thm:newreplearn}), along with a change of measure argument enabled by the assumption that $\Psi\ind{1:h}$ are policy covers (i.e.~\cref{eq:induct}). As in \citet{mhammedi2023representation}, a key feature of the analysis is that we work with the extended MDP and truncated policy class throughout the proof, only passing back to the true MDP once the induction is complete and \cref{eq:induct} has been proven to hold for all layers $H$. To pass back to the true MDP, we use the following lemma.
	\begin{lemma}
		\label{lem:transfer0}
		Let $h\in [H]$, $\alpha\in (0,1)$, and $\eta >0$ be given. 
		If $\unif(\Psi\ind{h})\in \Delta(\Pim)$ is an $(\alpha,\eta)$-randomized policy cover relative to $\Pibar_\eta$ for layer $h$ in $\Mbar$ (\cref{def:approxcover}), then $\unif(\Psi\ind{h})$ is an $(\alpha/2,\veps)$-randomized policy cover relative to $\Pim$ for layer $h$ in the true MDP $\cM$ (\cref{def:randcover}), where $\veps \coloneqq 4 H d^{3/2}\eta$.
	\end{lemma}
The lemma is restated and proven in \cref{sec:structural}. Note that by \cref{rem:backandforth}, if $\unif(\Psi\ind{h})$ is an $(\alpha/2,\veps)$-randomized policy cover relative to $\Pim$ for layer $h$ in the true MDP $\cM$ (as in \cref{lem:transfer0}), then $\Psi\ind{h}$ is an $(\alpha/2,\veps)$-policy cover relative to $\Pim$ for layer $h$ in the true MDP $\cM$ in the (non-randomized) sense of \cref{def:polcover101}.

\subsection{Guarantee for \psdp as a Subroutine for \spanner}
\label{sec:PSDPspanner}
We begin by showing that $\psdp$, as configured within \mainalg, acts as an approximate linear optimization oracle as required by $\spanner$. In particular, we fix a layer $h$, assume that $\Psi\ind{1:h+1}$ satisfy \eqref{eq:induct}, and then apply the generic guarantees for \psdp in \Cref{sec:nonnegative}.

	For $\theta \in \cB(1)$ and $\phi \in \Phi$, define function classes $\cG_{1:h}$ as follows \begin{align}
 \forall h \in [h-1], \quad  \cG_t \coloneqq \cG \coloneqq \{g:(x,a)\mapsto \phi(x,a)^\top w \mid \phi \in \Phi, w \in \cB(2\sqrt{d})\}, \ \ \text{and} \ \ \cG_h \coloneqq \{r'_h(\cdot, \cdot; \theta, \phi)\}, \label{eq:Gclass}
\end{align} 
where we define reward functions $r'_{1:h}(\cdot,\cdot;\theta, \phi)$ by:
\begin{align}
	\forall (x,a)\in \cX\times \cA, \quad  r'_{t}(x,a;\theta,\phi)\ldef{} \left\{\begin{array}{ll} \phi(x,a)^\top \theta \cdot \mathbb{I}\{ \phi(x,a)^\top \theta\geq 0\}, &  \text{for }
		t=h, \\ 0, &  \text{otherwise}.
	\end{array}\right. \label{eq:reward}
\end{align}
With these rewards and function classes, we will show that for any $\theta \in \cB(1)$ and $\phi \in \Phi$, the output \[\pihat = \psdp(h, r'_{1:h}(\cdot, \cdot;\theta,\phi), \cG_{1:h}, P\ind{1:h}, n),\] where $P\ind{t}\coloneqq \unif(\Psi\ind{t})$, for each $t\in[h]$, satisfies the property that for $\eta$ as in \cref{alg:spanRL}:
\begin{align}
	\max_{\pi \in \Pibar_\eta} \theta^\top {\Ebar^\pi\left[ \tilde\phi(\x_h, \a_h) \cdot \mathbb{I}\{\tilde\phi(\x_h, \a_h)^\top \theta \geq 0 \} \right]}\leq  \theta^\top {\Ebar^{\pihat}\left[ \tilde\phi(\x_h, \a_h) \cdot \mathbb{I}\{\tilde\phi(\x_h, \a_h)^\top \theta \geq 0 \} \right]} + \frac{\eta}{36 d^{5/2}},  \label{eq:tosolve}
\end{align}
with high probability if $n\geq 1$ is sufficiently large; recall that $\tilde\phi$ is the restriction of $\phibar$ to its first $d$ coordinates, with $\phibar$ as in \cref{sec:reach0}. 

Note that this matches the choice of reward functions in \mainalg{} (\cref{alg:spanRL}) at iteration $h$ with $\phi = \phi\ind{h}$, the feature map returned by \replearn{} in \cref{line:reward}.

We first verify that the classes $\cG_{1:h}$ realize the reward functions specified in \eqref{eq:reward} in the sense of \Cref{def:funcrealize}.
\begin{lemma}
	\label{lem:realizefunc}
	Under \cref{assum:real}, the function classes $\cG_{1:h}$ in \eqref{eq:Gclass} realize (\cref{def:funcrealize}) the reward functions in \eqref{eq:reward} for any $\phi\in\Phi$ and $\theta\in \cB(1)$. Furthermore, the functions in $\cG_{1:h}$ are uniformly bounded by $2\sqrt{d}$, and for any $\veps'>0$, $\ln \cN_{\cG_t}(\veps')\leq \ln |\Phi|+ d \ln (2\sqrt{d} /\veps')$, for all $t\in[h]$, where we recall that $\cN_{\cG}(\veps')$ denotes the $\veps'$-covering number of $\cG$ in $\ell_\infty$-distance (see \cref{def:covering}).
\end{lemma} 
\begin{proof}[\pfref{lem:realizefunc}]
	Fix $\phi\in\Phi$ and $\theta\in \cB(1)$, and let $r'_t(\cdot,\cdot)\equiv r'_t(\cdot,\cdot; \theta, \phi)$, for $t\in[h]$. Further, for $t\in[h]$ and $\pi\in \Pim^{t+1:h}$, we define the \emph{state-action value function} ($Q$-function) at layer $t$ with respect to the rewards $r'_{1:h}$ and partial policy $\pi$:
	\begin{align}
		\forall (x,a)\in \cX_t\times \cA,\quad 		Q^{\pi}_t(x,a)\coloneqq r'_t(x,a)+\E^{\pi}\left[\left.\sum_{\ell=t+1}^{h} r'_\ell(\x_\ell,\a_\ell)\ \right| \ \x_t=x,\a_t=a\right]. \nn
	\end{align}
	For $t=h$, we clearly have that for any $\pi \in \Pim^{h:h}$, $Q^{\pi}_h(\cdot,\cdot)=r'_h(\cdot,\cdot)\in \cG_h$. For $t<h$ and $\pi \in \Pim^{t+1:h}$, we have by the low-rank structure that 
	\begin{align}
		Q^{\pi}_t(x,a) & = \int_{\cX_{t+1}} \E^{\pi}[r'_h(\x_h,\a_h)\mid \x_{t+1}=y,\a_{t+1}=\pi(y)] \cdot \phi^\star_t(x,a)^\top  \mu_{t+1}^\star(y)  \dd\nu (y), \nn \\
		& =  \phi^\star_t(x,a)^\top  \left( \int_{\cX_{t+1}} \E^{\pi}[r'_h(\x_h,\a_h)\mid \x_{t+1}=y,\a_{t+1}=\pi(y)] \cdot \mu_{t+1}^\star(y)  \dd\nu (y)\right).\label{eq:hand}
		\end{align}
	Now, by the fact that $\E^{\pi}[r'_h(\x_h,\a_h)\mid \x_{t+1}=y,\a_{t+1}=\pi(y)] \in [-1,1]$, for all $y\in \cX_{t+1}$ (since $\phi(\cdot,\cdot)\in \cB(1)$, for all $\phi\in\Phi$), and the normalizing assumption made on $(\muh[h])_{h\in[H]}$ in \cref{sec:onlineRL} (i.e.~that for all $g:\cX_{t+1}\to\brk{0,1}$, $\nrm*{\int_{\cX_{t+1}} \muh[t+1](y)g(y) \dd\nu(y)} \leq \sqrt{d}$), we have that 
	\begin{align}
		w_t \coloneqq \int_{\cX_{t+1}} \E^{\pi}[r'_h(\x_h,\a_h)\mid \x_{t+1}=y,\a_{t+1}=\pi(y)] \cdot \mu_{t+1}^\star(y)  \dd\nu (y) \in \cB(2\sqrt{d}). \label{eq:one}
		\end{align}
This, together with \eqref{eq:hand} and the fact that $\phistarh[t]\in \Phi$ (by \cref{assum:real}), implies that that $Q_t^{\pi} \in \cG_t$. The bound on the covering number $\cN_{\cG_t}(\veps')$, follows from a standard bound on the covering number of the ball $\cB(2\sqrt{d})$ \citep{wainwright2019high}.
\end{proof}

Combining \Cref{lem:realizefunc} with \Cref{thm:psdp+} (with $\veps=0$) results in the following bound on the quality of \psdp{} as an approximate linear optimization oracle.%
\begin{corollary}
	\label{cor:psdp}
	Let $\delta \in(0,1)$ and $n\in \mathbb{N}$ be given and fix $h\in[H]$. Given $\theta\in \cB(1)$ and $\phi\in \Phi$, let $\pihat$ be the output of $\psdp$ when given input $(h, r'_{1:h}(\cdot, \cdot;\theta,\phi), \cG_{1:h}, P\ind{1:h}, n)$, where
	\begin{itemize}
		\item The reward functions $r'_{1:h}(\cdot, \cdot;\theta,\phi)$ are as in \eqref{eq:reward}.
		\item The function classes $\cG_{1:h}$ are as in \eqref{eq:Gclass}.
		\item $P\ind{t}\coloneqq \unif(\Psi\ind{t})$, for each $t\in[h]$, and the collection of policies $\Psi\ind{1:h}$ satisfy \eqref{eq:induct}.
	\end{itemize} 
	Then, for any $\eta \in(0,1)$, under \cref{assum:real}, with probability at least $1-\delta$, we have that 
	\begin{align}
		\max_{\pi\in \Pibar_\eta} \theta^\top \Ebar^{\pi}[\tilde\phi(\x_h,\a_h)\cdot\mathbb{I}\{\tilde\phi(\x_h,\a_h)^\top \theta \geq 0\}]\leq \theta^\top \Ebar^{\pihat}[\tilde\phi(\x_h,\a_h)\cdot \mathbb{I}\{\tilde\phi(\x_h,\a_h)^\top \theta \geq 0\}] + \veps_\psdp(n,\delta), 
	\end{align}
where $\veps_\psdp(n,\delta)\coloneqq c H A  d \sqrt{d n^{-1} ({d \ln (2n d^{1/2})+\ln ({|\Phi|}/{\delta})})}$ for a sufficiently large absolute constant $c>0$.
\end{corollary}

\subsection{Guarantee for $\spanner
$ as a Subroutine for \mainalg}
\label{sec:spannerSpanRL}
\begin{algorithm}[tp]
	\caption{$\veceval(h, F, \pi,n)$: Estimate
          $\E^{\pi}[F(\x_h,\a_h)]$ for policy $\pi$ and function
        $F:\cX\times\cA\rightarrow\reals^d$.}
	\label{alg:veceval}
	\begin{algorithmic}[1]\onehalfspacing
		\Require~
		\begin{itemize}
			\item Target layer $h\in[H]$.
			\item Vector-valued function $F:\cX \times \cA\rightarrow \reals^d$.
			\item Policy $\pi\in \Pim$.
			\item Number of samples $n\in \mathbb{N}$.
		\end{itemize}
		\State $\cD \gets\emptyset$. 
		\For{$n$ times}
		\State Sample $(\x_h, \a_h)\sim
		\pi$.
		\State Update dataset: $\cD \gets \cD \cup \left\{\left(\x_h, \a_h\right)\right\}$.
		\EndFor
		\State \textbf{Return:} $\bar F = \frac{1}{n}\sum_{(x, a)\in\cD} F(x,a)$. 
	\end{algorithmic}
\end{algorithm}

In this section, we prove a guarantee for the invocation of $\spanner$ within $\mainalg$. We first show that $\veceval$ (\Cref{alg:veceval})  is a valid choice for the $\est$ subroutine passed to $\spanner$.
\begin{lemma}[Guarantee of \veceval{}]
\label{lem:veceval}
Let $\delta \in(0,1)$, $h\in[H]$, $\phi\in \Phi$, $\theta \in \reals^d$, $\pi \in \Pim$, and $n\in \mathbb{N}$ be given. The output $\phi^{\veceval}_h= \veceval(h,\phi(\cdot,\cdot) \cdot \mathbb{I}\{\phi(\cdot,\cdot)^\top \theta \geq 0\},\pi, n)$ (\cref{alg:veceval}) satisfies, with probability at least $1-\delta$,  
\begin{align}
	\| \phi^{\veceval}_h -  \E^{\pi}[\phi(\x_h,\a_h) \cdot \mathbb{I}\{\phi(\x_h,\a_h)^\top \theta \geq 0 \}] \| \leq \veps_{\veceval}(n,\delta), \label{eq:est}
	\end{align}
where $\veps_{\veceval} \coloneqq c \cdot \sqrt{n^{-1} \cdot \log (1/\delta)}$ and $c>0$ is a sufficiently large absolute constant.
\end{lemma}
\begin{proof}[\pfref{lem:veceval}]
	By a standard vector-valued concentration bound in Euclidean space (see for example \citet{pinelis1994optimum}) and the fact that $\norm{\phi(x, a)} \leq 1$ for all $x \in \cX$ and $a \in \cA$, there exists an absolute constant $c>0$ such that with probability at least $1 - \delta$,
	\begin{align}
		\nrm*{\phi^{\veceval}_h - \ee^\pi\left[ \phi(\x_h, \a_h) \cdot \mathbb{I}\{\phi(\x_h,\a_h)^\top \theta \geq 0 \} \right]} \leq c \cdot \sqrt{\frac{\log(1/\delta)}{n}}.
	\end{align}
	\end{proof}
Recall that in $\mainalg$, we instantiate $\spanner$ passing $\veceval$ as $\est$ and $\psdp$ as part of the output of $\apx$ (see \cref{line:linopt}). Combining \Cref{cor:psdp,lem:veceval} with the general guarantee for $\spanner$ in \Cref{prop:spanner}, we have the following result.
\begin{lemma}\label{lem:spanner}
Consider iteration $h\in [H]$ of \mainalg{}$(\Phi,\veps,\cfrak,\delta)$ (\cref{alg:spanRL}) with $\tveps,\cfrak>0$, $\delta\in(0,1)$, and feature class $\Phi$ satisfying \cref{assum:real}. Further, let $\eta \coloneqq \veps/(4 Hd^{3/2})$ and $\phih\ind{h}$ denote the feature map returned by $\replearn$ in \Cref{alg:spanRL} at iteration $h$. If $\Psi\ind{1:h}$ in \cref{alg:spanRL} satisfy \cref{eq:induct} and \cref{eq:induct0}, and $\cfrak=\polylog(A,d,H,\ln(|\Phi|/\delta))$ is sufficiently large, then with probability at least $1 - \frac{\delta}{2H}$, we have that
\begin{itemize}
	\item The number of iterations of $\spanner$ in \Cref{line:spanner} of \Cref{alg:spanRL} is at most $N \coloneqq  \left\lceil\frac d2 \log_2\left( \frac {3600d^{7/2}}{\eta} \right)\right\rceil$.
	\item The output $((\pi_{1}, \theta_{1}),\dots,(\pi_{d}, \theta_{d}))$ of $\spanner$ has the property that for all $\pi \in \Pibar_\eta$ and $\theta \in \reals^d$, there exist $\beta_{1},\dots,\beta_d\in[-2,2]$ such that
	\begin{align}
		\nrm*{\tilde\phi_h^{\pi, \theta} - \sum_{i=1}^d \beta_i \tilde\phi_h^{\pi_i,\theta_i}} \leq \frac{\eta}{12 d^{3/2}} ,\quad {where} \quad \tilde\phih_h^{\pi', \theta}\coloneqq \Ebar^{\pi'}\left[\tilde\phih\ind{h}(\x_h,\a_h) \cdot \mathbb{I}\{\tilde\phih\ind{h}(\x_h,\a_h)^\top  \theta \geq 0\}\right],
	\end{align}
where we recall that $\tilde\phi$ is the restriction of $\phibar$ to its first $d$ coordinates, and $\phibar$ is the extension of $\phi\ind{h}$ to $\Mbar$; see \cref{sec:reach0}. 
\end{itemize}

\end{lemma}
\begin{proof}[\pfref{lem:spanner}]
By \Cref{prop:spanner}, on the event that the instances of $\psdp$ and $\veceval$ used by \spanner{} satisfy \Cref{ass:spanner} with $\psdp$ as $\apx$, $\veceval$ as $\est$, and $\veps_{\apx}= \veps_{\est} = \frac{\eta}{72 d^{5/2}}$, the two prerequisite assumptions of the lemma hold; Here, we instantiate the guarantee in \cref{prop:spanner} with $C=2$, as used by \cref{alg:spanRL}. We claim that, with probability at least $1-  \frac{\delta}{8 d N H}$, each call to $\psdp$ and to $\veceval$ satisfies \Cref{ass:spanner} with 
\begin{align}
	\cZ=\Pibarm \times \reals^d, \quad \cZ_{\texttt{ref}} =\Pibar_\eta\times \reals^d, \quad \cZ_{\supp}=\Pim\times \reals^d, \quad  \shortintertext{and} 
	\cW = \{ \Ebar^\pi[\tilde\phi\ind{h}(\x_h,\a_h) \cdot \mathbb{I}\{\tilde\phi\ind{h}(\x_h,\a_h)^\top \theta \geq 0 \} ]\mid (\pi,\theta) \in \Pibarm\times \reals^d \}.
\end{align}
Because each of $\psdp$ and $\veceval$ get called at most $4 d$ times per iteration of $\spanner$, a union bound concludes the proof contingent on this claim.

We now prove the claim. First, note that the instance of \apx{} that \spanner{} uses within \cref{alg:spanRL} is of the form:
\begin{align}
\apx{}(\theta) =  (\psdp(h, r_{1:h}(\cdot, \cdot, \theta), \cG_{1:h}, P\ind{1:h}, n_{\psdp}),\theta) \in \Pim \times \reals^d,  \label{eq:psdpinst+}
\end{align} with $r_{1:h}$ and $\cG_{1:h}$ as in \Cref{alg:spanRL}, and $P\ind{t}\coloneqq \unif(\Psi\ind{t})$ for each $t\in[h]$; this matches the form in \Cref{cor:psdp} (\psdp's guarantee) with $\phi = \phi\ind{h}$, which implies that with probability at least $1- \frac{\delta}{8 d N H}$, the output of $(\pihat,\theta/\|\theta\|)$ of $\apx(\theta/\|\theta\|)$ in \eqref{eq:psdpinst+} ($\pihat$ is the output of the \psdp{} instance in \eqref{eq:psdpinst+}) satisfies: for any $(\pi, \vartheta)\in \Pibar_\eta \times \reals^d$,
\begin{align}
& 	 \frac{\theta^\top}{\|\theta\|} \Ebar^{\pi}[\tilde\phi(\x_h,\a_h) \cdot \mathbb{I}\{\tilde\phi(\x_h,\a_h)^\top \vartheta\geq  0 \}]\\  & \leq 	 \frac{\theta^\top}{\|\theta\|} \Ebar^{\pi}[\tilde \phi(\x_h,\a_h) \cdot \mathbb{I}\{\tilde \phi(\x_h,\a_h)^\top \theta\geq 0\}], \quad (\text{since } \theta^\top \tilde\phi(\x_h,\a_h) \leq \theta^\top \tilde\phi(\x_h,\a_h) \cdot \mathbb{I}\{\theta^\top \tilde\phi(\x_h,\a_h) \geq 0\}) \nn \\
& = 	 \frac{\theta^\top}{\|\theta\|} \Ebar^{\pi}\left[\tilde \phi(\x_h,\a_h) \cdot \mathbb{I}\left\{\tilde \phi(\x_h,\a_h)^\top \frac{\theta}{\|\theta\|}\geq 0\right\}\right], \nn \\
& \leq \frac{\theta^\top}{\|\theta\|} \Ebar^{\pihat}\left[\tilde \phi(\x_h,\a_h) \cdot \mathbb{I}\left\{\tilde \phi(\x_h,\a_h)^\top \frac{\theta}{\|\theta\|}\geq 0\right\}\right] + c  H A  d \sqrt{\frac{d \cdot ({d \ln (2n_{\psdp} d^{1/2})+\ln (8 dNH{|\Phi|}/{\delta})})}{n_\psdp}}, \label{eq:psdpguar+}
\end{align}
for a sufficiently large absolute constant $c>0$. Thus, by choosing 
\begin{align}
	n_\psdp  = \cfrak \cdot \eta^{-2}  A^2  d^8  H^2 \cdot (d +\ln (|\Phi|/\delta)), \label{eq:npsdp}
\end{align}
for $\cfrak=\polylog(A,d,H,\ln(|\Phi|/\delta))$ sufficiently large and multiplying by $\|\theta\|$, the \rhs{} of \eqref{eq:psdpguar+} is bounded by $\|\theta\|\eta/(72 d^{5/2})$, which implies the claim for the invocation of \psdp{} within \spanner. Similarly, the choice of $n_{\veceval}$ in \Cref{alg:spanRL} ensures that the claim holds for the invocation of $\veceval$ within \spanner{} by \Cref{lem:veceval}.

\end{proof}

\subsection{Guarantee for \replearn{} as a Subroutine for \mainalg}
\label{sec:replearnforspanRL}
In this section, we state a guarantee for the invocation of $\replearn$ within $\mainalg$

Recall that ${P}\ind{h}= \unif(\Psi\ind{h})$ is the distribution over policies that \mainalg{} passes to \replearn{} at iteration $h\in[H-2]$ to compute feature map $\phi\ind{h}$. Thus, by invoking \cref{thm:newreplearn} in \cref{sec:replearn} and using the choice of $n_{\replearn}$ in \cref{alg:spanRL}, we immediately obtain the following corollary.
\begin{corollary}
	\label{thm:newfirstguarantee}
	Let $\delta,\eta\in(0,1)$, and $\cF$ be as in \cref{alg:spanRL}, and fix $h\in[H-2]$. Suppose that the feature class $\Phi$ satisfies \cref{assum:real}. Then, with probability at least $1-\frac{\delta}{2H}$, the instance of \replearn{} in \cref{line:oldreplearn} of \cref{alg:spanRL} runs for $t\leq   \cfrak\cdot d$ iterations for $\cfrak = \polylog(A,d,H,\log(|\Phi|/\delta))$ sufficiently large, and returns output $\phi\ind{h}$ such that for all $f\in \cF$, there exists $w_f\ind{h}\in  \cB(3d^{3/2})$ satisfying
	\begin{align}
\E^{\unif(\Psi\ind{h})}\left[\sum_{a\in\cA}\left(\phi\ind{h}(\x_h,a)^\top{w\ind{h}_f} - {\phi}_h^{\star}(\x_h,a)^\top{w_f} \right)^2\right] \leq \frac{\eta^2}{256 A^2 d^4},
	\end{align}
	where $w_f \coloneqq \int_{\cX_{h+1}} f(y) \muh(y) \dd\nu(y)$.
\end{corollary}

\subsection{Concluding the Proof of \crefzak{thm:spanrlmain}}
\label{sec:mainproof}
In this section, we conclude the proof of the main guarantee (\cref{thm:spanrlmain}). We derive the guarantee from the following inductive claim. 
\begin{theorem}
	\label{lem:spanrl}
	Consider iteration $h\in [H]$ of \mainalg{}$(\Phi,\veps,\cfrak,\delta)$ (\cref{alg:spanRL}) with parameters $\tveps,\cfrak>0$, $\delta\in(0,1)$ and a feature class $\Phi$ satisfying \cref{assum:real}. Further, let $\eta \coloneqq \veps/(4 H d^{3/2})$ and assume that:
	\begin{itemize}
		\item The collection of policies $\Psi\ind{1:h+1}$ at the start of the $h$th iteration of \mainalg{} satisfy \cref{eq:induct} and \cref{eq:induct0}.
			\item The input parameter $\cfrak =\polylog(A,d,H,\ln (|\Phi|/\delta))$ is sufficiently large.
		\end{itemize} 
	Then, with probability at least $1-\frac{\delta}{H}$, the set of policies $\Psi\ind{h+2}$ produced by \mainalg{}$(\Phi,\veps,\cfrak,\delta)$ at the end of iteration $h$ is such that $\unif(\Psi\ind{h+2})$ is an $(\frac{1}{\tC  Ad},\eta)$-randomized policy cover relative to $\Pibar_\eta$ for layer $h+2$ in $\Mbar$ (see \cref{def:approxcover}).
\end{theorem}
With this, we can now prove \cref{thm:spanrlmain}.
\begin{proof}[\pfref{thm:spanrlmain}]
	In a first step, we prove that with probability at least $1-\delta$, 	$\Psi\ind{1},\dots \Psi\ind{H}$ are $\big(\tfrac{1}{\tC  Ad},\eta\big)$-policy covers relative to $\Pibar_\eta$ for layers $1$ through $H$ in $\Mbar$; that is, we need to show \cref{eq:induct} and \cref{eq:induct0} hold for $h=H-1$ with probability at least $1-\delta$. To do this, we proceed by induction over $h=1,\dots,H-1$. The base case of $h=1$ trivially holds because $\Psi\ind{1}=\emptyset$ and $\Psi\ind{2}=\{\pi_\unif\}$. The induction step now follows by \cref{lem:spanrl} and the union bound (see \cref{lem:unionbound}). Now, \cref{lem:transfer0} together with \cref{rem:backandforth} implies that $\Psi\ind{1},\dots, \Psi\ind{H}$ are $(\frac{1}{8 Ad}, \veps)$-policy covers relative to $\Pim$ for layers 1 through $H$ in the real MDP $M$.

		We now bound the number of trajectories \cref{alg:spanRL} requires. The total number of trajectories is equal to the sum of the number of trajectories \psdp{}, \veceval{}, and \replearn{} require. We know that \replearn{} is called once at each layer. On the other hand, $\psdp$ and $\veceval$ get called at most $4 d$ times per iteration of $\spanner$, which itself runs for $N= \wtilde O(d)$ iterations (see \cref{lem:spanner}). Thus, $\psdp$ and $\veceval$ get called at most $T = \wtilde{O}(d^2)$ times per layer. Furthermore, each call to \psdp{} requires $H \cdot n_{\psdp}$ trajectories, and \veceval{} and \replearn{} require $n_{\veceval}$ and $n_{\replearn}$ trajectories, respectively. Thus, the total number of trajectories is equal to 
	\begin{align}
		& n_{\psdp} \cdot H^2    T+ n_{\veceval} \cdot H  T + n_{\replearn} \cdot H \\
		& \leq    \wtilde{O}(\veps^{-2} A^2 d^{13} H^6 \cdot (d + \log (|\Phi|/\delta))) +\wtilde{O}(\veps^{-2} H^3 d^{10} \ln (1/\delta)) +\wtilde{O}(\veps^{-2} A^2 d^{11} H^3  \ln (|\Phi|/\delta)),
	\end{align}  
	where the inequality follows by the choice of parameters in \cref{alg:spanRL}.
	This implies the desired bound on the number of trajectories.

\end{proof}

Before proving \cref{lem:spanrl}, we make the following simple observation.
\begin{lemma}
	\label{lem:negative}
	For any $\pi \in \Pibarm$, $h\in [H-1]$, any $x\in \cX_{h+1}$, we have 
	\begin{align}
		\mubar_h^\star(x)^\top  \Ebar^{\pi}[\phi_h^\star(\x_h,\a_h)]=\dbar^{\pi}(x)\geq 0.
		\end{align} 
	\end{lemma}
\begin{proof}[\pfref{lem:negative}]
	The equality follows by construction. The non-negativity of $\dbar^\pi(x)$ follows by definition of a probability density.
\end{proof}

We now prove \cref{lem:spanrl}.
\begin{proof}[\pfref{lem:spanrl}]
	Let $\cE_h$ and $\cE_h'$ denote the success events in \cref{lem:spanner} and \cref{thm:newfirstguarantee}, respectively, and note that by the union bound, we have $\P[\cE_h \cap \cE_h']\geq 1 - \delta/H$. For the rest of this proof, we will condition on $\cE\coloneqq \cE_h \cap \cE_h'$. 
	
	Throughout, we denote 
	\[\phibar_{h}^{\star,\pi} \coloneqq \Ebar^{\pi}[\phibar_h^\star(\x_h,\a_h)],\quad \forall h\in[H], \forall \pi\in \Pibarm.\] 
Because $\Psi\ind{1:h+1}$ satisfy \cref{eq:induct} (i.e., are a policy cover) it holds that for all $\ell\in[h+1]$ and $x\in\cX_{\ell,\eta}(\Pibar_\eta)$, 
\begin{align}
	\E_{\pi \sim\unif(\Psi\ind{\ell})} \left[\dbar^{\pi}(x)\right] =	\E_{\pi \sim \unif(\Psi\ind{\ell})}\left[ \mubar^\star_\ell(x)^\top \phibar_{\ell-1}^{\star,\pi}\right] \geq  \alpha \cdot \sup_{\pi \in \Pibar_\eta}\dbar^{\pi}(x) =\alpha \cdot \sup_{\pi \in \Pibar_\eta}	 \mubar^\star_\ell(x)^\top \phibar_{\ell-1}^{\star,\pi},\quad \text{for } \alpha\coloneqq \frac{1}{\tC  A d},\label{eq:invar0}
\end{align}
where the equalities follow by \cref{lem:negative}. We will show that with probability at least $1-\frac{\delta}{H}$, the policy set $\Psi\ind{h+2}$ has the same property for layer $h+2$; that is, for all $x\in\cX_{h+2,\eta}(\Pibar_\eta)$,
\begin{align}
	\E_{\pi \sim \Psi\ind{h+2}} [ \mubar^\star_{h+2}(x)^\top \phibar_{h+1}^{\star,\pi}] \geq \alpha\cdot  \sup_{\pi \in \Pibar_\eta}		 \mubar^\star_{h+2}(x)^\top \phibar_{h+1}^{\star,\pi}.\label{eq:invar20}
\end{align}
Again, by \cref{lem:negative} this is equivalent to the statement that $\Psi\ind{h+2}$ is an $(\frac{1}{\tC  Ad},\eta)$-policy cover relative to $\Pibar_\eta$ for layer $h+2$ in $\Mbar$.

Throughout the proof, for any $\ell\in [2\ldots H]$ and $z\in \cX_\ell$, we define 
\begin{align}
\pi_z \in \argmax_{\pi \in \Pibar_\eta} \dbar^\pi(z), \label{eq:newrem0}
	\end{align}
 and note that by \cref{lem:negative}, we have 
 \begin{align}\pi_z \in \argmax_{\pi \in \Pibar_\eta}\mubar_{h+2}(z)^\top \phibar_{h+1}^{\star,\pi}.
 	\label{eq:rem0}
 \end{align}
Fix $x\in \cX_{h+2,\eta}(\Pibar_\eta)$. In the remainder of the proof, we will argue that $\Psi\ind{h+2}$ satisfies the coverage property in \cref{eq:invar20} for $x$.

\paragraph{Preliminaries}  We begin with some notation. We introduce a function $f: \cXbar_{h+1}\rightarrow \reals$ such that
\begin{align}
 f(y)\coloneqq \thetabar_x^\top  \phibar^{\star}_{h+1}(y,\pi_x(y)), \quad   \text{where} \quad  \thetabar_x \coloneqq [\theta_x^\top,0]^\top \quad \text{and}\quad  \theta_x  \coloneqq \frac{\muh[h+2](x)}{\nrm{\muh[h+2](x)}}. \label{eq:thef0}
	\end{align}
Note that $\nrm{\muh[h+2](x)}>0$, since $x\in \cX_{h+2,\eta}(\Pibar_\eta)$. Next, we define 
\begin{align}
 w_x \coloneqq \int_{\cX_{h+1}} f(y) \muh(y) \dd \nu(y), \quad \text{and} \quad \bar{w}_x\coloneqq [w_x^\top, 0]^\top \in \reals^{d+1} \label{eq:wx} \end{align}
By definition of $\pi_x$, we have that for all $y\in \cX_{h+1}$, 
\begin{align} \thetabar_x^\top\bar{\phi}^{\star}_{h+1}(y,\pi_x(y))& = \max_{a\in \cAbar} \thetabar_x^\top\bar{\phi}^{\star}_{h+1}(y,a),\nn \\
	& \leq \max_{a\in \cA} \thetabar_x^\top\bar{\phi}^{\star}_{h+1}(y,a), \quad \text{(justified below)} \label{eq:nonnegative0}\\
	& = \max_{a\in \cA} \theta_x^\top{\phi}^{\star}_{h+1}(y,a), \quad \text{(since $y\not=\tfrak_{h+1}$ and $[\bar{\theta}_x]_{d+1}=0$)} \label{eq:cravit0}
\end{align}
where \eqref{eq:nonnegative0} follows by the facts that $\thetabar_x^\top\bar{\phi}^{\star}_{h+1}(y,\afrak)=0$ (since $\bar{\phi}^{\star}_{h+1}(\cdot,\afrak)\equiv e_{d+1}$ and $[\thetabar_x]_{d+1}=0$) and that
\begin{align}
	\forall a\in \cA,\ \  \thetabar_x^\top\bar{\phi}^{\star}_{h+1}(y,a) & = \theta_x^\top{\phi}^{\star}_{h+1}(y,a), \quad (\text{since $y\not=\tfrak_{h+1}$}) \nn \\ &= \frac{\muh[h+2](x)^\top \phi_{h+1}^{\star}(y,a)}{\|\muh[h+2](x)\|},\nn \\& \geq 0. \quad \text{($\muh[h+2](\cdot)^\top \phi_{h+1}^{\star}(y,a)$ is a conditional law)}
\end{align}
\pref{eq:cravit0} and the fact that $\nrm{\theta_x}=1$ implies that %
\begin{align}
	f_x|_{\cX_{h+1}} \in \cF, \label{eq:newinclass0}	
\end{align}
where $f_x|_{\cX_{h+1}}$ denotes the restriction of $f_x$ to $\cX_{h+1}$. We also note that since $x\in \cX_{h+2,\eta}(\Pibar_\eta)$, we have 
\begin{align}\wwbar_{x}^\top\bar{\phi}_h^{\star, \pi_x} & =  \left[  \int_{\cX_{h+1}} f_x(y) \muh(y)^\top \dd \nu(y), \ 0\right] \bar{\phi}_h^{\star, \pi_x}, \quad \text{(by definition of $\bar{w}_x$ in \eqref{eq:wx})}\nn \\  
	&=	\int_{\cX_{h+1}} f_x(y) \muhb(y)^\top \bar{\phi}_h^{\star, \pi_x} \dd \nu(y), \quad \text{(since  $\muhb(y)=[\muh(y)^\top, 0]$, for all $y\neq \tfrak_{h+1}$)}  \nn \\
	&=	\int_{\wb{\cX}_{h+1}} f_x(y) \muhb(y)^\top \bar{\phi}_h^{\star, \pi_x} \dd \nubar(y), \quad \text{(since $f_x(\tfrak_{h+1})=0$)}  \nn \\ 
	& =\thetabar_x^\top \bar{\phi}_{h+1}^{\star,\pi_x}, \quad \text{(by definition of $f_x$ in \eqref{eq:thef0})} \label{eq:prereach} \\ 
	& = \frac{1}{\nrm*{\muh[h+2](x)}} \max_{\pi\in \Pibar_\eta}\muh[h+2](x)^\top \tilde {\phi}_{h+1}^{\star,\pi}, \quad (\text{by definition of $\bar{\theta}_x$ in \eqref{eq:thef0}})\label{eq:pen0} \\
	& \geq \eta>0,
	\label{eq:newlower00}
\end{align} 
where \eqref{eq:pen0} uses the definition of reachable states $\cX_{h+2,\eta}(\Pibar_\eta)$ (see \cref{def:reachable}); we recall (see \cref{sec:reach0}) that $\tilde\phi^{\star,\pi}_{h} \coloneqq \Ebar^{\pi}[\tilde {\phi}^{\star}_h(\x_h, \a_h)]$ and $\tilde \phi^{\star}_h$ represents the restriction of $\bar{\phi}^{\star}_h$ to its first $d$ coordinates.

Now, let $(x',a)\in \cX_{h}\times \cA$ be given. By repeating the steps that lead to \eqref{eq:prereach}, but with $\phibar^{\star,\pi_x}_h$ replaced with $\phibar^\star_h(x',a)$ we get that: 
\begin{align}
	\wwbar_{x}^\top\bar{\phi}_h^{\star}(x',a) &= \thetabar_x^\top \Ebar[\bar{\phi}_{h+1}^{\star}(\x_{h+1},\a_{h+1})\mid\x_h=x',\a_h=a], \nn \\ 
	& = \frac{1}{\nrm*{\muh[h+2](x)}} \muh[h+2](x)^\top\Ebar[\tilde{\phi}_{h+1}^{\star}(\x_{h+1},\a_{h+1})\mid\x_h=x',\a_h=a], \quad (\text{by definition of $\bar{\theta}_x$ in \eqref{eq:thef0}}) \\
	& = \frac{1}{\nrm*{\muh[h+2](x)}} \E[\muh[h+2](x)^\top  {\phi}_{h+1}^{\star}(\x_{h+1},\a_{h+1})\mid\x_h=x',\a_h=a], \nn \quad \text{(since $x'\neq \tfrak_h$ and $a\neq \afrak$)}\\
	&  \geq 0, \label{eq:ineq}
	\end{align}
where the last inequality follows by the fact that $\muh[h+2](x)^\top  {\phi}_{h+1}^{\star}(\tilde x,\tilde a) \geq 0$, for all $(\tilde x, \tilde a)\in \cX_{h+1}\times \cA$.
\paragraph{Applying the guarantee for \replearn}
Moving forward, we let $\phi\ind{h}$ be the feature map returned by \replearn{} within \mainalg{} (\cref{alg:spanRL}) at iteration $h$, and define $\bar{\phi}^{\pi}_{h} \coloneqq \Ebar^{\pi}[\phibar\ind{h}(\x_h,\a_h)]$, for any $\pi\in \Pibarm$, where we recall that $\phibar\ind{h}$ is the extension of $\phi\ind{h}$ to $\Mbar$; see \cref{sec:reach0}. Further, let $w\ind{h}_x$ be the vector $w\ind{h}_f$ in \cref{thm:newfirstguarantee} with $f=f_x|_{\cX_{h+1}}$, and note that 
\begin{align}
	\|w_x\ind{h}\| \leq{3}d^{3/2}. \label{eq:normbound0}
\end{align}
We will use the extended vector $\bar{w}_x\ind{h}\coloneqq [(w_x\ind{h})^\top,0]^\top \in \reals^{d+1}$. Further, for $\pi \in \Pibarm$ let \begin{align} 
	\bar{\varphi}\ind{h}(x,a) \coloneqq	\phibar\ind{h}(x, a)\cdot \mathbb{I}\{\phibar\ind{h}(x, a)^\top \bar{w}_x\ind{h}\geq 0\} \quad \text{and} \quad 
	\bar\phi_h^{\pi, \bar{w}_x\ind{h}} \coloneqq  \Ebar^{\pi}[\bar\varphi\ind{h}(\x_h, \a_h)]. 
\end{align}  
By Jensen's inequality, we have
\begin{align}	&\left( \inner{\wwbar\ind{h}_x}{ \bar\phi_h^{\pi_x, \wwbar\ind{h}_x}}-  \inner{\wwbar_x}{\bar{\phi}_h^{\star, \pi_x}} \right)^2\nn \\ & \leq \Ebar^{\pi_x}\left[\left(\bar\varphi\ind{h}(\x_h,\a_h)^\top{\wwbar\ind{h}_x} - \bar{\phi}_h^{\star}(\x_h,\a_h)^\top{\wwbar_x} \right)^2\right],\nn \\
	& = \Ebar^{\pi_x}\left[\left(\bar\varphi\ind{h}(\x_h,\pi_x(\x_h))^\top{\wwbar\ind{h}_x} - \bar{\phi}_h^{\star}(\x_h,\pi_x(\x_h))^\top{\wwbar_x} \right)^2\right], \nn \\
	& = \Ebar^{\pi_x}\left[\mathbb{I}\{\x_h \in \cX_{h,\eta}(\Pibar_{\eta})\}\cdot \left(\bar\varphi\ind{h}(\x_h,\pi_x(\x_h))^\top{\wwbar\ind{h}_x} - \bar{\phi}_h^{\star}(\x_h,\pi_x(\x_h))^\top{\wwbar_x} \right)^2\right],\\
	& \leq \Ebar^{\pi_x}\left[\mathbb{I}\{\x_h \in \cX_{h,\eta}(\Pibar_{\eta})\}\cdot \sum_{a\in\cA}\left(\bar\varphi\ind{h}(\x_h,a)^\top{\wwbar\ind{h}_x} - \bar{\phi}_h^{\star}(\x_h,a)^\top{\wwbar_x} \right)^2\right],  \label{eq:newjustt0} 
\end{align}
where the last equality follows by the fact that $\phibar\ind{h}(\cdot,\afrak)\equiv\bar{\phi}^{\star}_h(\cdot,\afrak) \equiv e_{d+1}$ and $[\bar{w}_x\ind{h}]_{d+1}=[\bar{w}_x]_{d+1}=0$ (by definition). Thus, for $g(y) \coloneqq \mathbb{I}\{y\in \cX_{h,\eta}(\Pibar_{\eta})\}\cdot \sum_{a\in\cA}\big(\bar{\varphi}\ind{h}(y,a)^\top \wwbar_x\ind{h}  - \bar{\phi}_h^\star(y,a)^\top \wwbar_x \big)^2$,
	 \eqref{eq:newjustt0} implies that
\begin{align}
	&\left( \inner{\wwbar\ind{h}_x}{\bar{\phi}_h^{\pi_x, \bar{w}_x\ind{h}} }-  \inner{\wwbar_x}{\bar{\phi}_h^{\star, \pi_x}} \right)^2\nn \\ 
	& \leq \int_{\wb{\cX}_{h}}  g(y)\muhb[h](y)^\top \bar{\phi}^{\star,\pi_x}_{h-1}\dd \nubar(y),\nn \\
	& \leq  \int_{\wb{\cX}_{h}}  g(y)\muhb[h](y)^\top \bar{\phi}^{\star,\pi_y}_{h-1}\dd \nubar(y),\quad \text{(by definition of $\pi_y$ \eqref{eq:newrem0} and \eqref{eq:rem0})}\nn\\
	& \leq \alpha^{-1}\max_{\pi \in \Psi\ind{h}}\left[ \int_{\wb{\cX}_{h}}  g(y)\muhb[h](y)^\top \bar{\phi}^{\star,\pi}_{h-1}\dd \nubar(y)\right],  \quad \text{(by \eqref{eq:invar0} with $\ell=h$, and $g(y)=0$ for all $y\not\in \cX_{h,\eta}(\Pibar_{\eta})$)}\nn \\ 
		& \leq \alpha^{-1}\sum_{\pi \in \Psi\ind{h}} \int_{\wb{\cX}_{h}}  g(y)\muhb[h](y)^\top \bar{\phi}^{\star,\pi}_{h-1}\dd \nubar(y), \quad \text{(by \cref{lem:negative})}  \nn \\ 
	& =  \alpha^{-1}\sum_{\pi \in \Psi\ind{h}}\Ebar^{\pi}\left[\sum_{a\in\cA}\left(\bar\varphi\ind{h}(\x_h,a)^\top{\wwbar\ind{h}_x} - \bar{\phi}_h^{\star}(\x_h,a)^\top{\wwbar_x} \right)^2\right],
	\nn \\
	& =  \alpha^{-1}\sum_{\pi \in \Psi\ind{h}}\Ebar^{\pi}\left[\sum_{a\in\cA}\left(\max(0,\phibar\ind{h}(\x_h,a)^\top{\wwbar\ind{h}_x}) - \bar{\phi}_h^{\star}(\x_h,a)^\top{\wwbar_x} \right)^2\right], \quad \text{(by definition of $\bar{\varphi}\ind{h}$)} \label{eq:indication}
	\nn \\
	& \leq   \alpha^{-1}\sum_{\pi \in \Psi\ind{h}}\Ebar^{\pi}\left[\sum_{a\in\cA}\left(\phibar\ind{h}(\x_h,a)^\top{\wwbar\ind{h}_x} - \bar{\phi}_h^{\star}(\x_h,a)^\top{\wwbar_x} \right)^2\right], \quad \text{(since $\bar{\phi}_h^{\star}(\cdot,\cdot)^\top{\wwbar_x} \geq 0$ by \eqref{eq:ineq})}\nn \\
	& =  \alpha^{-1}\sum_{\pi \in \Psi\ind{h}}\E^{\pi}\left[\sum_{a\in\cA}\left(\phi\ind{h}(\x_h,a)^\top{w\ind{h}_x} - {\phi}_h^{\star}(\x_h,a)^\top{w_x} \right)^2\right],	 \label{eq:newhit00}
\end{align}
where \eqref{eq:newhit00} follows by the fact that the policies in $\Psi\ind{h}$ never take the terminal action (by assumption) and that $\phibar\ind{h}(x,a)^\top{\wwbar\ind{h}_x} - \bar{\phi}_h^{\star}(x,a)^\top{\wwbar_x}=\phi\ind{h}(x,a)^\top{w\ind{h}_x} - {\phi}_h^{\star}(x,a)^\top{w_x}$ for all $a\in\cA$ whenever $x\neq\term_h$. On the other hand, by Jensen's inequality, we have 
\begin{align}
	\forall \pi' \in \Psi\ind{h},	\quad	&\left( \inner{\wwbar\ind{h}_x}{ \bar\phi_h^{\pi', \wwbar\ind{h}_x}}-  \inner{\wwbar_x}{\bar{\phi}_h^{\star, \pi'}} \right)^2\nn \\ & \leq \Ebar^{\pi'}\left[\left(\bar\varphi\ind{h}(\x_h,\a_h)^\top{\wwbar\ind{h}_x} - \bar{\phi}_h^{\star}(\x_h,\a_h)^\top{\wwbar_x} \right)^2\right],\nn \\
	& \leq \sum_{\pi \in \Psi\ind{h}}\Ebar^{\pi}\left[\left(\bar\varphi\ind{h}(\x_h,\a_h)^\top{\wwbar\ind{h}_x} - \bar{\phi}_h^{\star}(\x_h,\a_h)^\top{\wwbar_x} \right)^2\right],\nn \\
	& =  \sum_{\pi \in \Psi\ind{h}}\Ebar^{\pi}\left[\sum_{a\in\cA}\left(\max(0,\phibar\ind{h}(\x_h,a)^\top{\wwbar\ind{h}_x}) - \bar{\phi}_h^{\star}(\x_h,a)^\top{\wwbar_x} \right)^2\right], \quad \text{(by definition of $\bar{\varphi}\ind{h}$)} 
	\nn \\
	& \leq \sum_{\pi \in \Psi\ind{h}}\Ebar^{\pi}\left[\sum_{a\in\cA}\left(\phibar\ind{h}(\x_h,a)^\top{\wwbar\ind{h}_x} - \bar{\phi}_h^{\star}(\x_h,a)^\top{\wwbar_x} \right)^2\right], \quad \text{(since $\bar{\phi}_h^{\star}(\cdot,\cdot)^\top{\wwbar_x} \geq 0$)}\nn \\
	& =  \sum_{\pi \in \Psi\ind{h}}\E^{\pi}\left[\sum_{a\in\cA}\left(\phi\ind{h}(\x_h,a)^\top{w\ind{h}_x} - {\phi}_h^{\star}(\x_h,a)^\top{w_x} \right)^2\right].	 \label{eq:newjustt00} 
\end{align}
We note that $\unif({\Psi\ind{h}})$ is the distribution over policies that \mainalg{} passes to \replearn{} to compute $\phi\ind{h}$. Thus, since $w_x = \int_{\cX_{h+1}} f_x(y) \muh(y) \dd \nu(y)$ (see \eqref{eq:wx}) and $f_x|_{\cX_{h+1}}\in \cF$ (see \eqref{eq:newinclass0}), the guarantee for \replearn{} in \cref{thm:newfirstguarantee} together with \eqref{eq:newhit00} and \eqref{eq:newjustt00}, implies that (recall that we condition on the event $\cE$) 
\begin{align}
	\left|\inner{\wwbar\ind{h}_x}{\bar{\phi}_h^{\pi_x, \bar{w}_x\ind{h}} }-  \inner{\wwbar_x}{\bar{\phi}_h^{\star, \pi_x}}\right| \leq   \frac{\eta}{8d} \quad \text{and}\quad \forall \pi \in \Psi\ind{h}, \	\left|\inner{\wwbar\ind{h}_x}{\bar{\phi}_h^{\pi, \bar{w}_x\ind{h}} }-  \inner{\wwbar_x}{\bar{\phi}_h^{\star, \pi}}\right| \leq   \frac{\eta}{8d}. \label{eq:trianglewith0}
\end{align}

\paragraph{Applying the guarantee for $\spanner$}
Letting $((\pi_{1},\theta_{1}),\dots, (\pi_d,\theta_d))$ be the policies returned by \spanner{} at iteration $h$ of  $\mainalg$, the guarantee of \spanner{} in \cref{lem:spanner} (instantiated with $\theta = \wwbar\ind{h}_x$) implies that there exist $\beta_1, \dots, \beta_d\in[-2,2]$ such that
\begin{align}
   \nrm*{\phibar^{\pi_x, \wwbar\ind{h}_x}_h-\sum_{i=1}^d \beta _i\phibar_h^{\pi_i, \theta_i}} \leq  \frac{\eta}{12 d^{3/2}}. \label{eq:spannergar}
\end{align}
 Combining \eqref{eq:spannergar} with \eqref{eq:trianglewith0} and using the triangle inequality, we get that
\begin{align}
		\bar{w}_x^\top\phibar_h^{\star, \pi_x} &\leq  \langle\bar{w}_x\ind{h},\phibar_h^{\pi_x, \wwbar\ind{h}_x}\rangle  +\frac{\eta}{8d}, \quad \text{(by \eqref{eq:trianglewith0})} \nn \\
	& \leq  \sum_{i=1}^d \beta_i \langle\bar{w}\ind{h}_x,\phibar_h^{\pi_i,\theta_i}\rangle +   \|\bar{w}\ind{h}_x\| \cdot \frac{\eta}{12 d^{3/2}} +\frac{\eta}{8d},\nn \quad \text{(by \eqref{eq:spannergar})}\\
	& =  \sum_{i=1}^d \beta_i \Ebar^{\pi_i}\left[ \langle\bar{w}\ind{h}_x,\phibar\ind{h}(\x_h,\a_h)\rangle \cdot \mathbb{I}\{ \phibar\ind{h}(\x_h,\a_h)^\top\theta_i \geq 0\}  \right] +   \|\bar{w}\ind{h}_x\| \cdot \frac{\eta}{12 d^{3/2}} +\frac{\eta}{8d},\nn \\
& \leq   \sum_{i=1}^d \beta_i \Ebar^{\pi_i}\left[ \langle\bar{w}\ind{h}_x,\phibar\ind{h}(\x_h,\a_h)\rangle \cdot \mathbb{I}\{ \phibar\ind{h}(\x_h,\a_h)^\top\wwbar\ind{h}_x \geq 0\}  \right] +   \|\bar{w}\ind{h}_x\| \cdot \frac{\eta}{12 d^{3/2}} +\frac{\eta}{8d}, \quad \text{(see below)}\nn \\
			&=   \sum_{i=1}^d \beta_i \inner{\bar{w}\ind{h}_x}{\phibar_h^{\pi_i, \wwbar_x\ind{h}}} +   \|\bar{w}\ind{h}_x\| \cdot \frac{\eta}{12 d^{3/2}} +\frac{\eta}{8d}, \label{eq:mile}
			\end{align}
	where the last inequality follows by the fact that $y \cdot \mathbb{I}\{z \geq 0\} \leq y \cdot \mathbb{I}\{y \geq 0\}$ for any $y,z\in \reals$.
Continuing from \eqref{eq:mile} and using \eqref{eq:trianglewith0} once again, we obtain 
\begin{align}
	\bar{w}_x^\top\phibar_h^{\star, \pi_x}	&\leq   \sum_{i=1}^d \beta_i \bar{w}_x^\top\phibar_h^{\star, \pi_i} +   \|\bar{w}\ind{h}_x\| \cdot \frac{\eta}{12 d^{3/2}} +\frac{\eta}{8d} + \sum_{i=1}^d \beta_i \frac{\eta}{8d},\nn \\
	&\leq   \sum_{i=1}^d \beta_i \bar{w}_x^\top\phibar_h^{\star, \pi_i} +  \frac{\eta}{8d}+\frac{\eta}{8d}+  \sum_{i=1}^d \beta_i \frac{\eta}{8d},\quad \text{(by \eqref{eq:normbound0})}\nn \\
	&\leq   2 \sum_{i=1}^d \bar{w}_x^\top\phibar_h^{\star, \pi_i} + \frac{\eta}{4} +   \frac{ \eta}{4},\nn \\
	& =    2  \sum_{i=1}^d  \bar{w}_x^\top\phibar_h^{\star, \pi_i} + \frac{\eta}{2}.
		\label{eq:upper0}
\end{align}
Combining this with \eqref{eq:newlower00} and rearranging implies
\begin{align}
	\bar{w}_x^\top\phibar_h^{\star, \pi_x} \leq   4 \sum_{i\in[d]}  \bar{w}_x^\top\phibar_h^{\star, \pi_i}. \label{eq:still}
\end{align}
On the other hand, since $x\in \cX_{h+2,\eta}(\Pibar_\eta)$, we have 
\begin{align}
\sum_{i\in[d]}	\wwbar_{x}^\top\bar{\phi}_h^{\star, \pi_i} & = \sum_{i\in [d]} \left[  \int_{\cX_{h+1}} f_x(y) \muh(y)^\top \dd \nu(y), \ 0\right] \bar{\phi}_h^{\star, \pi_i}, \quad \text{(by definition of $\bar{w}_x$ in \eqref{eq:wx})}\nn \\  
	&=\sum_{i\in [d]}	\int_{\cX_{h+1}} f_x(y) \muhb(y)^\top \bar{\phi}_h^{\star, \pi_i} \dd \nu(y), \quad \text{(since  $\muhb(y)=[\muh(y)^\top, 0]$, for all $y\neq \tfrak_{h+1}$)}  \nn \\
	&= \sum_{i\in [d]}	\int_{\wb{\cX}_{h+1}} f_x(y) \muhb(y)^\top \bar{\phi}_h^{\star, \pi_i} \dd \nubar(y), \quad \text{(since $f_x(\tfrak_{h+1})=0$)}  \nn \\ 
	& =\sum_{i\in [d]}\thetabar_x^\top \bar{\phi}_{h+1}^{\star,\pi_i \circ_{h+1} \pi_x}, \quad \text{(by definition of $f_x$ in \eqref{eq:thef0})}\nn \\ 
	& = \frac{1}{\nrm*{\muh[h+2](x)}}	\sum_{i\in [d]}\Ebar^{\pi_i \circ_{h+1} \pi_x} \left[\mubar^\star_{h+2}(x)^\top \phibar_{h+1}^\star(\x_{h+1},\a_{h+1})\right], \nn \\
	& \leq \frac{A}{\nrm*{\mu^\star_{h+2}(x)}}	\sum_{i\in [d]} \Ebar^{\pi_i \circ_{h+1} \pi_\unif} \left[\muh[h+2](x)^\top \phibar_{h+1}^\star(\x_{h+1},\a_{h+1})\right], \quad \text{(see below)} \label{eq:just}\\
	& =  \frac{A d}{\nrm*{\muh[h+2](x)}}  \E_{\pi \sim \unif(\Psi\ind{h+2})} \left[\mubar_{h+1}^\star(x)^\top \phibar_{h+1}^{\star, \pi}\right], \label{eq:her}
\end{align} 
where the inequality follows from the non-negativity of $\inprod{\mubar^\star_{h+1}(\cdot)}{\phibar^\star_{h+1}(x,a)}$, for all $(x,a)\in \cX_h\times \cA$ (due to \Cref{lem:negative}), and \eqref{eq:her} follows from the definition of $\Psi\ind{h+2}$ in \Cref{line:cover} of \Cref{alg:spanRL}. Combining  \eqref{eq:still} and \eqref{eq:her} then implies that 
\begin{align}
\frac{1}{\nrm*{\muh[h+2](x)}}  \mubar^\star_{h+2}(x)^\top \phibar_{h+1}^{\star, \pi_x} & =\frac{1}{\nrm*{\muh[h+2](x)}}  \mu^\star_{h+2}(x)^\top \tilde\phi_{h+1}^{\star, \pi_x}, \quad (\text{since } x\neq \tfrak_{h+1}) \nn \\  & = \bar{w}_x^\top\phibar_h^{\star, \pi_x}, \quad (\text{by \eqref{eq:pen0}}\nn \\ & \leq   4 \sum_{i\in[d]}  \bar{w}_x^\top\phibar_h^{\star, \pi_i},\quad \text{(by \eqref{eq:still})} \nn \\
& \leq \frac{4 A d}{\nrm*{\mubar^\star_{h+2}(x)}}  \E_{\pi \sim \unif(\Psi\ind{h+2})}\left[ \mubar^\star_{h+2}(x)^\top \phibar_{h+1}^{\star, \pi}\right]. \quad \text{(by \eqref{eq:her})}
	\end{align}
This, together with \cref{lem:negative}, implies that \eqref{eq:invar20} holds. Since this argument holds uniformly for all $x\in\cX_{h+2}$, this completes the proof.
\end{proof}

\subsection{Proof of \crefzak{lem:barycentricspannerknownphi}}\label{ssec:pf_baryspanner_knownphi}
Fix $x \in \cX_{h+1}$ such that 
\begin{align}
\max_{\pi\in \Pim} d^{\pi}(x) \geq 2\veps \cdot \|\mu^\star_{h+1}(x)\|. \label{eq:threshold}
\end{align}
By definition, we have $d^\pi(x) = \ee^\pi\left[ {\muh(x)}^\top{\phistarh[h](\x_h, \a_h)} \right]$. Let $\pi_x$ denote the policy maximizing $d^\pi(x)$ (if no such maximizer exists, we may pass to a maximizing sequence) and let $\Psi = \left\{ \pi_1, \dots, \pi_d \right\}$. Then, we have for some $\beta_1, \dots, \beta_d \in [-C, C]$,
	 \begin{align}
		 d^{\pi_x}(x) &= {\muh(x)}^\top{\left(\sum_{i = 1}^d \beta_i \phistarhpi[\pi_i]\right)} + {\muh(x)}^\top{\left( \phistarhpi[\pi_x] -  \sum_{i = 1}^d \beta_i\phistarhpi[\pi_i]\right)}, \\
		 &\leq C  d \cdot \max_{i \in[d]}
		   {\muh(x)}^\top{\phistarhpi[\pi_i]} + \veps\cdot \norm{\muh(x)}
		   , \quad \text{(Cauchy-Schwarz)}\\
				   &\leq C  d \cdot \max_{i \in[d]} {\muh(x)}^\top{\phistarhpi[\pi_i]} + \frac{1}{2}d^{\pi_x}(x), 
	 \end{align}
	 where the inequality follows by \eqref{eq:threshold}. The result now
	 follows by rearranging.

	\section{Application to Reward-Based RL}
	\label{sec:reward_based}

In this section, we show how the output $\Psi\ind{1:H}$ of \mainalg{} (\cref{alg:spanRL}), which is a $(\frac{\eta^3}{\cfrak \cdot d^6 A^2}, \veps)$-policy cover for $\eta = \veps/(4 H d^{3/2})$ and $\cfrak= \polylog(A,H,d, \log(|\Phi|/\delta))$ sufficiently large (see \cref{thm:spanrlmain}), can be used to optimize downstream reward functions $r_{1:H}$. One way to optimize the sum of rewards $S_H \coloneqq \sum_{h=1}^H r_h$ is by first generating trajectories using policies in $P\ind{1:H}$, then applying an offline RL algorithm, e.g.~Fitted Q-Iteration (\texttt{FQI}) \citep{ernst2005tree}, to optimize $S_H$. It is also possible to use \psdp{} with the randomized policy covers $\unif(\Psi\ind{1}), \dots, \unif(\Psi\ind{H})$ to achieve the same goal. We will showcase the latter approach, since we can make use of the guarantees for \psdp{} given in \cref{sec:nonnegative}.

As in \cref{sec:nonnegative}, we assume access to a function class $\cG_{1:H}$, where $\cG_h \subseteq \{g: \cX_h\times \cA\rightarrow \reals\}$ for each $h\in[H]$, that realize the rewards $r_{1:H}$ in the following sense: for all $h\in[H]$ and all $\pi\in \Pim^{h+1:H}$,
\begin{align}
	Q_h^{\pi}\in \cG_h, \quad \text{where}\quad 	Q^{\pi}_h(x,a)\coloneqq r_h(x,a)+\E^{\pi}\left[\left.\sum_{t=h+1}^{H} r_t(\x_t,\a_t)\ \right| \ \x_h=x,\a_h=a\right]. \label{eq:real2}
\end{align}
Note that when the reward functions $r_{1:H}$ are linear in the feature map $\phistarh$; that is, when for all $h\in[H]$ and $(x,a)\in \cX_h\times \cA$, \begin{align}
	r_h(x,a)=\theta_h^\top \phistarh(x,a)
	\label{eq:rewards+}
\end{align} for some $\theta_h\in \cB(1)$ (this is a common assumption in the context of RL in Low-Rank MDPs \citep{misra2019kinematic,mhammedi2023representation,zhang2022efficient,modi2021model}), then the function classes $\cG_{1:H}$, where
\begin{align}
\forall h\in[H],\quad \cG_h = 	\cG \coloneqq \{g:(x,a)\mapsto \phi(x,a)^\top w \mid \phi \in \Phi , w \in \cB(2H\sqrt{d})\}, \label{eq:linear}
\end{align}
realize $r_{1:H}$. We show this claim next.
\begin{lemma}
	\label{lem:realizefunc+}
	Under \cref{assum:real}, the function classes $\cG_{1:H}$ in \eqref{eq:linear} realize the reward functions in \eqref{eq:rewards+}. Furthermore, the functions in $\cG_{1:H}$ are uniformly bounded by $2\sqrt{d}H$, and $\ln \cN_{\cG_h}(\veps)\leq \ln |\Phi|+ d \ln (2\sqrt{d}H /\veps)$, for all $h\in[H]$, where we recall that $\cN_{\cG}(\veps)$ denotes the $\veps$-covering number of $\cG$ in $\ell_\infty$-distance (see \cref{def:covering}).
\end{lemma} 
\begin{proof}[\pfref{lem:realizefunc+}]
 For $h=H$, we clearly have that for any $\pi \in \Pim^{H:H}$, $Q^{\pi}_H(\cdot,\cdot)=r_H(\cdot,\cdot)\in \cG_H$. For $h<H$ and $\pi \in \Pim^{h+1:H}$, we have, by the low-rank MDP structure and the expression of the rewards in \eqref{eq:rewards+}, that 
	\begin{align}
		Q^{\pi}_h(x,a) & =r_h(\x_h,\a_h)+\int_{\cX_{h+1}} \E^{\pi}\left[\left.\sum_{t=h+1}^{H} r_t(\x_t,\a_t)\, \right| \,\x_{h+1}=y,\a_{h+1}=\pi(y)\right] \cdot \phi^\star_h(x,a)^\top  \mu_{h+1}^\star(y)  \dd\nu (y), \nn \\
		& =  \phi^\star_h(x,a)^\top  \left( \theta_h + \int_{\cX_{h+1}} \E^{\pi}\left[\left.\sum_{t=h+1}^{H} r_t(\x_t,\a_t)\, \right| \, \x_{h+1}=y,\a_{h+1}=\pi(y)\right] \cdot \mu_{h+1}^\star(y)  \dd\nu (y)\right).\label{eq:hand+}
	\end{align}
	Now, by the fact that $\E^{\pi}\left[\sum_{t=h+1}^{H} r_t(\x_t,\a_t)\mid \x_{h+1}=y,\a_{h+1}=\pi(y)\right] \in [-H-h,H-h]$, for all $y\in \cX_{h+1}$ (since the rewards take values between $-1$ and 1 thanks to $\phi(\cdot,\cdot),\theta_h\in \cB(1)$, for all $h\in[H]$), and the normalizing assumption made on $(\muh[h])_{h\in[H]}$ in \cref{sec:onlineRL} (i.e.~that for all $g:\cX_{h+1}\to\brk{0,1}$, $\nrm*{\int_{\cX_{h+1}} \muh[h+1](y)g(y) \dd\nu(y)} \leq \sqrt{d}$), we have that 
	\begin{align}
		w_h \coloneqq  \theta_h+\int_{\cX_{h+1}} \E^{\pi}\left[\left.\sum_{t=h+1}^{H} r_t(\x_t,\a_t)\, \right| \, \x_{h+1}=y,\a_{h+1}=\pi(y)\right] \cdot \mu_{h+1}^\star(y)  \dd\nu (y) \in \cB(2H\sqrt{d}). \label{eq:one}
	\end{align}
	This, together with \eqref{eq:hand+} and the fact that $\phistarh[h]\in \Phi$ (by \cref{assum:real}), implies that that $Q_h^{\pi} \in \cG_h$. The bound on the covering number $\cN_{\cG_h}(\veps)$, follows from a standard bound on the covering number of the ball $\cB(2H\sqrt{d})$ \citep{wainwright2019high}.
\end{proof}

Combining \Cref{lem:realizefunc+} with \Cref{thm:psdp+} and \cref{rem:backandforth} results in the following guarantee for \psdp{}.%
\begin{corollary}
	\label{cor:psdp+}
	Let $\alpha,\tveps,\delta \in(0,1)$ be given and fix $h\in[H]$. Let $\pihat$ be the output of $\psdp$ when given input $(H, r_{1:H}, \cG_{1:H}, P\ind{1:H}, n)$, where
	\begin{itemize}
		\item The reward functions $r_{1:H}$ are as in \eqref{eq:rewards+}, with $\theta_{1:H}\in \cB(1)$;
		\item The function classes $\cG_{1:H}$ are as in \eqref{eq:linear};
		\item For all $1\leq h \leq H$, $P\ind{h}= \unif(\Psi\ind{h})$;
		\item For each $1\leq h\leq H$, it holds that $\Psi\ind{h}$ is a $(\alpha,\veps)$-policy cover for layer $h$ (see \cref{def:polcover101}). 
	\end{itemize} 
	Then, under \cref{assum:real}, with probability at least $1-\delta$, we have that 
	\begin{align}
		\max_{\pi\in \Pim} \E^{\pi}\left[\sum_{h=1}^{H} r_h(\x_h,\a_h)\right]\leq  \E^{\pihat}\left[\sum_{h=1}^{H} r_h(\x_h,\a_h)\right] +  c H^2 \sqrt{\frac{d A \cdot (d \log(2n \sqrt{d}H) +\ln (n|\Phi|/\delta)) }{\alpha n }} + 2 H^2\veps d^{3/2},
	\end{align}
 for a sufficiently large absolute constant $c>0$.
\end{corollary} 
By using that the distributions returned by \mainalg{} are an $(\frac{\eta^3}{\cfrak \cdot d^6 A^2}, \veps)$-policy cover for $\eta = \veps/(4 H d^{3/2})$ and $\cfrak= \polylog(A,H,d, \log(|\Phi|/\delta))$ sufficiently large (\cref{thm:spanrlmain}), we obtain the claimed sample complexity for \cref{alg:spanRL} in \cref{tb:resultscomp+}.

\section{Helper Lemmas}
\label{sec:helper}
\begin{lemma}
	\label{lem:newnegative}
	For any $h\in[2\ldotst H]$, $x\in \cX_{h}$, and $\pi \in \Pim$, we have 
	\begin{align}
		d^{\pi}(x) = \muh[h](x)^\top \phi^{\star, \pi}_{h-1},\quad \text{where}\quad \phi^{\star, \pi}_{h-1} \coloneqq \E^{\pi}[\phi^{\star}_{h-1}(\x_{h-1},\a_{h-1})],
	\end{align}
\end{lemma}

\begin{lemma}
	\label{lem:unionbound}
	Let $\delta \in(0,1)$ and $H\geq 1$ be given. If a sequence of events $\cE_1,\ldots,\cE_H$ satisfies $\bbP[\cE_h\mid{}\cE_1,\ldots,\cE_{h-1}]\geq{}1-\delta/H$ for all $h\in[H]$, then \[\bbP[\cE_{1:H}]\geq{}1-\delta.\]
\end{lemma}
\begin{proof}[\crefzak{lem:unionbound}]
	By the chain rule, we have
	\begin{align}
		\bbP[\cE_{1:H}] = \prod_{h\in[H]} \bbP[\cE_h\mid{}\cE_1,\ldots,\cE_{h-1}] \geq  \prod_{h\in[H]} (1-\delta/H) =(1-\delta/H)^H \geq 1-\delta.
	\end{align}
\end{proof}

The normalization assumption in \eqref{eq:normalization} has the following useful implication.
\begin{lemma}
	\label{lem:normalization}
	For any $h\in[H]$, if the normalization condition \eqref{eq:normalization} holds, then 
	\begin{align}
		\int_{\cX_h} \|\mu^\star_h(x)\| \dd \nu(x) \leq d^{3/2}.
	\end{align}
\end{lemma}
\begin{proof}[Proof of \crefzak{lem:normalization}]
	For each $i\in[d]$, if we define $g(x)\coloneqq \mathrm{sgn}([\mu^\star_{h}(x)]_i)$, we have
	\begin{align}
		\int_{\cX_h} |[\mu^\star_{h}(x)]_i| \dd \nu (x)	 & =  \int_{\cX_h} g(x) \cdot [\mu^\star_{h}(x)]_i \dd \nu (x),\nn \\
		& =  \sqrt{ \left(\int_{\cX_h} g(x) \cdot [\mu^\star_{h}(x)]_i \dd \nu (x)\right)^2},\nn \\
		& \leq  \sqrt{ \sum_{j\in[d]} \left(\int_{\cX_h} g(x) \cdot [\mu^\star_{h}(x)]_j \dd \nu (x)\right)^2},\nn \\
		& =   \left\|  \int_{\cX_h} g(x) \cdot \mu^\star_h(x)\dd\nu(x) \right\|,\nn \\
		& \leq \sqrt{d}.
	\end{align}
	Therefore, we have 
	\begin{align}
		\int_{\cX_h} \|\mu^\star_{h}(x)\| \dd \nu (x)\leq  \sum_{i\in[d]}	\int_{\cX_h} |[\mu^\star_{h}(x)]_i| \dd \nu (x)\leq d^{3/2}.
	\end{align}
\end{proof}

\end{document}